%% file: main.tex
\documentclass[twoside]{article}

\usepackage[accepted]{aistats2025}
\usepackage[utf8]{inputenc} 
\usepackage[T1]{fontenc}    
\usepackage[hidelinks]{hyperref}       
\usepackage{url}            
\usepackage{booktabs}       
\usepackage{amsfonts}       
\usepackage{nicefrac}       
\usepackage{microtype}      
\usepackage{xcolor}         
\usepackage{algorithm}
\usepackage{algpseudocode}
\input{preamble}
\begin{document}

%

%
\runningauthor{Achraf Azize, Debabrota Basu}
\twocolumn[

\runningtitle{Quantifying the Target-dependent Membership Leakage}

\aistatstitle{Some Targets Are Harder to Identify than Others:\\ Quantifying the Target-dependent Membership Leakage}


\aistatsauthor{Achraf Azize\footnotemark \And Debabrota Basu}
\aistatsaddress{FairPlay Joint Team \\ CREST, ENSAE Paris \And Univ. Lille, Inria\\ CNRS, Centrale Lille\\ UMR 9189-CRIStAL}]

\footnotetext{This work is done during PhD at the Inria Centre at University of Lille.}


\input{sections/abstract}

\input{sections/introduction}

\input{sections/membership_inf_game}
\input{sections/mia_emp_mean}
\input{sections/priv_miss_effect}
\input{sections/beyond_emp_mean}

\input{sections/experiments}
\input{sections/conclusion}

\section*{Acknowledgements}
This work has been partially supported by the THIA ANR program “AI PhD@Lille”. D.
Basu acknowledges the Inria-Kyoto University Associate Team “RELIANT” for supporting
the project, the ANR JCJC for the REPUBLIC project (ANR-22-CE23-0003-01), and the
PEPR project FOUNDRY (ANR23-PEIA-0003). We thank Timothée Mathieu for the interesting conversations. We also thank Philippe Preux and Vianney Perchet for their supports.

\bibliography{references}
\bibliographystyle{apalike}

\section*{Checklist}

 \begin{enumerate}

 \item For all models and algorithms presented, check if you include:
 \begin{enumerate}
   \item A clear description of the mathematical setting, assumptions, algorithm, and/or model. Yes
   \item An analysis of the properties and complexity (time, space, sample size) of any algorithm. Yes
   \item (Optional) Anonymized source code, with specification of all dependencies, including external libraries. Yes
 \end{enumerate}

 \item For any theoretical claim, check if you include:
 \begin{enumerate}
   \item Statements of the full set of assumptions of all theoretical results. Yes
   \item Complete proofs of all theoretical results. Yes
   \item Clear explanations of any assumptions. Yes
 \end{enumerate}

 \item For all figures and tables that present empirical results, check if you include:
 \begin{enumerate}
   \item The code, data, and instructions needed to reproduce the main experimental results (either in the supplemental material or as a URL). Yes
   \item All the training details (e.g., data splits, hyperparameters, how they were chosen). Yes
     \item A clear definition of the specific measure or statistics and error bars (e.g., with respect to the random seed after running experiments multiple times). Yes
     \item A description of the computing infrastructure used. (e.g., type of GPUs, internal cluster, or cloud provider). Yes
 \end{enumerate}

 \item If you are using existing assets (e.g., code, data, models) or curating/releasing new assets, check if you include:
 \begin{enumerate}
   \item Citations of the creator If your work uses existing assets. Yes
   \item The license information of the assets, if applicable. Yes
   \item New assets either in the supplemental material or as a URL, if applicable. Not Applicable
   \item Information about consent from data providers/curators. Not Applicable
   \item Discussion of sensible content if applicable, e.g., personally identifiable information or offensive content. Not Applicable
 \end{enumerate}

 \item If you used crowdsourcing or conducted research with human subjects, check if you include:
 \begin{enumerate}
   \item The full text of instructions given to participants and screenshots. Not Applicable
   \item Descriptions of potential participant risks, with links to Institutional Review Board (IRB) approvals if applicable. Not Applicable
   \item The estimated hourly wage paid to participants and the total amount spent on participant compensation. Not Applicable
 \end{enumerate}

 \end{enumerate}


\newpage
\appendix
\onecolumn

\clearpage
\section{OUTLINE OF THE APPENDIX} \label{app:outline}

The appendices are organised as follows:
\begin{itemize}
    \item The connection between fixed-target and ``average-target'' games is discussed in Appendix~\ref{app:games}. The LR test of~\citep{sankararaman2009genomic} is revisited in Section~\ref{ssec:2:sankaraman} and the scalar product attack of~\citep{dwork2015robust} in Section~\ref{sec:dwork}. Also, a proof of Theorem~\ref{thm:opt_adv} is given in Section~\ref{app:proof_opt_adv}.
    \item All the missing proofs of Lemma~\ref{thm:asymp_lr}, Theorem~\ref{crl:per_leak}, Theorem~\ref{thm:lr_sub} and Theorem~\ref{thm:misspec} are given in Appendix~\ref{app:proof}. There, we introduce our main technical tools from asymptotic statistics. 
    \item Additional details on the white-box federated learning setting and discussions on related works are presented in Appendix~\ref{app:audt}.
    \item Extended experiments are presented in Appendix~\ref{app:exp_details}.
\end{itemize}

\input{appendix/mem_inf_games_avg_rec}
\clearpage
\input{appendix/proof_asymp_dist_lr}
\clearpage
\input{appendix/audit_pseudo_code}
\clearpage
\input{appendix/experiments_appendix}
\clearpage
\input{appendix/limit_impact}

\end{document}

%% file: preamble.tex
\usepackage{microtype}
\usepackage{graphicx}
\usepackage{booktabs,enumitem} 
\usepackage{natbib} 

\usepackage{wrapfig}
\usepackage{amsmath}
\usepackage{amssymb}
\usepackage{mathtools}
\usepackage{amsthm}
\usepackage{textcomp} 
\usepackage[capitalize,noabbrev]{cleveref}

\theoremstyle{plain}
\newtheorem{theorem}{Theorem}[section]

\newtheorem{lemma}[theorem]{Lemma}

\theoremstyle{definition}
\newtheorem{definition}[theorem]{Definition}

\theoremstyle{remark}
\newtheorem{remark}[theorem]{Remark}

\usepackage[textsize=tiny]{todonotes}

\usepackage{pgfplots}
\pgfplotsset{compat=1.18}
\usepackage[utf8]{inputenc} 
\usepackage[T1]{fontenc}    
\usepackage{url}            
\usepackage{booktabs}       
\usepackage{amsfonts}       
\usepackage{nicefrac}       
\usepackage{xcolor}         
\usepackage{dsfont}

\usepackage[toc,page,header]{appendix}
\usepackage{minitoc}

\usepackage[textsize=tiny]{todonotes}

\newcommand{\real}{\mathbb{R}}

\newcommand \defn {\mathrel{\triangleq}}
\newcommand \dd {\,\mathrm{d}}

\newcommand \expect {\mathop{\mbox{\ensuremath{\mathbb{E}}}}\nolimits}
\newcommand \Var {\mathop{\mbox{\ensuremath{\mathbb{V}}}}\nolimits}

\newcommand \ind[1] {\mathds{1}\left\{#1\right\}}

\newcommand \TV[2] {\mathrm{TV}\left( #1 ~\middle\|~ #2\right)}

\makeatletter
\newtheorem*{rep@theorem}{\rep@title}
\newcommand{\newreptheorem}[2]{%
	\newenvironment{rep#1}[1]{%
		\def\rep@title{\textbf{#2} \ref{##1}}%
		\begin{rep@theorem}}%
		{\end{rep@theorem}}}
\makeatother
\newreptheorem{theorem}{Theorem}
\newreptheorem{lemma}{Lemma}
\newreptheorem{proposition}{Proposition}
\newreptheorem{assumption}{Assumption}
\newreptheorem{corollary}{Corollary}

\usepackage{subfig}
\usepackage{multirow}
\usepackage{array}

%% file: sections/abstract.tex
\begin{abstract}
In a Membership Inference (MI) game, an attacker tries to infer whether a target point was included or not in the input of an algorithm. Existing works show that some target points are easier to identify, while others are harder. This paper explains the target-dependent hardness of membership attacks by studying the powers of the optimal attacks in a \textit{fixed-target} MI game. 
We characterise the optimal advantage and trade-off functions of attacks against the empirical mean in terms of the Mahalanobis distance between the target point and the data-generating distribution. We further derive the impacts of two privacy defences, i.e. adding Gaussian noise and sub-sampling, and that of target misspecification on optimal attacks. As by-products of our novel analysis of the Likelihood Ratio (LR) test, we provide a new covariance attack which generalises and improves the scalar product attack. Also, we propose a new optimal canary-choosing strategy for auditing privacy in the white-box federated learning setting. Our experiments validate that the Mahalanobis score explains the hardness of \textit{fixed-target} MI games.
\end{abstract}

%% file: sections/introduction.tex
\section{INTRODUCTION}\label{sec:intro}
A growing body of works on privacy attacks has shown that leakage of sensitive information through ML models is a common issue~\citep{shokri2017membership,yeom2018privacy}. 
For example, an attacker can violate the privacy of users involved in an input dataset by inferring whether a particular target data point of a user was used in the input dataset.
These attacks are called Membership Inference (MI) attacks~\citep{homer2008resolving,shokri2017membership}, and are identified as a privacy threat for and confidentiality violation of the users' data by ICO (UK) and NIST (USA)~\citep{murakonda2020ml,privacytesttensorflow}. 

The study of statistical efficiency and design of MI attacks has begun with the summary statistics on genomic data~\citep{homer2008resolving,sankararaman2009genomic,dwork2015robust} under the name of \textit{tracing attacks}~\citep{dwork2017exposed}. 
\citet{homer2008resolving} and \citet{sankararaman2009genomic} studied the first attacks to detect an individual in the exact empirical mean statistic, computed on a dataset generated by Bernoulli distributions.
They both propose and analyse the \textit{Likelihood Ratio (LR) tests} and assume access to the exact statistics and a pool of \textit{reference points}.  \citet{dwork2015robust} assumes access to only a noisy statistic and one reference sample, and develops a \textit{scalar product attack} to understand the correlation of a target point with the marginals of noisy statistics. 
\textit{However, these tracing attacks~\citep{homer2008resolving,sankararaman2009genomic,dwork2015robust} are studied in a threat model where the target point attacked is sampled randomly, either from the input dataset, or from a data generating distribution}. This means that the metrics of the attack under analysis, i.e. the \textit{advantage} of an attacker or \textit{trade-off functions between Type-I/Type-II errors}, are `averaged' over the target point's sampling. This obfuscates the target-dependent hardness of tracing attacks, which is important to understand due to the worst-case nature of privacy~\citep{diffprivorg_blog}. Also, prior works~\citep{carlini2022privacy, ye2022enhanced} empirically demonstrated that some target points could be easier to identify than others.
Hence, we ask the question
\begin{center}
\textit{Why are some points statistically harder to identify than others, and how can we quantify this hardness?}
\end{center}

Concurrent with the study of attacks, researchers have developed defence mechanisms to preserve the privacy of the users contributing to the input of an algorithm, and Differential Privacy (DP) has emerged as the gold standard~\citep{dpbook}.
Though DP promises to bound the worst-case privacy leakage on any attack, it is not always evident how these guarantees bound the accuracy of specific privacy attacks, such as MI attacks~\citep{zhang2020privacy,humphries2023investigating}. A DP algorithm comes with a mathematical proof yielding an upper bound on the privacy parameters that control the maximum leakage. 
On the other hand, a \textit{privacy audit} tries to empirically estimate the privacy parameters, by providing lower bounds on the parameters. Typically, a privacy audit runs an MI attack, and translates the Type-I/Type-II errors of the attack into a lower bound on the privacy budget. These algorithms apply different heuristics to find the most leaking target points (aka \textit{canaries}) that minimise Type-I/Type-II errors to estimate the privacy parameters tightly~\citep{maddock2022canife,nasr2023tight}. Thus, understanding the target-dependent hardness of MI attacks can lead to optimal canary-choosing strategies. We ask:
\begin{center}
	\textit{Can we quantify the target-dependent effect of privacy-preserving mechanisms?}\\ \textit{How to design an optimal canary-choosing strategy?}
\end{center}

\textbf{Our contributions} address these queries.\\
1. \textit{Defining the target-dependent leakage.} To understand the target-dependent hardness of MI games, we instantiate a \textit{fixed-target MI game} (Algorithm~\ref{alg:fix_game},~\cite{ye2022enhanced}). We define the leakage of a target point as
the optimal advantage of its corresponding fixed-target MI game. We characterise the target-dependent leakage using a Total Variation distance (Equation~\eqref{eq:adv_tv}).

2. \textit{Explaining the target-dependent leakage for the mean.} We investigate the fixed-target MI game for the empirical mean. We quantify the exact optimal advantage (Equation~\eqref{eq:leak_emp}) and trade-off function (Equation~\eqref{eq:pow_emp}) of the LR attack in this setting. This shows that the target-dependent hardness of MI games depends on the Mahalanobis distance between the target point $z^\star$ and the true data-generating distribution (Table~\ref{tab:1}). 

3. \textit{Tight quantification of the effects of noise addition, sub-sampling, and misspecified targets on leakage.} We further study the impact of privacy-preserving mechanisms, such as adding Gaussian noise and sub-sampling, on the target-dependent leakage. 
As shown in Table~\ref{tab:1}, both of them reduce the leakage scores, and thus, the powers of the optimal attacks. We also quantify how \textit{target misspecification} affects the leakage.

4. \textit{A new covariance attack and optimal canary-choosing strategy.} We analyse the LR score for the empirical mean asymptotically. Our novel proof technique combines an Edgeworth expansion with a Lindeberg-Feller central limit theorem to show that the \textit{LR score is asymptotically a scalar product attack, corrected by the inverse of the covariance matrix} (Equation~\eqref{eq:asymp_lr_real}). This enables us with a novel attack score that improves the scalar product by correcting it for the geometry of the data. We use this ``covariance score'' to propose a novel white-box attack (Algo.~\ref{alg:cov_att}) that experimentally outperforms the scalar product attack. We also use the insights from the target-dependent leakage to propose a new white-box 
 \textit{optimal} canary-choosing strategy (Algo.~\ref{alg:cana_maha}), based on an estimated Mahalanobis distance.

\begin{table}[t!]
\caption{Target-dependent leakage score}\label{tab:1}
\begin{center}
\resizebox{0.48\textwidth}{!}{
\begin{tabular}{cc}
\toprule
\textbf{SETTING} & \textbf{LEAKAGE SCORE} \\
\midrule
Empirical mean &  $\frac{1}{n} \| z^\star - \mu \|_{C_\sigma^{-1}}^2$ \\
Gaussian Noise $(\gamma > 0)$ &  $\frac{1}{n} \| z^\star - \mu \|_{(C_\sigma + C_\gamma)^{-1}}^2$ \\
Sub-sampling $(\rho < 1)$ &  $\frac{\rho}{n} \| z^\star - \mu \|_{C_\sigma^{-1}}^2$ \\
Similar point &  $\frac{1}{n} \left(z^\star_{\mathrm{targ}} - \mu\right)^T C_\sigma^{-1}\left(z^\star_{\mathrm{true}} - \mu \right)$ \\
\bottomrule
\end{tabular}}
\end{center}
\end{table}









%% file: sections/membership_inf_game.tex
\section{MEMBERSHIP INFERENCE GAMES}\label{sec:mig}
First, we introduce the fixed-target Membership Inference (MI) game. Then, we discuss different performance metrics to assess the power of an adversary. Finally, we characterise the optimal performances of adversaries using the Neyman-Pearson lemma.

\subsection{Fixed-target MI Game}
Let $\mathcal{M}$ be a randomised mechanism that takes as input a dataset $D$ of $n$ points belonging to $\mathcal{Z}$ and outputs $o \in \mathcal{O}$. In a Membership Inference (MI) game, an adversary attempts to infer whether a given target point $z^\star$ was included in the input dataset $D$ of $\mathcal{M}$, given only access to an output $o \sim \mathcal{M}(D)$. A fixed-target MI game is presented in Algo~\ref{alg:fix_game}. It is a game between two entities: the Crafter (Algo~\ref{alg:crafter}) and the adversary $\mathcal{A}_{z^\star}$. The MI game runs in multiple rounds. At each round $t$, the crafter samples a pair $(o_t, b_t)$, where $o_t$ is an output of the mechanism and $b_t$ is the secret binary membership of $z^\star$. The adversary $\mathcal{A}_{z^\star}$ takes as input only $o_t$ and outputs $\hat{b}_t$ trying to reconstruct $b_t$. 

The specificity of the \textit{fixed-target} MI game is that the target $z^\star$ is fixed throughout the game. Thus, the performance metrics of the attacker, i.e. the advantage and trade-off functions, are target-dependent. In contrast, in the MI game originally proposed in Experiment 1 of~\cite{yeom2018privacy}, the target $z^\star$ is sampled randomly at each step of the game, i.e. Step 3 in Experiment 1 of~\cite{yeom2018privacy}. In this case, the performance metrics of the attacker are averaged over the sampling of the target points and, thus, obfuscate the dependence of the leakage on each target point. To study the target-dependent hardness of MI games, we use this fixed-target formulation of Algo~\ref{alg:fix_game}, which has also been proposed in Definition 3.3 of~\cite{ye2022enhanced}.

A fixed-target MI game can also be seen as a hypothesis test. Here, the adversary tries to test the hypothesis ``$H_0$: \textit{The output $o$ observed was generated from a dataset sampled i.i.d. from $\mathcal{D}$}'', i.e. $b =0$, versus ``$H_1$: \textit{The target point $z^\star$ was included in the input dataset producing the output $o$}'', i.e. $b =1$. We denote by $p_\mathrm{out}(o \mid z^\star)$ and $p_\mathrm{in}(o \mid z^\star)$ the distributions of the output $o$ under $H_0$ and $H_1$ respectively.

\begin{algorithm}[t!]
	\caption{The Crafter}\label{alg:crafter}
	\begin{algorithmic}[1]
		\State {\bfseries Input:} Mechanism $\mathcal{M}$, Data distribution $\mathcal{D}$,  $\#$samples $n$, Target $z^\star$
		\State {\bfseries Output:} $(o, b)$, where $o \in \mathcal{O}$ and $b \in \{0,1\}$
		\State Build a dataset $D \sim \bigotimes\nolimits_{i = 1}^n \mathcal{D}$
		\State Sample $b \sim \mathrm{Bernoulli}\left(\frac{1}{2} \right)$
		\If{$b = 1$} 
		\State Sample $j \sim \mathcal{U}[n]$
		\State $D \leftarrow \mathrm{Replace}(D, j, z^\star)$  \Comment{Put $z^\star$ at position $j$}
		\EndIf
		\State Let $o \sim \mathcal{M}(D)$ 
		\State Return $(o, b)$
	\end{algorithmic}
\end{algorithm}

\subsection{Performance Metrics of the Attack}~\vspace{-1.7em}

An adversary $\mathcal{A}_{z^\star}$ is a (possibly randomised) algorithm that takes as input $o$ the output of the mechanism $\mathcal{M}$, and generates a guess $\hat{b} \sim \mathcal{A}_{z^\star}(o)$ trying to infer $b = \ind{z^\star \in D}$. The performance of $\mathcal{A}_{z^\star}$ can be assessed either with aggregated metrics like the advantage, or with test-based metrics like a trade-off function.

The \textit{accuracy} of $\mathcal{A}_{z^\star}$ is defined as $\mathrm{Acc}_n(\mathcal{A}_{z^\star}) \defn \mathrm{Pr} [\mathcal{A}_{z^\star}(o) = b]$,
where the probability is over any randomness in both the crafter and the adversary. The \textit{advantage} of an adversary is the centred accuracy: $\mathrm{Adv}_n(\mathcal{A}_{z^\star}) \defn 2 \mathrm{Acc}_n(\mathcal{A}_{z^\star}) - 1$. We can also define two errors from the hypothesis testing formulation. \textit{The Type-I error}, aka False Positive Rate, is $\alpha_n(\mathcal{A}_{z^\star}) \defn \mathrm{Pr} \left[\mathcal{A}_{z^\star}(o) = 1 \mid b = 0\right]$. \textit{The Type-II error}, aka the False Negative Rate, is $\beta_n(\mathcal{A}_{z^\star}) \defn \mathrm{Pr} \left[\mathcal{A}_{z^\star}(o) = 0 \mid b = 1\right]$. The \textit{power} of the test is $1 - \beta_n(\mathcal{A}_{z^\star})$.
In MI games, an adversary can threshold over a score function $s$ to conduct the MI games, i.e. for $\mathcal{A}_{s, \tau, z^\star}(o) \defn \ind{s(o; z^\star) > \tau}$ where $s$ is a score function and $\tau$ is a threshold. We want to design score functions that maximise the power under a fixed significance level $\alpha$, i.e. $\mathrm{Pow}_n(s, \alpha, z^\star) \defn \max_{\tau \in T_\alpha} 1 - \beta_n(\mathcal{A}_{s, \tau, z^\star})$
where $T_\alpha \defn \{\tau \in \real: \alpha_n(\mathcal{A}_{s, \tau, z^\star}) \leq \alpha\}$. $\mathrm{Pow}_n(s, \alpha, z^\star)$ is also called a \textit{trade-off function}.

\subsection{Optimal Adversaries \& Defining Leakage} 
Given two data generating distributions $p_0$ and $p_1$ under hypotheses $H_0$ and $H_1$ respectively, no test can achieve better power than the Likelihood Ratio (LR) test~\citep{neyman1933ix}. By considering the hypothesis testing formulation of the fixed-target MI game, the \textit{log-Likelihood Ratio} (LR) score is 
\begin{equation}
    \ell_n(o; z^\star) \triangleq \log \left ( \frac{p^\mathrm{in}_n(o \mid z^\star)}{p^\mathrm{out}_n(o \mid z^\star)} \right).
\end{equation}
The LR-based adversary uses a threshold $\tau$ on the LR score, i.e. $\mathcal{A}_{\ell, z^\star, \tau}(o) \defn \ind{\ell_n(o; z^\star) > \tau}$. We denote by $\mathcal{A}_{\mathrm{Bayes}, z^\star} \defn \mathcal{A}_{\ell, z^\star, 0}$ the LR attacker with threshold $\tau = 0$. We provide Theorem~\ref{thm:opt_adv} to characterise optimal adversaries under aggregated and test-based metrics.

\begin{algorithm}[t!]
   \caption{Fixed-target MI Game}\label{alg:fix_game}
\begin{algorithmic}[1]
   \State {\bfseries Input:} Mechanism $\mathcal{M}$, Data distribution $\mathcal{D}$,  $\#$samples $n$, Target $z^\star$, Adversary $\mathcal{A}_{z^\star}$, Rounds~$T$
   \State {\bfseries Output:} A list $L \in \{0, 1\}^T$, where $L_t = 1$ if the adversary succeeds at step $t$.
   \State Initialise a empty list $L$ of length $T$
   \For{t = 1, \dots, T}
   \State Sample $(o_t, b_t) \sim$ Crafter($\mathcal{M}$, $\mathcal{D}$, n, $z^\star$)
   \State Sample $\hat{b}_t \sim \mathcal{A}_{z^\star}(o_t)$
   \State Set $L_t \leftarrow \ind{b_t = \hat{b}_t}$
   \EndFor
   \State Return $L$
\end{algorithmic}
\end{algorithm}

\begin{theorem}[Characterising Optimal Adversaries]\label{thm:opt_adv}{\ }
    \begin{itemize}
    \item[(a)] \textit{Optimal power:} $\forall \alpha \in [0,1], \forall$ score $s$, $\forall$ target $z^\star$, $\mathrm{Pow}_n(\ell_n, \alpha, z^\star)  \geq \mathrm{Pow}_n(s, \alpha, z^\star)$.

    \item[(b)] \textit{Largest advantage:} $\forall$ target $z^\star$, $\forall$ adversary $\mathcal{A}_{z^\star}$, $\mathrm{Adv}_n(\mathcal{A}_{\mathrm{Bayes}, z^\star}) \geq \mathrm{Adv}_n(\mathcal{A}_{z^\star})$.

    \item [(c)] \textit{Optimal advantage as TV distance:}
\end{itemize}~\vspace{-2.5em}
    \begin{equation}\label{eq:adv_tv}
        \mathrm{Adv}_n(\mathcal{A}_{\mathrm{Bayes}, z^\star}) = \TV{p^\mathrm{out}_n(. \mid z^\star)}{p^\mathrm{in}_n(. \mid z^\star)}.
    \end{equation}
\end{theorem}


The detailed proof is in Appendix~\ref{app:proof_opt_adv}. 
As a consequence of Theorem~\ref{thm:opt_adv}, we \textbf{define} the \textit{target-dependent leakage} of $z^\star$, for mechanism $\mathcal{M}$ and data-generating distribution $\mathcal{D}$, as the advantage of the optimal Bayes attacker on $z^\star$, i.e. $\xi_n(z^\star, \mathcal{M}, \mathcal{D}) \defn  \mathrm{Adv}_n(\mathcal{A}_{\mathrm{Bayes}, z^\star})$. 


\textbf{Goal:} Our main goal is to \textit{quantify the target-dependent leakage $\xi_n(z^\star, \mathcal{M}, \mathcal{D})$ and trade-off functions for different mechanisms}, namely the empirical mean and its variations. These two quantities may be intractable to characterise for any generic data-generating distribution. To get over this limitation, we use the asymptotic properties of the empirical mean as the main tool. 

%% file: sections/mia_emp_mean.tex
\section{TARGET-DEPENDENT LEAKAGE OF EMPRICAL MEAN}\label{sec:mia_mean}
We instantiate the fixed-target MI game with the empirical mean. We quantify the target-dependent leakage of a target $z^\star$ and characterise its dependence on the Mahalanobis distance between $z^\star$ and the data-generating distribution. Finally, we connect our results to the tracing literature~\citep{sankararaman2009genomic,dwork2015robust} and explain the privacy onion phenomenon.

\textbf{Notations and the asymptotic regime.} We denote by $\mathcal{M}^\mathrm{emp}_n$ the empirical mean mechanism. $\mathcal{M}^\mathrm{emp}_n$ takes as input a dataset of size $n$ of $d$-dimensional points, i.e. $D = \{Z_1, \dots, Z_n \} \in (\real^d)^n $, and outputs the exact empirical mean $\hat{\mu}_n \defn \frac{1}{n} \sum_{i=1}^n Z_i \in \real^d $. Let $\Phi$ represent the Cumulative Distribution Function (CDF) of the standard normal distribution, i.e. $\Phi(\alpha) \defn \frac{1}{\sqrt{2 \pi}} \int_{- \infty}^\alpha e^{- t^2/2} \dd t$ for $\alpha \in \real$. For a matrix $M$ and a vector $x$, we write $\|x\|^2_{M} \defn x^T M x$.
Since the LR test can be non-tractable in general cases, we study the \textit{asymptotic behaviour of the LR test}, when both the sample size $n$ and the dimension $d$ tend to infinity such that $d/n = \tau > 0$.  

\textbf{Assumptions on the data generating distribution $\mathcal{D}$.} We suppose that the data-generating distribution is column-wise independent, i.e. $\mathcal{D} \defn \bigotimes_{j = 1}^{d} \mathcal{D}_j$ and has a finite $(4 + \delta)$-th moment for some small $\delta$, i.e. there exists $\delta > 0$, such that $\expect[Z^{4 + \delta}] < \infty$. We denote by $\mu \defn (\mu_1, \dots, \mu_d) \in \real^d$ the mean of $\mathcal{D}$, and by $C_\sigma \defn \mathrm{diag}(\sigma_1^2, \dots, \sigma_d^2) \in \real^{d \times d}$ the covariance matrix. We recall that the Mahalanobis distance~\citep{mahalanobisdistance} of $z^\star$ with respect to $\mathcal{D}$ is $\| z^\star - \mu \|_{C_\sigma^{-1}}$.

\textbf{Main result.} Let us denote by $\ell_n(\hat{\mu}_n; z^\star, \mu, C_\sigma)$ the LR score for the empirical mean target-dependent MI game. By studying the asymptotic distribution of the LR score $\ell_n$ under $H_0$ and $H_1$, we characterise the exact asymptotic leakage and optimal trade-off functions.

\begin{theorem}[Target-dependent leakage of the empirical mean]\label{crl:per_leak}
The asymptotic target-dependent leakage of $z^\star$ in the empirical mean is\vspace*{-.5em}
\begin{equation}\label{eq:leak_emp}
    \hspace*{-0.5em}\lim_{n,d} \xi_n(z^\star, \mathcal{M}^\mathrm{emp}_n, \mathcal{D}) = \Phi\left (\frac{\sqrt{\textcolor{blue}{m^\star}}}{2}\right) - \Phi\left(-\frac{\sqrt{\textcolor{blue}{m^\star}}}{2}\right).\vspace*{-.5em}
\end{equation}
The asymptotic trade-off function, achievable with threshold $\tau_\alpha = -\frac{\textcolor{blue}{m^\star}}{2} + \sqrt{\textcolor{blue}{m^\star}} \Phi^{-1}(1 - \alpha)$, is\vspace*{-0.8em}
\begin{equation}\label{eq:pow_emp}
    \lim_{n,d} \mathrm{Pow}_n(\ell_n, \alpha, z^\star) = \Phi\left( \Phi^{-1}(\alpha) + \sqrt{\textcolor{blue}{m^\star}}\right).\vspace*{-.8em}
\end{equation}
Here, $\textcolor{blue}{m^\star} \defn \lim_{n,d} \frac{1}{n} \| z^\star - \mu \|^2_{C_\sigma^{-1}}$, which we call the \emph{leakage score} of target $z^\star$.
\end{theorem}
\textit{Proof Sketch.} Theorem~\ref{crl:per_leak} is a consequence of Lemma~\ref{thm:asymp_lr} presented in Appendix~\ref{sec:prf_thm_mean}. Lemma~\ref{thm:asymp_lr} characterises the asymptotic distributions of the LR score as Gaussians under $H_0$ and $H_1$. We retrieve the asymptotic leakage and trade-off functions by using testing results between Gaussians. To prove Lemma~\ref{thm:asymp_lr}, (a) we rewrite the LR score with respect to the density of a centred normalised mean. Then, we use the Edgeworth asymptotic expansion (Theorem~\ref{thm:edgeworth}) to get an expansion of the LR score. Finally, we conclude the asymptotic distribution of the LR test using the Lindeberg-Feller theorem (Theorem~\ref{thm:lind_feller}). The detailed proofs are presented in Appendix~\ref{sec:prf_thm_mean}.

\textit{Target-dependent hardness.} Theorem~\ref{crl:per_leak} shows that the optimal advantage and trade-off functions are increasing in the leakage score $\textcolor{blue}{m^\star}$. As $z^\star$ has a higher Mahalanobis distance with respect to the data-generating distribution, the leakage score increases, making it easier to identify the target $z^\star$ in the MI game.


\textbf{Empirical LR attack.} Following the proof of Theorem~\ref{crl:per_leak}, we show in Remark~\ref{rmk:asymp} that the LR score
\begin{align}\label{eq:asymp_lr_real}
    \ell_n \sim (z^\star - \mu)^T C_\sigma^{-1}(\hat{\mu}_n - \mu) 
    - \frac{1}{2n} \|z^\star - \mu\|^2_{C_\sigma^{-1}}
\end{align}
asymptotically in $n$ and $d$. Equation~\eqref{eq:asymp_lr_real} shows that the LR score is a scalar product between $z^\star - \mu $ and $\hat{\mu}_n - \mu$, corrected by the precision matrix $C_\sigma^{-1}$. The optimal LR score uses the true mean $\mu$ and covariance matrix $C_\sigma$. A realistic attack should replace the true mean $\mu$ and co-variance $C_\sigma$ in Equation~\eqref{eq:asymp_lr_real} with empirical estimates. This leads to a new covariance score $\ell_n^{\mathrm{cov}}(\hat{\mu}_n; z^\star) = (z^\star - \hat{\mu}_0)^T \hat{C}_0^{-1}(\hat{\mu}_n - \hat{\mu}_0) - \frac{1}{2n} \|z^\star - \hat{\mu}_0\|^2_{\hat{C}_0^{-1}},$
where $\hat{\mu}_0 = \frac{1}{n_0} \sum_{i =1}^{n_0} Z^{\mathrm{ref}}_i $ and $\hat{C}_0 = \frac{1}{n_0} \sum_{i =1}^{n_0} Z^{\mathrm{ref}}_i (Z^{\mathrm{ref}}_i)^T$ are estimated using reference points sampled independently from $\mathcal{D}$. The decrease in the power of the covariance attack compared to the LR attack depends on the accuracy of the estimators $\hat{\mu}_0$ and $\hat{C}_0$.

\textbf{Connection to~\cite{sankararaman2009genomic}.} For Bernoulli distributions, \citet{sankararaman2009genomic} shows that the hardness of ``average-target'' MI game depend on the ratio $\tau \defn d/n$~\citep[Section T2.1]{sankararaman2009genomic}. Since $\expect_{z^\star \sim \mathcal{D}}\left [ \textcolor{blue}{m^\star} \right] = \lim_{n,d} \frac{d}{n} =\tau$, our results retrieve the ``averaged'' results of~\citep{sankararaman2009genomic}. In addition, \citet{sankararaman2009genomic} uses an analysis tailored only for Bernoulli distributions. 
Our analysis generalises their results to the target-dependent setting and to any data-generating distribution with a finite fourth moment.

\textbf{Connection to the scalar product attack.} \citet{dwork2015robust} proposes a scalar product attack for tracing the empirical mean that thresholds over the score $s^\mathrm{scal}(\hat{\mu}_n; z^\star, z^\mathrm{ref})  \defn (z^\star - z^\mathrm{ref})^T \hat{\mu}_n$. The intuition behind this attack is to compare the target-output correlation $(z^\star)^T \hat\mu_n$ with a reference-output correlation $(z^\mathrm{ref}) ^T \hat\mu_n$. The analysis of~\citep{dwork2015robust} shows that with only one reference point $z^\mathrm{ref} \sim \mathcal{D}$, and even for noisy estimates of the mean, the attack is able to trace the data of some individuals in the regime $d \sim n^2$. \textit{Our asymptotic analysis shows that the LR attack, i.e. the optimal attack, is also a scalar-product attack} (Equation~\eqref{eq:asymp_lr_real}), but \textit{corrected for the geometry of the data using the inverse covariance matrix}. 


\textbf{Explaining the privacy onion effect.} Removing a layer of outliers is equivalent to sampling from a new data-generating distribution, with a smaller variance. In this new data-generating distribution with smaller variance, the points which are not removed will naturally have an increased Mahalanobis distance. Thus, removing a layer of outlier points yields a layer of newly exposed target points. Hence, the Mahalanobis score explains the privacy onion effect~\citep{carlini2022privacy}.

\textbf{Inherent privacy of the empirical mean.} Under our specific threat model of MI games, the empirical mean already imposes a trade-off between the Type-I and Type-II errors for \textit{any adversary}. This means that, if an auditor uses the fixed-target MI game with some target $z^\star$ to audit the privacy of the empirical mean, the auditor would conclude that the empirical mean is $\sqrt{m^\star}$-Gaussian DP~\citep{dong2019gaussian}, or equivalently $(\epsilon, \delta)$-DP where for all $\epsilon \geq 0$, $\delta(\epsilon) = \Phi\left( - \frac{\epsilon}{\sqrt{m^\star}} + \frac{\sqrt{m^\star}}{2} \right) - e^\epsilon \Phi\left( - \frac{\epsilon}{\sqrt{m^\star}} - \frac{\sqrt{m^\star}}{2} \right).$ The result is a direct consequence of Equation~\eqref{eq:pow_emp} and \citep[Corollary 2.13]{dong2019gaussian}.


\noindent \textbf{MI games with Z-estimators and relation to influence functions.} The main technical tool used to provide an asymptotic expansion of the LR score is the ``asymptotic normality'' of the empirical mean, i.e. the Edgeworth expansion in Theorem~\ref{thm:edgeworth}. The empirical mean estimator is only an instance of a more general class of estimators enjoying the asymptotic normality property, called $Z$-estimators~\citep{Vaart_1998}. Now, suppose we are interested in estimating a parameter $\theta$ that is a functional of the distribution of observations $X_1, \dots, X_n$. A popular method to construct an estimator $\hat \theta_n = \hat \theta_n( X_1, \dots, X_n)$ is to satisfy~\vspace*{-1em} 
\begin{equation}\label{eq:Z-est}
    \Psi_n(\theta) \defn \frac{1}{n} \sum_{i= 1}^n \psi_\theta(X_i) = 0.~\vspace*{-1em} 
\end{equation}
Here, $\psi_\theta$ are known functions. The class of $Z$-estimators retrieves the empirical mean with $\psi_\theta(X_i) \defn X_i - \theta$ and the median with $\psi_\theta(X_i) \defn \mathrm{sign}(X_i - \theta)$. The class of $Z$-estimators also recovers many other estimators, such as maximum likelihood estimators, least square estimators, and empirical risk minimisers.

Under technical conditions on the data-generating distribution and the ``regularity" of the function $\psi_\theta$, it is possible to show that $\theta_n$ converges in probability to a parameter $\theta_0$ , i.e. a zero of the function $\Psi(\theta) \defn \mathbb{E}_X(\psi_ \theta(X))$. Also, for any $Z$-estimator, Theorem 5.21 in~\citep{van2000asymptotic} shows that
\begin{equation}\label{eq:asymp_normality}
 \hat \theta_n - \theta_0 = - V_{\theta_0}^{-1} \frac{1}{n} \sum_{i = 1}^n \psi_{\theta_0}(X_i) + o_p\left( \frac{1}{\sqrt{n}} \right),
\end{equation}
where $V_\theta$ is a non-singular derivative matrix of the map $ \theta \rightarrow \Psi(\theta)$ at $\theta_0$. Generally, $I_{\theta_0}(X_i) \defn V_{\theta_0}^{-1} \psi_{\theta_0}(X_i) $ is called the influence function. Thus, Equation~\eqref{eq:asymp_normality} shows that, asymptotically, any $Z$-estimator can be thought of \textit{as the empirical mean of its influence functions}. Using our target-dependent analysis of empirical mean MI games, it is direct to provide a new covariance score and a new canary selection strategy for all $Z$ estimators. For the score, the covariance attack becomes 
\begin{equation*}
 \left(I_{\theta_0}(X^\star)  - \theta_0\right)^T V_{\theta_0}^{-1}(\hat{\theta}_n - \theta_0) - \frac{1}{2n} \|I_{\theta_0}(X^\star) - \theta_0 \|^2_{V_{\theta_0}^{-1}}.
\end{equation*}
Similarly, to chose canaries, i.e. targets that are easy to identify, find points for which the estimated Mahalanobis distance of the influence functions at $X^\star$ is high, i.e. $\|I_{\theta_0}(X^\star) - \theta_0 \|_{V_{\theta_0}^{-1}}$. We leave it for future work to provide a rigorous statement of when these statements are correct, i.e. rigorous conditions on the data-generating distribution and regularity of $\psi_\theta$.

%% file: sections/priv_miss_effect.tex
\section{IMPACT OF PRIVACY DEFENCES \& MISSPECIFICATION}\label{sec:defences}
We quantify the effect of adding noise and sub-sampling on the leakage of the empirical mean. Both defences act like contractions of the leakage score.
We also study the effect of target misspecification. The detailed proofs for this section are presented in Appendix~\ref{app:proof}.

\textbf{I. Adding Gaussian noise.} We denote by $\mathcal{M}^{\gamma}_n$ the mechanism releasing the noisy empirical mean of a dataset using the Gaussian mechanism~\citep{dpbook}. $\mathcal{M}^\gamma_n$ takes as input a dataset of size $n$ of $d$-dimensional points, i.e. $D = \{Z_1, \dots, Z_n \} \in (\real^d)^n $, and outputs the noisy mean $\tilde{\mu}_n \triangleq \frac{1}{n} \sum_{i=1}^n Z_i + \frac{1}{\sqrt{n}} N_d \in \real^d$, where $N_d \sim \mathcal{N}\left(0,C_{\gamma}\right)$ such that $\gamma = \left( \gamma_1, \dots, \gamma_d \right) \in \real^d $ and  $C_{\gamma} = \mathrm{diag}\left( \gamma_1^2, \dots, \gamma_d^2 \right) \in \real^{d \times d}$. Similar to Section~\ref{sec:mia_mean}, we assume that the data-generating  distribution $\mathcal{D}$ is colon-wise independent, has a mean $\mu \defn (\mu_1, \dots, \mu_d) \in \real^d$, a covariance matrix $C_\sigma \defn \mathrm{diag}(\sigma_1^2, \dots , \sigma_d^2) \in \real^{d \times d}$, and a finite $(4 + \delta)$-th moment for $\delta>0$. 

The output of $\mathcal{M}^\gamma_n$ could be re-written as $\tilde{\mu}_n = \frac{1}{n} \sum_{i=1}^n (Z_i + N_i) = \frac{1}{n} \sum_{i=1}^n \tilde{Z}_i$, where $\tilde{Z}_i \sim \tilde{\mathcal{D}}$ s.t. $\tilde{\mathcal{D}} \defn \mathcal{D} \bigotimes \mathcal{N}\left(0,C_{\gamma}\right)$. This means that $\mathcal{M}^\gamma_n$ could be seen as \textit{the exact empirical mean of $n$ i.i.d samples} from a new data-generating distribution $\tilde{\mathcal{D}}$. The results of Section~\ref{sec:mia_mean} directly apply to $\mathcal{M}^\gamma_n$, by replacing $\mathcal{D}$ by $\tilde{\mathcal{D}}$. The noisy leakage score $\textcolor{blue}{\tilde{m}^\star_\gamma}$ is now defined as $\textcolor{blue}{\tilde{m}^\star_\gamma} \defn \lim_{n,d} \frac{1}{n} \| z^\star - \mu \|^2_{\left(C_{\sigma} + C_{\gamma}\right)^{-1}}$.
Directly using the results of Section~\ref{sec:mia_mean} for $\tilde{\mathcal{D}}$ yields the following theorem.

\begin{theorem}[Target-dependent leakage of the noisy empirical mean]\label{thm:noisy_mean}
The asymptotic target-dependent leakage of $z^\star$ in the noisy empirical mean is\vspace*{-.5em}
\begin{equation}\label{eq:leak_nois}
    \hspace*{-0.5em}\lim_{n,d} \xi_n(z^\star, \mathcal{M}^\gamma_n, \mathcal{D}) = \Phi\left (\frac{\sqrt{\textcolor{blue}{\tilde{m}^\star_\gamma}}}{2}\right) - \Phi\left(-\frac{\sqrt{\textcolor{blue}{\tilde{m}^\star_\gamma}}}{2}\right).\vspace*{-.5em}
\end{equation}
The asymptotic trade-off function, achievable with threshold $\tau_\alpha = -\frac{\textcolor{blue}{\tilde{m}^\star_\gamma}}{2} + \sqrt{\textcolor{blue}{\tilde{m}^\star_\gamma}} \Phi^{-1}(1 - \alpha)$, is\vspace*{-0.8em}
\begin{equation}\label{eq:pow_nois}
    \lim_{n,d} \mathrm{Pow}_n(\tilde{\ell}_n, \alpha, z^\star) = \Phi\left( \Phi^{-1}(\alpha) + \sqrt{\textcolor{blue}{\tilde{m}^\star_\gamma}}\right).\vspace*{-.8em}
\end{equation}
\end{theorem}

Theorem~\ref{thm:noisy_mean} shows that the Gaussian Mechanism acts by increasing the variance of the data-generating distribution, thus decreasing the Mahalanobis distance of target points and their leakage score.

\textbf{II. Effect of sub-sampling.}
We consider the \textit{empirical mean with sub-sampling} mechanism~\citep{balle2018privacy} $\mathcal{M}^{\mathrm{sub}, \rho}_n$ that uniformly sub-samples $k_n$ rows without replacement from the original dataset, and then computes the exact empirical mean of the sub-sampled rows. $\mathcal{M}^{\mathrm{sub}, \rho}_n$ takes as input a dataset $D =\{ Z_1, \dots, Z_n \} \in (\real^d)^n$ and outputs $\hat{\mu}^\mathrm{sub}_{k_n} \triangleq \frac{1}{k_n} \sum_{i=1}^n Z_i \ind{\varsigma(i) \leq k_n}.$
Here, $k_n \defn \rho n$, $0 < \rho < 1$ and $\varsigma \sim^{\mathrm{unif}} S_n$ is a permutation sampled uniformly from $S_n$ the set of permutations of $\{1 \dots, n\}$, and independently from $(Z_1, \dots, Z_n)$. We get the following result by adapting the proofs in Sec.~\ref{sec:mia_mean}.


\begin{theorem}[Leakage for sub-sampling]\label{thm:lr_sub}
The asymptotic target-dependent leakage of $z^\star$ in $\mathcal{M}^{\mathrm{sub}}_{n, \rho}$ is\vspace*{-.5em}
\begin{equation*}\label{eq:leak_sub}
    \hspace*{-0.5em}\lim_{n,d} \xi_n(z^\star, \mathcal{M}^{\mathrm{sub}}_{n, \rho}, \mathcal{D}) = \Phi\left (\frac{\sqrt{\textcolor{blue}{\rho  m^\star}}}{2}\right) - \Phi\left(-\frac{\sqrt{\textcolor{blue}{\rho  m^\star}}}{2}\right).\vspace*{-.5em}
\end{equation*}
The asymptotic trade-off function, achievable with threshold $\tau_\alpha = -\frac{\textcolor{blue}{\rho  m^\star}}{2} + \sqrt{\textcolor{blue}{\rho  m^\star}} \Phi^{-1}(1 - \alpha)$, is\vspace*{-0.8em}
\begin{equation}\label{eq:pow_sub}
    \lim_{n,d} \mathrm{Pow}_n(\ell^\mathrm{sub}_{n, \rho}, \alpha, z^\star) = \Phi\left( z_\alpha + \sqrt{\textcolor{blue}{\rho  m^\star}}\right).\vspace*{-.8em}
\end{equation}
\end{theorem}

Theorem~\ref{thm:lr_sub} shows that the sub-sampling mechanism acts by increasing the number of ``effective samples'' from $n$ to $n/\rho$, thus decreasing the leakage score.

\textbf{III. Misspecifying the target.} Suppose that the adversary has a misspecified target $z^{\mathrm{targ}}$, different from the real $z^\star$ used in the fixed-target MI game (Algorithm~\ref{alg:fix_game}). The adversary $\mathcal{A}_\mathrm{miss}$ then builds the ``optimal" LR test \textit{tailored} for $z^{\mathrm{targ}}$, i.e. $\ell_n(\hat{\mu}_n; z^{\mathrm{targ}}, \mu, C_\sigma)$. The misspecified adversary is sub-optimal but can still leak enough information depending on the amount of misspecification. Now, we quantify the sub-optimality of the misspecified adversary, which we define as a measure of leakage similarity between $z^{\mathrm{targ}}$ and $z^\star$.


\begin{theorem}[Leakage of a misspecified adversary]\label{thm:misspec}
The advantage of the misspecified adversary is $$\lim_{n,d} \mathrm{Adv}_n(\mathcal{A}_\mathrm{miss}) = \Phi\left(\frac{|\textcolor{violet}{\textcolor{violet}{m^\mathrm{scal}}} |}{2 \sqrt{\textcolor{teal}{m^\mathrm{targ}}}}\right) - \Phi\left(-\frac{|\textcolor{violet}{m^\mathrm{scal}} |}{2 \sqrt{\textcolor{teal}{m^\mathrm{targ}}}}\right).$$
Here, $\textcolor{violet}{m^\mathrm{scal}} \defn \lim_{n, d} \frac{1}{n}  \left(z^{\mathrm{targ}} - \mu\right)^T C_\sigma^{-1}\left(z^\star - \mu \right)$ and $\textcolor{teal}{m^\mathrm{targ}} \defn \lim_{n, d} \frac{1}{n} \| z^{\mathrm{targ}} - \mu \|^2_{C_\sigma^{-1}}$.\vspace*{-.5em}
\end{theorem}

If the adversary specified well the target by using $z^\star$ rather than $z^\mathrm{targ}$, then they achieve the optimal asymptotic advantage $\lim_{n,d} \xi_n(z^\star, \mathcal{M}^\mathrm{emp}_n, \mathcal{D})$ of Equation~\ref{eq:leak_emp}.
Theorem~\ref{thm:misspec} \textit{quantifies the sub-optimality of the misspecified adversary}, which is $ \Delta(z^\mathrm{targ}, z^\star) = \lim_{n,d} \xi_n(z^\star, \mathcal{M}^\mathrm{emp}_n, \mathcal{D}) - \mathrm{Adv}_n(\mathcal{A}_\mathrm{miss})$. This quantity depends on the comparison between $\textcolor{blue}{m^\star}$ and $|\textcolor{violet}{m^\mathrm{scal}} | / \sqrt{\textcolor{teal}{m^\mathrm{targ}}}$.
By the Cauchy Schwartz inequality, $|\textcolor{violet}{m^\mathrm{scal}} | \leq \sqrt{\textcolor{teal}{m^\mathrm{targ}} \textcolor{blue}{m^\star}}$, which means that $\Delta(z^\mathrm{targ}, z^\star) \geq 0$. The misspecified attack is still strong as long as $ \sqrt{m^{\mathrm{targ}} \textcolor{blue}{m^\star}} - |\textcolor{violet}{m^\mathrm{scal}} | = \sqrt{\textcolor{teal}{m^\mathrm{targ}} \textcolor{blue}{m^\star}} (1 - |\mathrm{cos}(\theta)|) $ stays small. We geometrically illustrate $\theta$ in Figure~\ref{fig:lr_misss}.

\begin{figure}[t!]
    \centering
    \includegraphics[width= 0.16\textwidth]{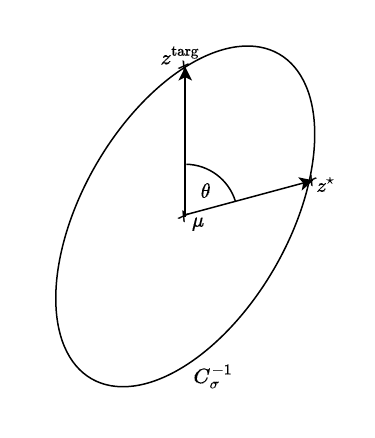}
    \caption{Effect of misspecification depends on the angle $\theta$ between $z^\star - \mu$ and $z^\mathrm{targ} - \mu$, corrected by $C_\sigma^{-1}$}\label{fig:lr_misss}
\end{figure}

\input{sections/figure_main}

%% file: sections/figure_main.tex
\begin{figure*}[t!]
\centering\vspace*{-2em}
\begin{tabular}{ccc}
\subfloat[Effect of $m^\star$]{\includegraphics[width=0.3\textwidth]{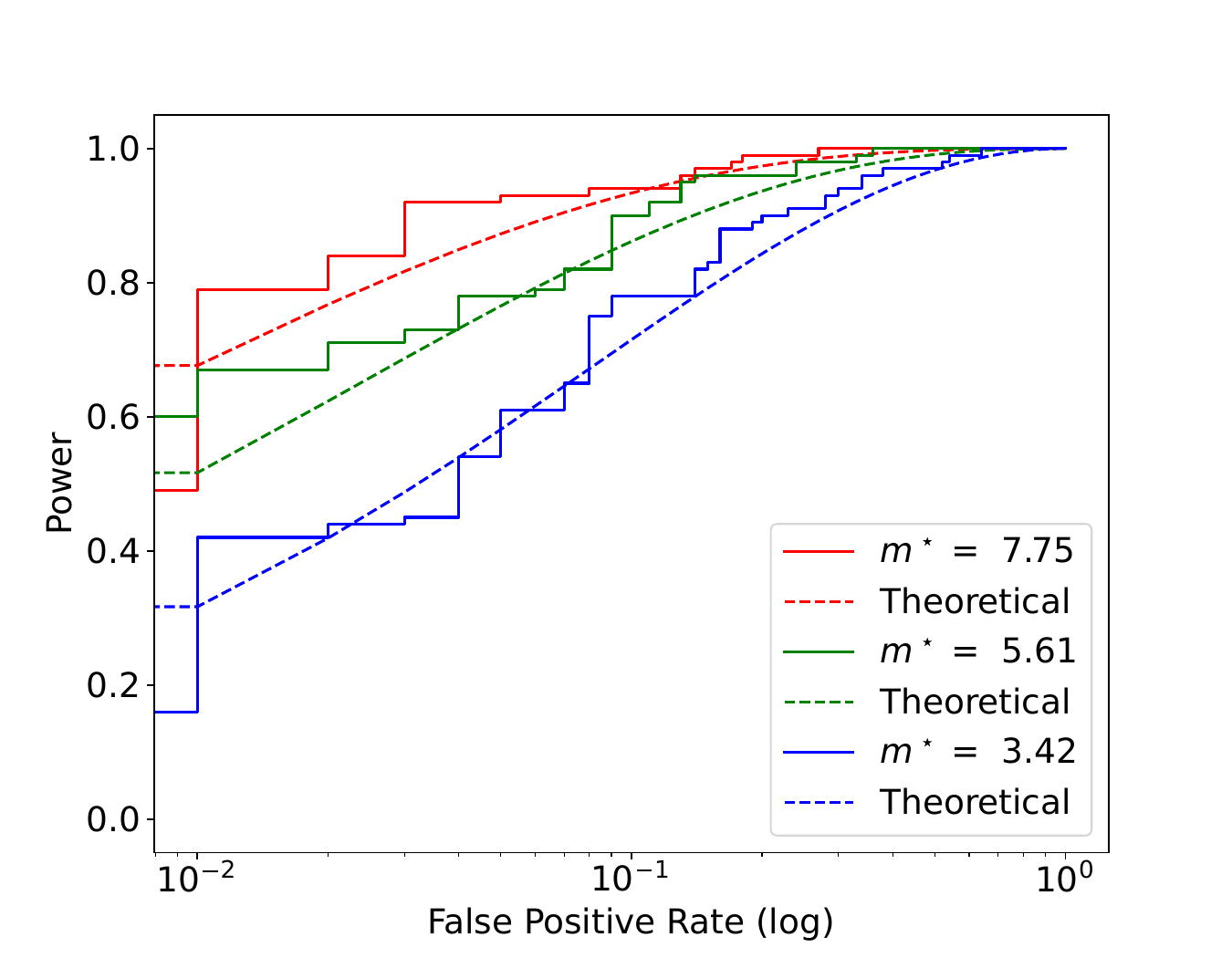}} &
\subfloat[Effect of $\gamma$]{\includegraphics[width=0.3\textwidth]{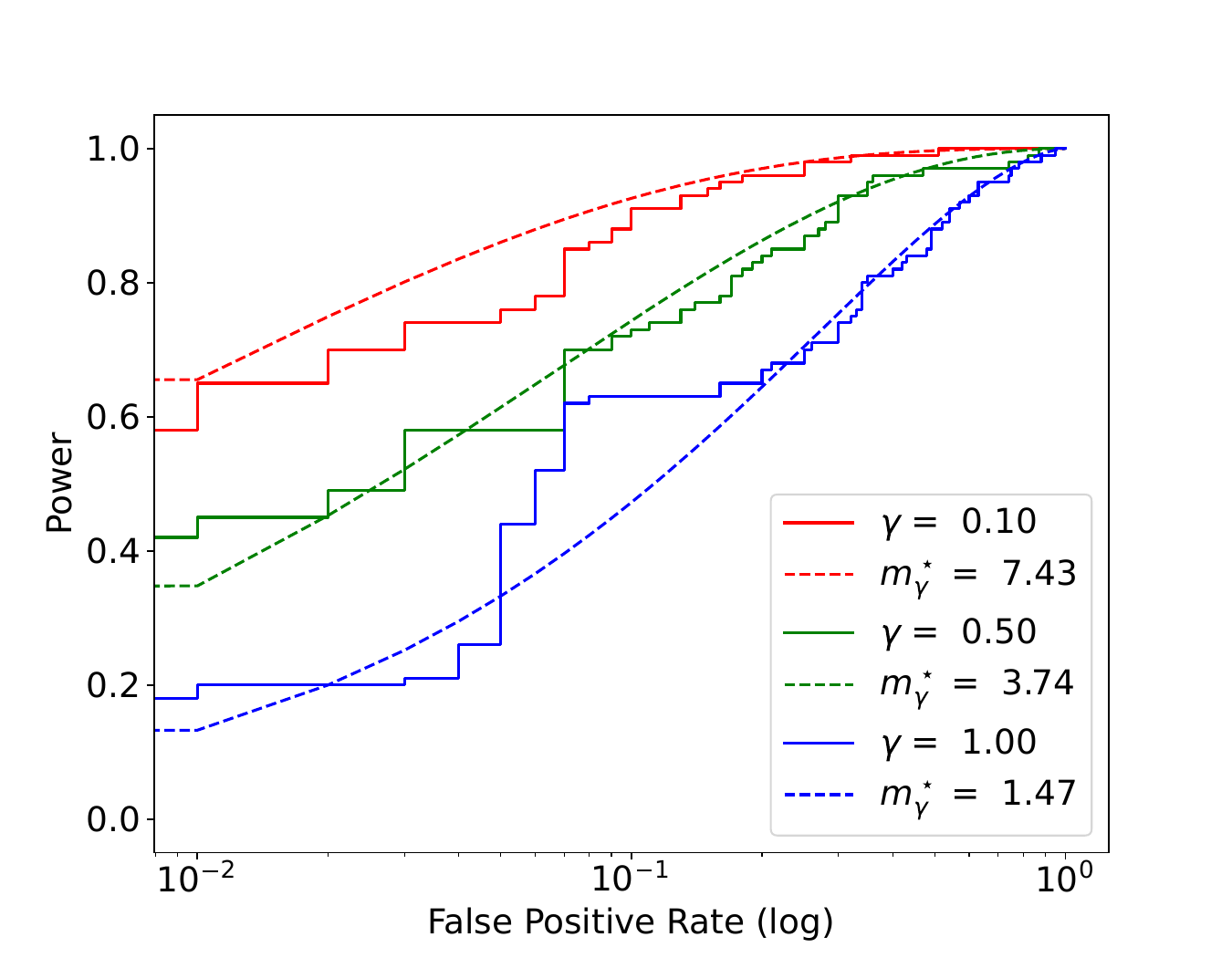}}&
\subfloat[Effect of $\rho$]{\includegraphics[width=0.3\textwidth]{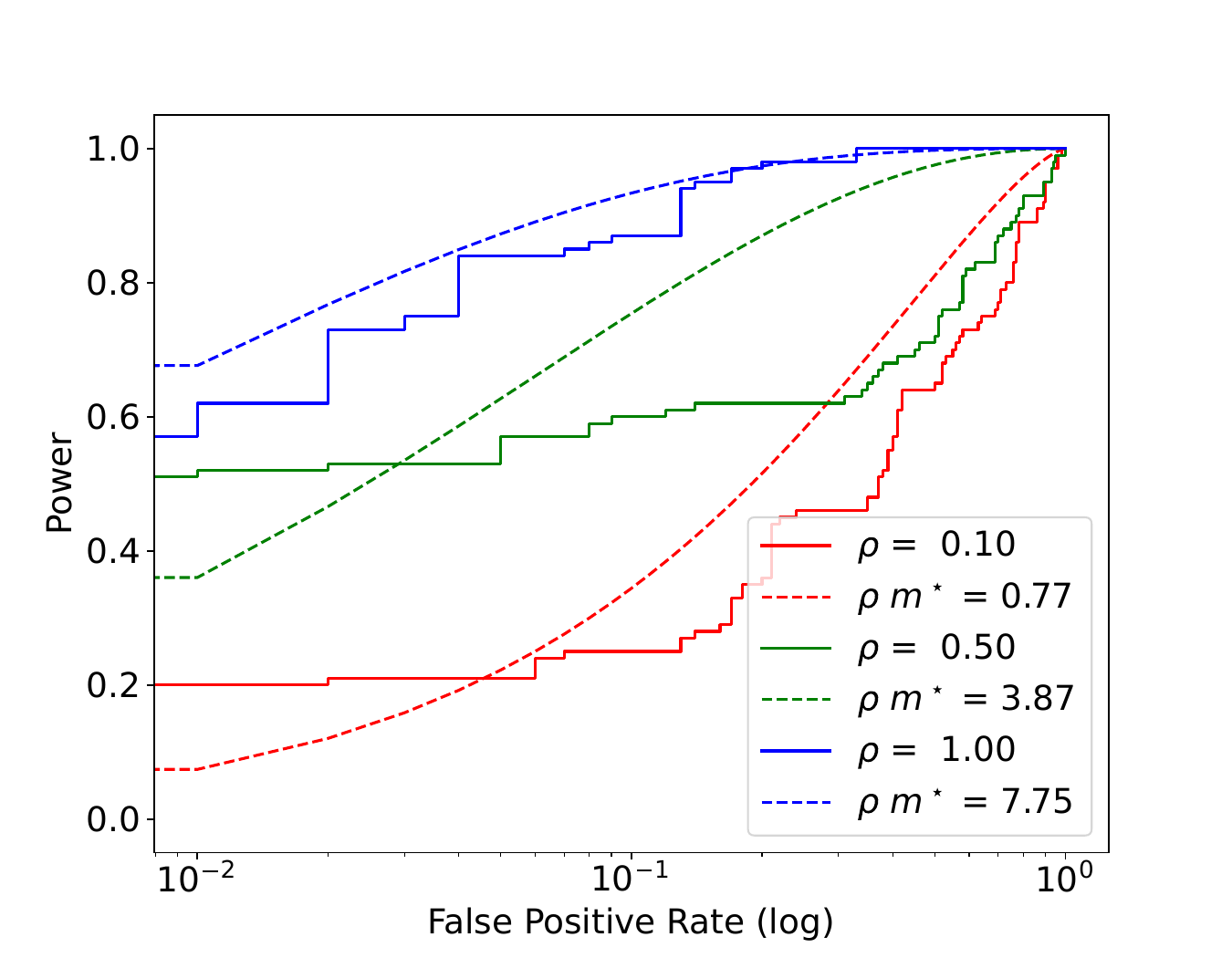}} \\
\end{tabular}\vspace*{-.5em}
\caption{Experimental validation of the theoretical results and impacts of $m^\star$, noise, and sub-sampling ratio on leakage. Dotted lines are for theoretical bounds and solid lines for empirical ones.}\label{fig:experiments}\vspace*{-1em}
\end{figure*}

%% file: sections/beyond_emp_mean.tex
\section{BEYOND EMPIRICAL MEAN: WHITE-BOX ATTACK \& CANARY SELECTION ON GRADIENT DESCENTS}\label{sec:byd_mea}

We attack supervised learning gradient descent algorithms. The main observation is that gradient descent algorithms operate by sequentially updating a parameter estimate $\theta_t$ \textit{in the direction of the empirical mean of gradients}. If an adversary has access to all the intermediates parameters $\{\theta_t\}_t$, i.e. the white-box federated learning setting, attacking gradient descent algorithms reduces to attacking the empirical mean mechanism. 

\noindent \textbf{The threat model.} We suppose that a input dataset contains $n$ examples of features and label pairs, i.e. $D \defn \{(x_i, y_i)\}_{i = 1}^n$. The goal is to find $\theta^\star \defn \arg\min_{\theta \in \real^d} \frac{1}{n} \sum_{i = 1}^{n} \ell(f_\theta(x_i), y_i)$ the best parameter which explains the dataset $D$ with respect to a loss function $\ell$. Here, we focus on gradient descent algorithms. Gradient Descent algorithms start with an initial parameter $\theta_0 \in \real^d$, and then update sequentially the parameter at each step $t$ by $\theta_t \defn \theta_{t - 1} - \eta \nabla_{\theta_{t - 1}} Q(\theta_{t- 1})$, where $\eta$ is the learning rate, and $Q(\theta_{t})$ is a quantity that aggregates the gradients on some input samples. For example, in batch gradient descent $\nabla_{\theta_{t}} Q(\theta_{t}) \defn \frac{1}{n} \sum_{i = 1}^n  \nabla_{\theta_{t}} \ell(f_{\theta_{t}}(x_i), y_i)$ is the mean of gradients on to the whole dataset. For DP-SGD~\citep{dpsgd}, $\nabla_{\theta_{t}} Q(\theta_{t}) \defn \left(\frac{1}{|B|}  \sum_{i \in B} \operatorname{Clip}_{C} \left[ \nabla_{\theta_{t}} \ell(f_{\theta_{t}}(x_i), y_i) \right] \right)   + \mathcal{N}\left( 0, \gamma^2 C^2 I_d \right),$ where $B$ is a batch uniformly sampled, $ \operatorname{Clip}_{C}(x) \defn \min\{1, C/ \|x\| \}  x$ is the clipping function, $C > 0$ is a clipping bound and $\gamma>0$ is the noise magnitude.
\setlength{\textfloatsep}{6pt}
\begin{algorithm}[t!]
   \caption{Mahalanobis canary-choosing strategy}
   \label{alg:cana_maha}
\begin{algorithmic}[1]
   \State \textbf{Input:} $\{(x_1, y_1), \dots, (x_{n_\text{r}}, y_{n_\text{r}})\}$ reference points, $\{(x_1, y_1), \dots, (x_{n_\text{c}}, y_{n_\text{c}})\}$ candidate canaries
   \State \textbf{Step1:} Estimate $\hat{\mu}_0$ and $\hat{C}_0$
   \State Initialise weights and biases of the model, i.e. $\theta^0$.
   \For{$i=1, \cdots, n_\text{r}$}
        \State Compute $g_i =\nabla_{\theta_{0}} \ell(f_{\theta_0}(x_i), y_i )$
   \EndFor
   \State Compute $\hat{\mu}_0 = \frac{1}{n} \sum_i^n g_i$ and $\hat{C}_0= \frac{1}{n} \sum_i^n g_i g_i^T$
   \State Return $(\hat{\mu}_0, \hat{C}_0)$
   \State \textbf{Step2:} Estimate the Mahalanobis score
   \For{$k=1, \cdots, n_\text{c}$}
        \State Compute $g_k =\nabla_{\theta_{0}} \ell(f_{\theta_0}(x_k), y_k )$
        \State {Estimate $m^\star_k = \| g_k - \hat{\mu}_0 \|^2_{\hat{C}_0^{-1}}$}
   \EndFor
   \State Return $(x^\star, y^\star) = \arg\max_{k =1}^{n_c} m_k^\star$
\end{algorithmic}
\end{algorithm}

We instantiate the fixed-target MI game with the gradient descent training algorithm, which takes as input the private dataset $D$ and produces the sequence $\{\theta_t \}_{t = 1}^T$. This is called the white-box federated learning setting~\citep{nasr2023tight}. We provide more details on the definition of this setting, its importance, and its relation to federated learning in Appendix~\ref{app:audt}.

\textbf{The covariance attack.} We present our covariance attack in Algorithm~\ref{alg:cov_att}. 
Given the target's gradient $g^\star_t$ and the batch gradient $g^t_\mathrm{batch} \defn \frac{\theta_{t+1} - \theta_t}{\eta}$, the covariance attack at step $t$ uses the empirical LR score of Section~\ref{sec:mia_mean} to compute the covariance score $s_t$. The attack uses the estimated empirical mean $\hat \mu_0$ and covariance $C_0$ of gradients over some reference points. To avoid the computational burden, the same estimates $(\hat \mu_0, C_0)$ are used for the attack over one epoch, i.e. one pass over the dataset. At the end of the epoch, the final score is the sum of step-wise scores $s_t$. When the inverse of the covariance matrix is well estimated, the covariance attack improves on the scalar product score, which is the state-of-the-art score in white box attacks~\citep{maddock2022canife, nasr2023tight, steinke2023privacy, andrew2023one}.

\textbf{The Mahalanobis canaries.} We present our canary selection strategy in Algorithm~\ref{alg:cana_maha}. Algorithm~\ref{alg:cana_maha} takes as input candidate ``canaries", and outputs the ``easiest" point to attack between the proposed candidates. To do so, Algorithm~\ref{alg:cana_maha} estimates the Mahalanobis score of gradients using reference samples. In addition to being an optimal strategy, Algorithm~\ref{alg:cana_maha} is the first strategy generating "in-distribution" canaries that do not hurt the accuracy of a model trained. In appendix~\ref{app:audt}, we compare our Mahalanobis canaries to the different heuristics used in the white-box literature~\citep{jagielski2020auditing,  maddock2022canife, nasr2023tight}.

\setlength{\textfloatsep}{6pt}
\begin{algorithm}[t!]
   \caption{The covariance attack}
   \label{alg:cov_att}
\begin{algorithmic}[1]
   \State \textbf{Input:} Estimated $(\hat{\mu}_0, \hat{C}_0)$ from Algorithm~\ref{alg:cana_maha}, target $z^\star$, learning rate $\eta$, batch size $b$.
   \For{$t=1, \cdots, L$}
        \State Set $g_t^\star = \nabla_{\theta_{t}} \ell(f_{\theta_t}(x^\star), y^\star)$
        \State Set $g_\text{batch}^t = (\theta_{t+1} - \theta_t)/\eta$
        \State {Compute the covariance score $s_t =$}\par  {$ (g^\star_t - \hat{\mu}_0)^T\hat{C}_0^{-1}\left (  g_\text{batch}^t  - \hat{\mu}_0 \right)$}   {$ - \frac{1}{2 b} \|g^\star_t - \hat{\mu}_0\|^2_{\hat{C}_0^{-1}} $}
   \EndFor
   \State Return $\sum_{t = 1}^L s_t$
\end{algorithmic}
\end{algorithm}

%% file: sections/experiments.tex
\vspace{-1em}\section{EXPERIMENTAL ANALYSIS}\label{sec:exp}
First, we empirically validate the theoretical analysis on synthetic data. Then, we test our covariance attack and canary selection strategy in the white-box federated learning setting on gradient descent on real datasets.   

\vspace*{-.5em}\subsection{Experiments on synthetic data}
We test: \textit{How tight is the leakage analysis in Theorem~\ref{crl:per_leak},~\ref{thm:noisy_mean}, and~\ref{thm:lr_sub}?}

\textbf{Experimental setup.} We take $n = 1000$, $\tau =5$, and thus, $d = 5000$. The data-generating distribution $\mathcal{D}$ is a $d$ dimensional Bernoulli, with parameter $p \in [0,1]^d$. The three mechanisms considered are $\mathcal{M}^\mathrm{emp}_n$, $\mathcal{M}^\gamma_n$ and $\mathcal{M}^\mathrm{sub}_{n,d}$. The adversaries chosen for each mechanism are the thresholding adversaries based on the asymptotic approximations of the LR tests. Finally, we choose three target data points in $\{0, 1\}^d$. 
(a) The \textit{easiest point to attack} $z^\star_{\mathrm{easy}} \defn \left(\ind{p_i \leq 1/2} \right)_{i = 1}^d$ is the point with the highest Mahalanobis distance with respect to $p$. 
(b) The \textit{hardest point to attack} is $z^\star_{\mathrm{hard}} \defn \left(\ind{p_i > 1/2} \right)_{i = 1}^d$, that has the binary coordinates closest to $p$. 
(c) A \textit{medium point to attack} $z^\star_{\mathrm{med}}$ is randomly sampled from the data-generating distribution $\mathrm{Bern}(p)$, for which the Mahalanobis distance and the leakage score are of orders $d$ and $\tau = d/n$. We simulate a fixed-target MI game $1000$ times for each mechanism and fixed target point. We plot the empirical ROC curve in solid lines in Figure~\ref{fig:experiments}. The theoretical trade-off functions of Equations~\eqref{eq:pow_emp},~\eqref{eq:pow_nois} and~\eqref{eq:pow_sub} are in dotted lines. Further details are in Appendix~\ref{app:exp_details}.

\textbf{Results and discussions.} Figure~\ref{fig:experiments} shows that
(a) the power of the LR test uniformly increases with an increase in  $m^\star$ and the sub-sampling ratio $\rho$, and uniformly decreases with an increase in the noise variance $\gamma^2$. (b) Figure~\ref{fig:experiments} validates that our theoretical analysis tightly captures the target-dependent hardness of MI games and the effect of privacy-preserving mechanisms on the experimental ROC curves.

\vspace*{-.5em}\subsection{Attacking in the white-box federated learning setting}\vspace*{-.5em}
We test: \textit{Does the covariance attack improve over the scalar product attack? Does the Mahalanobis leakage score explain the target-dependent hardness of MI games on real datasets? }

\textbf{Experimental setup.} We attack two models. We train a single linear layer (Linear($28 \times 28, 10$)) on Fashion MNIST (FMNIST)~\citep{fmnist}. Thus, the number of weights and biases in the model is $d_F = 7850$. We train a Convolutional Neural Net (CNN)~\citep{cnn} with three convolutional layers and a final linear layer on CIFAR10~\citep{cifar10}. 
\begin{figure}[t!]
\centering\vspace*{-1.5em}
\begin{tabular}{cc}
\hspace*{-.5em}\subfloat[Logistic reg. on FMNIST\vspace*{-1em}]{\includegraphics[width=0.48\columnwidth,trim={0.9cm 0.9cm 1.9cm 0.9cm},clip]{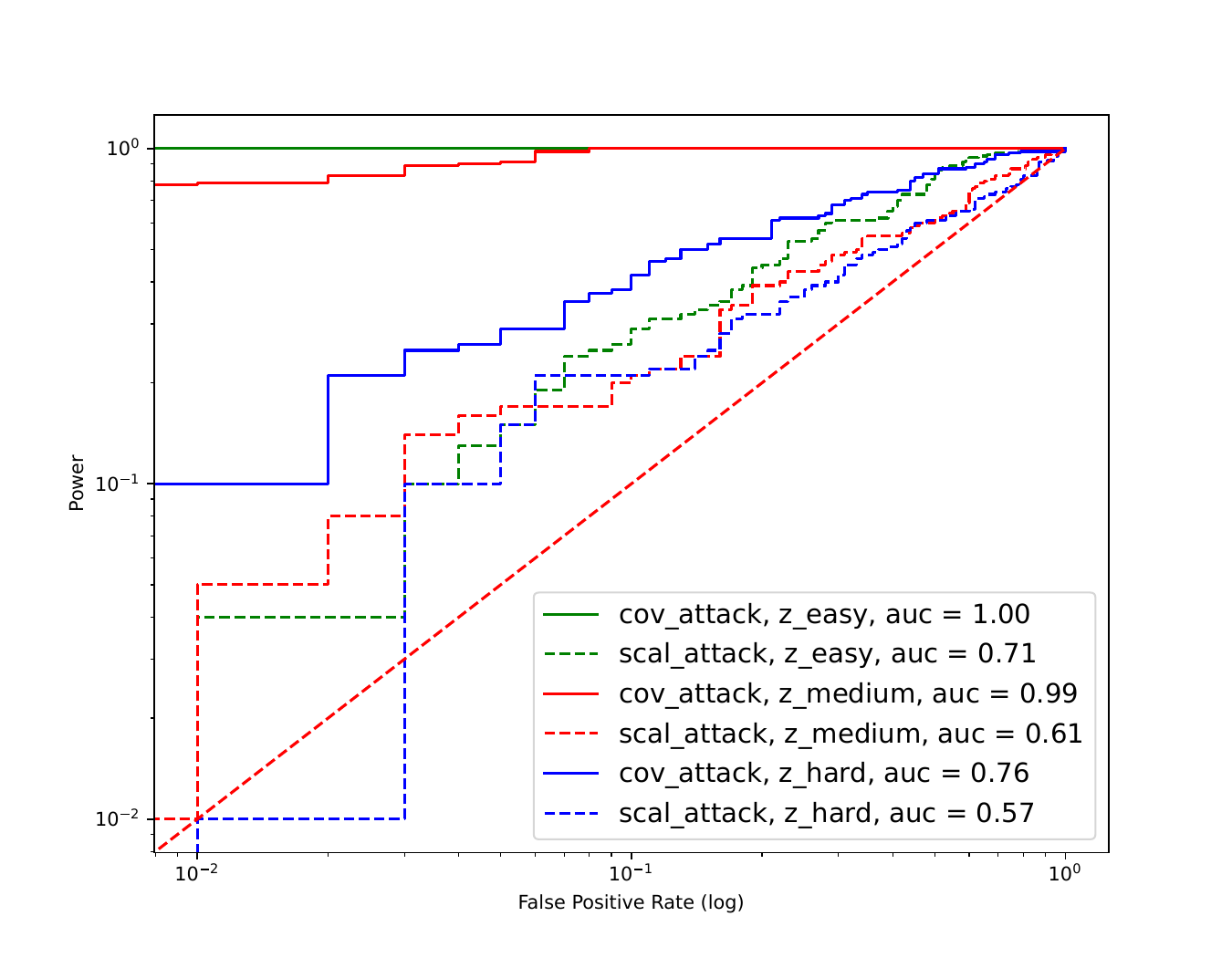}} 
&\hspace*{-.5em}\subfloat[CNN on CIFAR10\vspace*{-.5em}]{\includegraphics[width=0.48\columnwidth,trim={0.9cm 0.9cm 1.9cm 0.9cm},clip]{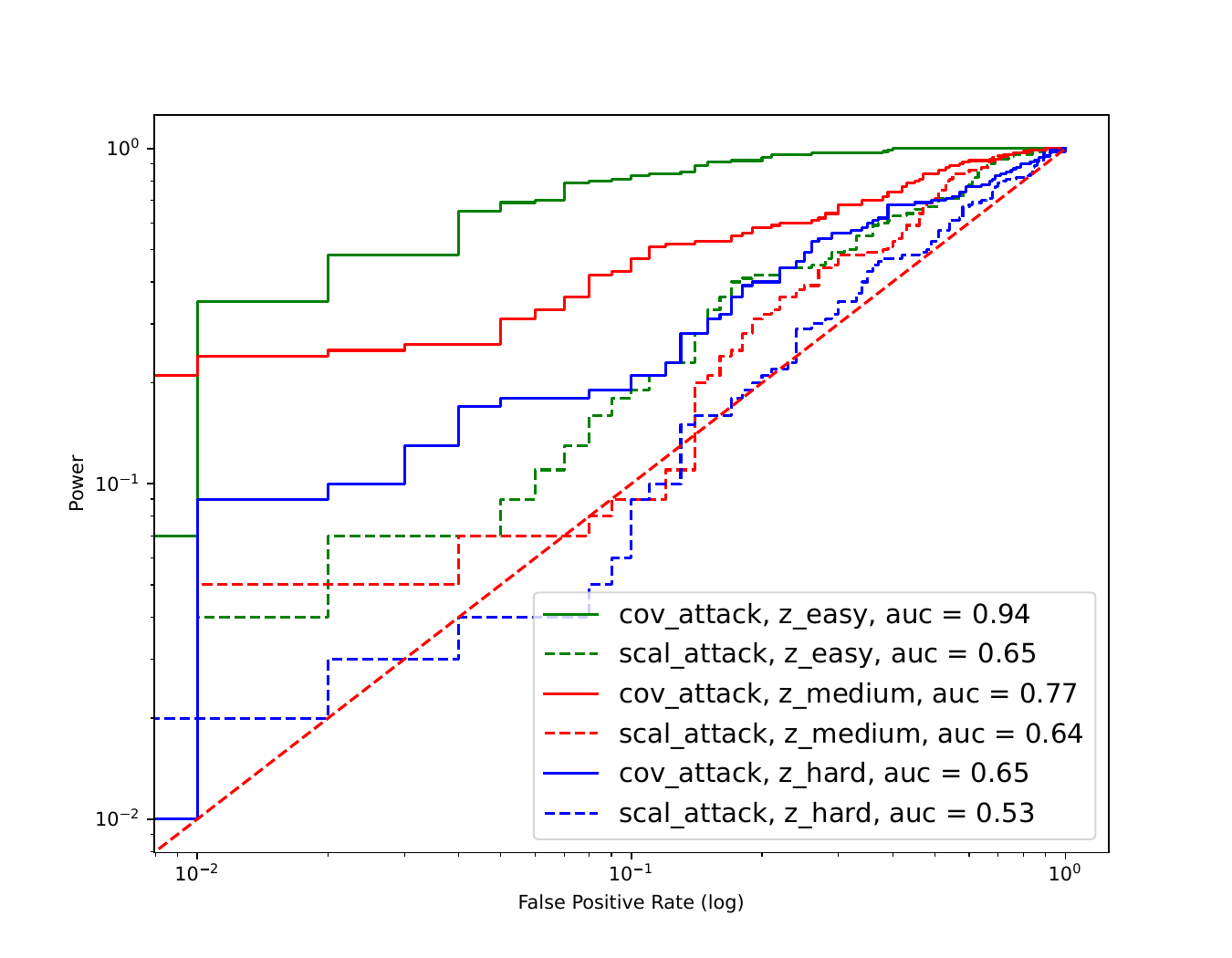}} 
\end{tabular}
\caption{Covariance and scalar product attacks.}\label{fig:experiments_white_box}
\end{figure}
The number of weights and biases in CNN is $d_C = 18786$, while the last linear layer has $d_L = 4080$ parameters. We use batched SGD with a batch size $64$, learning rate $10^{-3}$, and cross-entropy loss. We attack the models in a white-box FL setting. First, we run Algorithm~\ref{alg:cana_maha} to estimate the empirical mean and covariance, using $n_r = 1000$ reference points from training data. Then, we estimate the Mahalanobis score for every point in the training data. Finally, we chose as targets the points in the training data with the highest, medium and lowest Mahalanobis scores, i.e. $z_\text{easy}$, $z_\text{med}$ $z_\text{hard}$, respectively. We run both our covariance attack and the scalar product attack on these three points. The scalar product attack replaces the score $s_t$ in Line 5 of Algo.~\ref{alg:cov_att} with $s^\text{scal}_t = (g^\star_t)^T g^t_\text{batch}$. Both attacks are run \textit{only} on one epoch of SGD, i.e. one loop over the training data. For FMNIST, the attack is implemented with the \textit{full} gradient of the loss. For CIFAR10, we only attack the last linear layer of the model, leading to $d_L \times d_L$ covariance matrix rather than $d_C \times d_C$. This improves our attack's time and space complexity by storing and inverting a smaller matrix. It still maintains the strength of the tracing attack since $d_L$ is significantly larger than the batch size.
We show the ROC curves of the two attacks against easy, medium and hard targets of FMNIST and CIFAR10 in Fig.~\ref{fig:experiments_white_box}.

\textbf{Results and discussion.} Figure~\ref{fig:experiments_white_box} shows that (a) The covariance attack improves on the scalar product attack. (b) Points with high Mahalanobis scores are easier to attack than points with low Mahalanobis scores for both the datasets and models. (c) It is enough to run the covariance attack over one epoch and the last layer of a model. Also, the in-distribution Mahalanobis canaries, i.e. $z_\text{easy}$, leak enough to be easily identified.




%% file: sections/conclusion.tex
\vspace*{-1em}\section{CONCLUSION \& FUTURE WORK}\label{sec:conc}\vspace*{-.5em}
We study fixed-target MI games and characterise the target-dependent hardness of MI games for the empirical mean and its variants. We show that the hardness is captured exactly by the Mahalanobis distance of the target to the data-generating distribution (Table~\ref{tab:1}). Our generic analysis captures the impact of different variants, like adding Gaussian noise, sub-sampling and misspecification. Our analysis generalises different results from tracing literature and explains novel phenomena observed experimentally in MI attacks. We also provide a novel covariance attack and canary strategy.
As a perspective, it would be natural to generalise our analysis to Z-estimators and use the insights to design new black-box attacks using influence functions.



%% file: appendix/mem_inf_games_avg_rec.tex
\section{TARGET-DEPENDENT V.S AVERAGE-TARGET MI GAMES}\label{app:games}
In this section, we present the average-target MI game of~\citep{yeom2018privacy} (Algorithm~\ref{alg:avg_game}), and connect it to the fixed-target MI game (Algorithm~\ref{alg:fix_game}). We also instantiate the average-target MI game with the empirical mean mechanism and present the two main attacks from this literature: the likelihood ratio (LR) test of~\citep{sankararaman2009genomic} and the scalar product attack of~\citep{dwork2015robust}. Finally, we provide a proof of Theorem~\ref{thm:opt_adv}.

\subsection{The Average-target MI Game}

An average-target MI game (Algorithm~\ref{alg:avg_game}, Experiment 1 of~\citep{yeom2018privacy}) is a game between two entities: the Target-crafter (Algorithm~\ref{alg:targ_crafter}) and the adversary $\mathcal{A}$. The MI game runs in multiple rounds. At each round $t$, the target crafter samples a tuple $(z^\star_t, o_t, b_t)$, where $z^\star_t$ is a target point, $o_t$ is an output of the mechanism and $b_t$ is the secret binary membership of $z^\star$. The adversary $\mathcal{A}$ takes as input $(z^\star_t, o_t)$ and outputs $\hat{b}_t$ trying to reconstruct $b_t$. In an average-target MI game, the adversary attacks a different target point $z^\star_t$ at each round $t$ of the game.

\subsection{Performance Metrics in the Average-target MI Game}

An adversary $\mathcal{A}$ is a (possibly randomised) function that takes as input the pair $(z^\star, o)$ generated by the Target crafter (Algorithm~\ref{alg:targ_crafter}) and outputs a guess $\hat{b} \sim \mathcal{A}(z^\star, o)$ trying to infer $b$. The performance of $\mathcal{A}$ can be assessed either with aggregated metrics like the accuracy and the advantage, or with test-based metrics like Type I error, Type II error, and trade-off functions.

\textit{The accuracy of $\mathcal{A}$} is defined as
$\mathrm{Acc}_n(\mathcal{A}) \defn \mathrm{Pr} [\mathcal{A}(z^\star, o) = b]$, where the probability is over the generation of $(z^\star, o, b)$ using Algorithm~\ref{alg:targ_crafter} with input ($\mathcal{M}, \mathcal{D}, n$). \textit{The advantage of an adversary} is the re-centred accuracy $ \mathrm{Adv}_n(\mathcal{A}) \defn 2 \mathrm{Acc}_n(\mathcal{A}) - 1$. We can also define two errors from the hypothesis testing formulation. \textit{The Type I error}, also called the False Positive Rate, is $\alpha_n(\mathcal{A}) \defn \mathrm{Pr} \left[\mathcal{A}( z^\star, o) = 1 \mid b = 0\right]$. \textit{The Type II error}, also called the False Negative Rate, is $\beta_n(\mathcal{A}) \defn \mathrm{Pr} \left[\mathcal{A}( z^\star o) = 0 \mid b = 1\right]$. \textit{The power of the test} is $1 - \beta_n(\mathcal{A})$. An adversary can use a threshold over a score function to conduct the MI games, i.e. for $\mathcal{A}_{s, \tau}(z^\star, o) \defn \ind{s(o; z^\star) > \tau}$ where $s$ is a scoring function and $\tau$ is a threshold. We want to design scores that maximise the power under a fixed significance level $\alpha$, i.e. $\mathrm{Pow}_n(s, \alpha) \defn \max_{\tau \in T_\alpha} 1 - \beta_n(\mathcal{A}_{s, \tau})$,
where $T_\alpha \defn \{\tau \in \real: \alpha_n(\mathcal{A}_{s, \tau}) \leq \alpha\}$. $\mathrm{Pow}_n(s, \alpha)$ is also called a trade-off function.

\subsection{Connection Between the Fixed-target and the Average-target MI Games}
An adversary $\mathcal{A}$ can be regarded as an infinite collection of target-dependent adversaries $(\mathcal{A}_{z^\star})_{z^\star \in \mathcal{Z}}$, where $\mathcal{A} (z^\star, o) = \mathcal{A}_{z^\star}(o)$.

The advantage (and accuracy) of $\mathcal{A}$ is the expected advantage (and accuracy) of $\mathcal{A}_{z^\star}$, when $z^\star \sim \mathcal{D}$, i.e. $\mathrm{Adv}_n(\mathcal{A}) = \expect_{z^\star \sim \mathcal{D}} \mathrm{Adv}_n(\mathcal{A}_{z^\star})$. Thus, studying the performance metrics of an adversary under average-target MI game hides the dependence on the target, by averaging out the performance on different target points. As a consequence, using the average-target MI games for auditing privacy can hurt the performance. A gain could directly be observed by running the same attack on an ``easy to attack'' fixed-target MI game.

The optimal attack for the average-target MI game is the same optimal LR attack for the target-dependent MI game:
\begin{align*}
    \ell_n(o; z^\star) &\triangleq \log \left ( \frac{p^\mathrm{in}_n(z^\star, o)}{p^\mathrm{out}_n( z^\star, o)} \right) = \log \left ( \frac{p^\mathrm{in}_n(o \mid z^\star) p^\mathrm{in}_n(z^\star)}{p^\mathrm{out}_n(o \mid z^\star) p^\mathrm{out}_n(z^\star)}  \right) = \log \left ( \frac{p^\mathrm{in}_n(o \mid z^\star)}{p^\mathrm{out}_n(o \mid z^\star)} \right)
\end{align*}
since $p^\mathrm{in}_n(z^\star) = p^\mathrm{out}_n(z^\star) = \mathcal{D}(z^\star)$.

Thus, the same LR attack optimally solves both fixed-target and average-target MI games. The only difference is in analysing the performance of the attack, i.e. whether we average out the effect of $z^\star \sim \mathcal{D}$ in average-target MI games or we keep the dependence on the target $z^\star$, by fixing $z^\star$ in the fixed-target MI games.\vspace{-.5em}

\begin{center}
\begin{minipage}{\textwidth}
\begin{minipage}{0.45\textwidth}
\setlength{\textfloatsep}{4pt}
\begin{algorithm}[H]
	\caption{The Target-crafter}\label{alg:targ_crafter}
	\begin{algorithmic}
		\State {\bfseries Input:} Mechanism $\mathcal{M}$, Data distribution $\mathcal{D}$,  $\#$samples $n$
		\State {\bfseries Output:} $(z^\star, o, b)$, where $z^\star \in \mathcal{Z}$, $o \in \mathcal{O}$ and $b \in \{0,1\}$
		\State Build a dataset $D \sim \bigotimes\nolimits_{i = 1}^n \mathcal{D}$
		\State Sample $b \sim \mathrm{Bernoulli}\left(\frac{1}{2} \right)$
		\If{$b = 0$} 
           \State Sample $z^\star \sim \mathcal{D}$ ind. of $D$
           \Else
           \State Sample $j \sim \mathcal{U}[n]$
           \State Assign $z^\star$ to be the $i$-th element of $D$
           \EndIf
           \State Sample $o \sim \mathcal{M}(D)$ 
           \State Return $(z^\star, o, b)$
	\end{algorithmic}
\end{algorithm}
\end{minipage}\hfill
\begin{minipage}{0.5\textwidth}
\setlength{\textfloatsep}{4pt}
\begin{algorithm}[H]
   \caption{Average-target MI Game}\label{alg:avg_game}
\begin{algorithmic}
   \State {\bfseries Input:} Mechanism $\mathcal{M}$, Data distribution $\mathcal{D}$,  $\#$samples $n$, Adversary $\mathcal{A}$, Rounds~$T$
   \State {\bfseries Output:} A list $L \in \{0, 1\}^T$, where $L_t = 1$ if the adversary succeeds at step $t$.
   \State Initialise a empty list $L$ of length $T$
   \For{t = 1, \dots, T}
   \State Sample $(z^\star_t, o_t, b_t) \sim$ Target-crafter($\mathcal{M}$, $\mathcal{D}$, n)
   \State Sample $\hat{b}_t \sim \mathcal{A}(z^\star_t, o_t)$
   \State Set $L_t \leftarrow \ind{b_t = \hat{b}_t}$
   \EndFor
   \State Return $L$\medskip
\end{algorithmic}
\end{algorithm}
\end{minipage}\vspace*{-.4em}
\end{minipage}
\end{center}

\subsection{The Likelihood Ratio Test for Bernoulli Empirical Mean Average-target MI Game}\label{ssec:2:sankaraman}
We revisit results from~\citep{sankararaman2009genomic}. In \citep{sankararaman2009genomic}, the MI game is instantiated with the empirical mean mechanism denoted by  $\mathcal{M}^\mathrm{emp}_n$. The mechanism $\mathcal{M}^\mathrm{emp}_n$ takes as input a dataset of size $n$ of $d$-dimensional
points, i.e. $D = \{Z_1, \dots, Z_n \} \in (\real^d)^n $, and outputs the exact empirical mean $\hat{\mu}_n \defn \frac{1}{n} \sum_{i=1}^n Z_i \in \real^d $.

\noindent \textbf{Assumptions on the data generating distribution and asymptotic regime.} \cite{sankararaman2009genomic} supposes that the data-generating distribution $\mathcal{D}$ is colon-wise independent Bernoulli distributions, i.e. $\mathcal{D} \defn \bigotimes_{j = 1}^{d} \mathrm{Bernoulli}(\mu_j)$, with $\mu_j \in [a, 1 - a]$ for some $a \in (0,1/2)$. We denote by $\rightsquigarrow$ convergence in distribution, i.e. A sequence of random variables $X_n \rightsquigarrow X$ if and only if $\mathrm{Pr}(X_n \leq x) \rightarrow \mathrm{Pr}(X \leq x)$ for all $x$. Let $\Phi$ represent the Cumulative Distribution Function (CDF) of the standard normal distribution, i.e. $\Phi(\alpha) \defn \frac{1}{\sqrt{2 \pi}} \int_{- \infty}^\alpha e^{- t^2/2} \dd t$ for $\alpha \in \real$. 
\cite{sankararaman2009genomic} studies the \textit{asymptotic behaviour of the LR test}, when both the sample size $n$ and the dimension $d$ tend to infinity such that $d/n = \tau > 0$. 

The analysis of \citep{sankararaman2009genomic} starts by showing that the exact formula of the LR score, at output $o = \hat{\mu}_n$ and target $z^\star$ is 
\begin{equation}\label{eq:lr_bern}
    \ell_n(\hat{\mu}_n, z^\star) =  \sum_{j = 1}^{d} z^\star_j \log\left ( \frac{\hat{\mu}_{n, j}}{\mu_j } \right ) + (1 - z^\star_j ) \log\left ( \frac{1 - \hat{\mu}_{n, j} }{1 - \mu_j} \right ).
\end{equation}

As $d$ and $n$ tend to infinity such that $d/n = \tau$, \cite{sankararaman2009genomic} shows that the LR score converges in distribution to 
\begin{equation}\label{eq:avg_h0}
    \ell_n(\hat{\mu}_n, z^\star) \rightsquigarrow^{H_0} \mathcal{N}\left(-\frac{1}{2}  \tau, \tau \right)
\end{equation}
under $H_0$
and converges to 
\begin{equation}\label{eq:avg_h1}
    \ell_n(\hat{\mu}_n, z^\star) \rightsquigarrow^{H_1} \mathcal{N}\left(\frac{1}{2}  \tau, \tau \right)
\end{equation}
under $H_1$.

The asymptotic distribution of the LR score helps to provide the asymptotic trade-off of the optimal LR attacker. Specifically, the main result (Section T2.1 in~\citep{sankararaman2009genomic}) is that
\begin{equation}\label{eq:avg_pow}
\Phi^{-1}(1 - \alpha) + \Phi^{-1}(1 - \beta) \approx \sqrt{d/n}
\end{equation}

where $\alpha$ is the optimal Type I error, $\beta$ is the optimal Type II error and $\Phi$ represents the Cumulative Distribution Function (CDF) of the standard normal distribution, i.e. $\Phi(\alpha) \defn \frac{1}{\sqrt{2 \pi}} \int_{- \infty}^\alpha e^{- t^2/2} \dd t$ for $\alpha \in \real$. This trade-off between $\alpha$ and $\beta$ shows that the MI game gets easier, as $d/n$ gets bigger.

\noindent \textbf{Connection to our results.} Our results in Section~\ref{sec:mia_mean} are a target-dependent version of the results of~\citep{sankararaman2009genomic}. Also, our analysis generalises the analysis of~\citep{sankararaman2009genomic} beyond Bernoulli distributions to any distribution with finite 4-th moment. Specifically, Lemma~\ref{thm:asymp_lr} shows that, in the target-dependent MI game, 

- Under $H_0$: $$\ell_n(\hat{\mu}_n; z^\star, \mu, C_\sigma )  \rightsquigarrow \mathcal{N}\left(-\frac{1}{2}  m^\star, m^\star \right)$$

- Under $H_1$: $$\ell_n(\hat{\mu}_n; z^\star, \mu, C_\sigma )  \rightsquigarrow \mathcal{N}\left(\frac{1}{2}  m^\star, m^\star \right)$$
where the convergence is a convergence in distribution, such that $d,n \rightarrow \infty$, while $d/n = \tau$. 

This is a target-dependent version of Equations~\eqref{eq:avg_h0} and~\eqref{eq:avg_h1}. Also, Equation~\eqref{eq:pow_emp} of Theorem~\ref{crl:per_leak} is a target-dependent version of Equation~\eqref{eq:avg_pow}. Thus, our results retrieve the average-case results by observing that, since $\expect_{z^\star \sim \mathcal{D}}\left[ \| z^\star - \mu \|^2_{C_\sigma^{-1}} \right] = d$, we have $\expect_{z^\star \sim \mathcal{D}}\left[ m^\star\right] = \lim_{n,d} \frac{d}{n} =\tau$.

\subsection{The Scalar Product Attack for Average-target MI game}\label{sec:dwork}

\cite{dwork2015robust} proposes a scalar product attack for tracing the empirical mean that thresholds over the score $$s^\mathrm{scal}(\hat{\mu}_n, z^\star; z^\mathrm{ref})  \defn (z^\star - z^\mathrm{ref})^T \hat{\mu}_n,$$ where $z^\mathrm{ref}$ is one reference point. The intuition behind this attack is to compare the target-output correlation $(z^\star)^T \hat\mu_n$ with a reference-output correlation $(z^\mathrm{ref}) ^T \hat\mu_n$. The analysis of~\citep{dwork2015robust} shows that with only one reference point $z^\mathrm{ref} \sim \mathcal{D}$, and even for noisy estimates of the mean, the attack is able to trace the data of some individuals in the regime $d \sim n^2$. 

Informally, the analysis of \citep{dwork2015robust} considers a scalar product attack, taking as input any \textit{$1/2$-accurate estimate} $\hat{\mu}$ of a dataset $D \in (\{-1, 1\}^d)^n $ of dimension $d = O\left(n^2 \log(1/\delta) \right)$, a target point $z^\star \in \{-1, 1\}^d$ and \textit{only a single reference sample} $z^\mathrm{ref} \in \{-1, 1\}^d$. The data-generating distribution is assumed to be chosen from a strong class of distributions $\mathcal{P}$. Then,

\begin{itemize}
    \item If $z^\star$ is IN the dataset $D$, then $$\mathrm{Pr} \left\{s^\mathrm{scal}(\hat{\mu}_n, z^\star; z^\mathrm{ref}) > \tau\right\} \geq \Omega(1/n). $$
    \item If $z^\star$ is Out of the dataset $D$, then $$\mathrm{Pr} \left\{s^\mathrm{scal}(\hat{\mu}_n, z^\star; z^\mathrm{ref}) < \tau\right\} \geq 1 - \delta. $$
\end{itemize}
for a carefully chosen threshold $\tau = O(\sqrt{d \log(1/\delta)})$.

The condition of $1/2$-accuracy is a weak condition, compared to the exact empirical mean attack \citep{sankararaman2009genomic}. The price of the weak notion of accuracy is that the attack is only guaranteed for $d \gtrapprox n^2$, whereas the exact attack of \citep{sankararaman2009genomic} is able to trace for $d \approx n$.

\subsection{Proof of Theorem~\ref{thm:opt_adv}}\label{app:proof_opt_adv}

\begin{reptheorem}{thm:opt_adv}[Characterising Optimal Adversaries]

(a) $\mathcal{A}_{\mathrm{Bayes}, z^\star}$ is the adversary that maximises the advantage (and accuracy), i.e. for any adversary $\mathcal{A}_{z^\star}$, we have that $\mathrm{Adv}_n(\mathcal{A}_{\mathrm{Bayes}, z^\star}) - \mathrm{Adv}_n(\mathcal{A}_{z^\star})  \geq 0$.

(b)  The advantage of the optimal Bayes adversary is  
    \begin{equation*}
        \mathrm{Adv}_n(\mathcal{A}_{\mathrm{Bayes}, z^\star}) = \TV{p^\mathrm{out}_n(. \mid z^\star)}{p^\mathrm{in}_n(. \mid z^\star)},
    \end{equation*}
    where $\mathrm{TV}$ is the total variation distance.
    
(c) For every $\alpha \in [0,1]$, the log-likelihood score is the score that maximises the power under significance $\alpha$, i.e. for any $\alpha$ and any score function $s$, $\mathrm{Pow}_n(\ell, \alpha, z^\star)  \geq \mathrm{Pow}_n(s, \alpha, z^\star)$.
\end{reptheorem}

\begin{proof}
To prove (a), we observe that the log-likelihood adversary with threshold $\tau = 0$ is exactly the Bayes optimal classifier. Specifically, since $Pr(b= 0) = Pr(b = 1)$, we can rewrite the log-likelihood as the $\ell_n(o; z^\star) = \log \left ( \frac{\mathrm{Pr}(b = 1 \mid o, z^\star)}{\mathrm{Pr}(b = 0 \mid o, z^\star)} \right)$. Thus thresholding with $0$ gives exactly the Bayes optimal classifier, which has the highest accuracy among all classifiers.

For (b), we observe that
\begin{align*}
    \mathrm{Adv}_n(\mathcal{A}_{\mathrm{Bayes}, z^\star}) &=  \mathrm{Pr}(\ell_n(o; z^\star) \leq 0 \mid b = 0) - \mathrm{Pr}(\ell_n(o; z^\star) \leq 0 \mid b = 1)\\
    &= p^\mathrm{out}_n(O \mid z^\star) - p^\mathrm{in}_n(O \mid z^\star)
\end{align*}
where $O \defn \{o \in \mathcal{O}: p^\mathrm{out}_n(o \mid z^\star) \geq p^\mathrm{in}_n(o \mid z^\star) \}$.

The last equation is exactly the definition of the $\TV{p^\mathrm{out}_n(. \mid z^\star)}{p^\mathrm{in}_n(. \mid z^\star)}$.

Finally, (c) is a direct consequence of the Neyman-Pearson lemma.
\end{proof}

%% file: appendix/proof_asymp_dist_lr.tex
\section{MISSING PROOFS}\label{app:proof}
First, we present some classic background results from asymptotic statics, the Edgeworth asymptotic expansion (Theorem~\ref{thm:edgeworth} and Linderberg-Feller theorem~\ref{thm:lind_feller}. Then, we provide the missing proofs of Theorem~\ref{thm:asymp_lr}, Corollary~\ref{crl:per_leak}, Theorem~\ref{thm:lr_sub} and Theorem~\ref{thm:misspec}.

\subsection{Background on Asymptotic Statics}
A sequence of random variables $X_n$ is said to converge in distribution to a random variable $X$, i.e. $X_n \rightsquigarrow_n X$ if $\mathrm{Pr}(X_n \leq x) \rightarrow \mathrm{Pr}(X \leq x)$, for every $x$ at which the limit distribution $x \rightarrow \mathrm{Pr}(X \leq x)$ is continuous.

A sequence of random variables $X_n$ is said to converge in probability to $X$ if for every  $\epsilon > 0$, $\mathrm{Pr}(\| X_n - X \| > \epsilon) \rightarrow 0$, denoted by $X_n \rightarrow^\mathrm{P} X$.

We recall that continuous mappings preserve both convergences. 

A sequence of random variable $(X_n)$ is called uniformly tight if: for every $\epsilon$, $\exists M>0$, such that $\sup_{n} \mathrm{Pr}(\| X_n\| > M) < \epsilon$.

\begin{theorem}[Prohorov's theorem]\label{thm:proh} Let $X_n$ be a random vector in $\real^d$.

1. If $X_n \rightsquigarrow X$, for some $X$, then the sequence $(X_n)$ is uniformly tight;

2. If $(X_n)$ is uniformly tight, then there exits a subsequence with $X_{n_j} \rightsquigarrow X$ as $j \rightarrow \infty$ for some $X$.
\end{theorem}

We also recall the stochastic $o_p$ and $O_p$ notation for random variables.

\begin{definition}[Stochastic $o_p$ and $O_p$]
We say that $X_n = o_p(R_n)$ if $X_n = Y_n R_n$ and $Y_n \rightarrow^\mathrm{P} 0$

We say that $X_n = O_p(R_n)$ if $X_n = Y_n R_n$ and $Y_n = O_p(1)$ where $O_p(1)$ denotes a sequence that is uniformly tight (also called bounded in probability).
\end{definition}

The following lemma is used to get Taylor expansions of random variables.

\begin{lemma}[Lemma 2.12 in~\cite{van2000asymptotic}]\label{lem:tayl_exp}

Let $R$ be a function on $\real^k$, such that $R(0) = 0$. Let $X_n = o_p(1)$.

Then, for every $p > 0$,

(a) if $R(h) = o(\|h\|^p)$ as $h \rightarrow 0$, then $R(X_n) = o_p(\|X_n\|^p)$;

(b) if $R(h) = O(\| h \|^p)$ as $h \rightarrow 0$, then $R(X_n) = O_p(\|X_n\|^p)$.
\end{lemma}

The Lindeberg-Feller theorem is the simplest extension of the classical central
limit theorem (CLT) and is applicable to independent (and not necessarily identically distributed) random variables with finite variances.

\begin{theorem}[Lindeberg-Feller CLT]\label{thm:lind_feller}
Let $Y_{n,1}, \dots Y_{n, d_n}$ be independent random vectors with finite variances such that 

1. for every $\epsilon >0$, $\sum_{j = 1}^{d_n} \expect \left[ \| Y_{n,i} \|^2  \ind{\| Y_{n,i} \| > \epsilon} \right] \rightarrow 0$, 

2. $\sum_{j = 1}^{d_n} \expect \left[ Y_{n,i}  \right] \rightarrow \mu$,

3. $\sum_{j = 1}^{d_n} \mathrm{Cov} \left[ Y_{n,i}  \right] \rightarrow \Sigma$.

Then $\sum_{j= 1}^{d_n} Y_{n,j} \rightsquigarrow \mathcal{N} \left ( \mu, \Sigma \right)$
\end{theorem}

Finally, the last result from asymptotic statistics is the Edgeworth asymptotic expansion in the CLT.

\begin{theorem}[Edgeworth expansion, Theorem 15 of Chapter 7 in~\cite{petrov2012sums}]\label{thm:edgeworth}
    Let $Z_1, \dots, Z_n$ sampled i.i.d from $\mathcal{D}$, where $\mathcal{D}$ has a finite absolute moment of $k$-th order, i.e. $ \expect[| X_1|^k] < \infty$. Let $d_n$ be the density of the centred normalised mean $\frac{1}{\sigma \sqrt{n}} \sum_{i = 1}^n X_i$, then
    
    \begin{equation*}
        d_n(x) = \frac{1}{\sqrt{2 \pi}} e^{- x^2/2} + \sum_{\nu = 1}^{k - 2} \frac{q_\nu(x)}{n^{\nu/2}} + o \left ( \frac{1}{n^{(k - 2)/2}} \right)
    \end{equation*}
    
uniformly in x, where $q_v(x)$ are related to the Hermite Polynomials.
\end{theorem}

\subsection{Proof of the Asymptotic Distribution of the LR Scores,  Lemma~\ref{thm:asymp_lr}}\label{sec:prf_thm_mean}

\begin{lemma}[Asymptotic distribution of the LR score]\label{thm:asymp_lr}



Using an Edgeworth asymptotic expansion of the LR score and a Lindeberg-Feller central limit theorem, we show that

- Under $H_0$: $$\ell_n(\hat{\mu}_n; z^\star, \mu, C_\sigma )  \rightsquigarrow \mathcal{N}\left(-\frac{1}{2}  m^\star, m^\star \right)$$

- Under $H_1$: $$\ell_n(\hat{\mu}_n; z^\star, \mu, C_\sigma )  \rightsquigarrow \mathcal{N}\left(\frac{1}{2}  m^\star, m^\star \right)$$

where the convergence is a convergence in distribution, such that $d,n \rightarrow \infty$, while $d/n = \tau$. 
We call $m^\star$ the \emph{leakage score of target datum $z^\star$}.
\end{lemma}

\begin{proof}
We have that $\hat{\mu}_n = \frac{1}{n} \sum_{i = 1}^n Z_i$, where $Z_i = (Z_{i,j})_{j = 1}^{d_n} \in \real^{d_n}$ and $Z_i \sim^{\mathrm{i.i.d}} \mathcal{D} = \bigotimes_{j = 1}^{d_n} \mathcal{D}_j$.
    
    Each distribution $\mathcal{D}_j$ has mean $\mu_j$ and variance $\sigma_j^2$.

    We denote $\hat{\mu}_n = (\hat{\mu}_{n, j})_{j = 1}^{d_n}$, where $\hat{\mu}_{n, j} = \frac{1}{n} \sum_{i = 1}^n Z_{i, j}$.
    
    \underline{\textbf{Step 1: Rewriting the LR score}}

    Let $j \in [1, d_n]$.

    \underline{\textit{Under $H_0$,}} we can re-write 
    \begin{equation*}
        \hat{\mu}_{n, j} = \mu_j + \frac{\sigma_j}{\sqrt{n}} \hat{Z}_{n, j},
    \end{equation*}
     where 
     \begin{align*}
         \hat{Z}_{n, j} &\defn \sqrt{n} \left ( \frac{\hat{\mu}_{n, j} - \mu_j}{\sigma} \right)\\
         &= \frac{1}{\sqrt{n}} \sum_{i = 1}^n \frac{Z_{i, j} - \mu_j}{\sigma_j}.
     \end{align*}
     Since $(Z_{i,j})_{i = 1}^n$ are i.i.d from $\mathcal{D}_j$, using the CLT, $\hat{Z}_{n, j} \rightsquigarrow_{n \rightarrow \infty} \mathcal{N}(0,1)$.

     Let $d_{n, j}$ be the density function of $\hat{Z}_{n,j}$.

     The density $p^\mathrm{out}_{n,j}$ of $\hat{\mu}_{n, j}$ under $H_0$ can be written as
    \begin{equation*}
        p^\mathrm{out}_{n,j}(x; z^\star_j, \mu_j, \sigma_j) = \frac{\sqrt{n}}{\sigma_j} d_{n, j}\left [ \frac{\sqrt{n}}{\sigma_j} (x - \mu_j) \right]
    \end{equation*}

    \underline{\textit{Under $H_1$,}} we can re-write 
    \begin{align*}
        \hat{\mu}_{n, j} &= \frac{1}{n} z^\star_j + \frac{n - 1}{n} \left( \mu_j + \frac{\sigma_j}{\sqrt{n - 1}} \hat{Z}_{n - 1, j} \right)\\
        &= \mu_j + \frac{1}{n} \left (z^\star_{j} - \mu_j \right) + \frac{\sigma_j \sqrt{n - 1}}{n} \hat{Z}_{n - 1, j}
    \end{align*}


     The density $p^\mathrm{in}_{n,j}$ of $\hat{\mu}_{n, j}$ under $H_1$ can be written as
    \begin{equation*}
        p^\mathrm{in}_{n,j}(x; z^\star_j, \mu_j, \sigma_j) = \frac{n}{\sigma_j \sqrt{n - 1}} d_{n - 1, j}\left [ \frac{n}{\sigma_j \sqrt{n - 1}} \left (x - \mu_j - \frac{1}{n} \left (z^\star_{j} - \mu_j \right)  \right) \right]
    \end{equation*}

    The LR score is 
    \begin{align*}
        \ell_n(\hat{\mu}_n; z^\star, \mu, C_\sigma) &= \sum_{j = 1}^{d_n} \log\left ( \frac{p^\mathrm{in}_{n,j}(\hat{\mu}_{n,j}; z^\star_j, \mu_j, \sigma_j) }{p^\mathrm{out}_{n,j}(\hat{\mu}_{n,j}; z^\star_j, \mu_j, \sigma_j) } \right ) \notag\\
        &= \sum_{j = 1}^{d_n} -\frac{1}{2} \log \left ( 1 - \frac{1}{n} \right) + \log\left ( \frac{d_{n - 1,j} \left( \delta^\mathrm{in}_{n,j} \right)}{d_{n,j} \left( \delta^\mathrm{out}_{n,j} \right)} \right )
    \end{align*}
    where 
    \begin{align*}
        \delta^\mathrm{out}_{n,j} &\defn \frac{\sqrt{n}}{\sigma_j} \left (\hat{\mu}_{n,j} - \mu_j  \right ) \\
        \delta^\mathrm{in}_{n,j} &\defn \frac{n}{\sqrt{n - 1}\sigma_j} \left (\hat{\mu}_{n,j} - \mu_j + \frac{1}{n} \left (\mu_j - z^\star_j \right)  \right )
    \end{align*}    
\underline{\textbf{Step 2: Asymptotic expansion of the LR score}}

Using Lemma~\ref{lem:asymp_exp}, we have
\begin{align*}
    \log\left ( \frac{d_{n - 1,j} \left( \delta^\mathrm{in}_{n,j} \right)}{d_{n,j} \left( \delta^\mathrm{out}_{n,j} \right)} \right ) &= \frac{1}{2} \left( \left(\delta^\mathrm{out}_{n,j} \right)^2 -   \left(\delta^\mathrm{in}_{n,j} \right)^2\right) + \frac{\lambda_3 \left (\mu_j - z^\star_j \right)}{n \sigma_j} R_{n,j} + o_{p}\left ( \frac{1}{n} \right)
\end{align*}
Let $Y_{n, j} \defn \frac{1}{2} \left( \left(\delta^\mathrm{out}_{n,j} \right)^2 -   \left(\delta^\mathrm{in}_{n,j} \right)^2\right) + \frac{\lambda_3 \left (\mu_j - z^\star_j \right)}{n \sigma_j} R_{n,j}$.

We remark that we need an expansion up to $o_p\left(\frac{1}{n} \right)$, since $d_n/ n = \tau + o(1)$.

Thus 

\begin{align}\label{eq:asymp_lr}
    \ell_n(\hat{\mu}_n; z^\star, \mu, C_\sigma) &= \sum_{j = 1}^{d_n} -\frac{1}{2} \log \left ( 1 - \frac{1}{n} \right) + \log\left ( \frac{d_{n - 1,j} \left( \delta^\mathrm{in}_{n,j} \right)}{d_{n,j} \left( \delta^\mathrm{out}_{n,j} \right)} \right ) \notag\\
    &= \sum_{j = 1}^{d_n} \left(\frac{1}{2 n} +  Y_{n, j} + o_{p}\left ( \frac{1}{n} \right) \right) \notag\\
    &= \frac{\tau}{2} + o_p(1) + \sum_{j = 1}^{d_n} Y_{n, j} 
\end{align}

because $\frac{d_n}{n} = \tau + o(1)$.

\underline{\textbf{Step3: Concluding using the Lindeberg-Feller CLT
}}

\underline{\textit{Under $H_0$:}}

Using Lemma~\ref{lem:exp_comp}, $\expect_0 [ Y_n,j ] = - \frac{1}{2n} - \frac{(z_j^\star - \mu_j)^2}{2 n \sigma_j^2} + o\left (\frac{1}{n} \right)$ and $V_0 [ Y_n,j ] = \frac{(z_j^\star - \mu_j)^2}{n \sigma_j^2} + o\left (\frac{1}{n} \right)$.

Since $\sum_{j = 1}^{d_n} \frac{(z_j^\star - \mu_j)^2}{n \sigma_j^2}  = \frac{\| z^\star - \mu \|^2_{C_\sigma^{-1}}}{n}$, we get:
\begin{itemize}
    \item $\sum_{j = 1}^{d_n} \expect_0 [ Y_n,j ] \rightarrow -\frac{\tau}{2} - \frac{m^\star}{2}$

    \item $\sum_{j = 1}^{d_n} V_0 [ Y_n,j ] \rightarrow m^\star$
\end{itemize}

Using Lemma~\ref{lem:lind_cond}, we have that $Y_{n,j}$ verify the Lindeberg-Feller condition, i.e.

\begin{equation*}
    \sum_{j = 1}^{d_n} \expect_0 \left [ Y_{n, j}^2 \ind{|Y_{n,j} | > \epsilon }  \right] \rightarrow 0
\end{equation*}

for every $\epsilon >0$.

We conclude using the Lindeberg-Feller CLT (Theorem~\ref{thm:lind_feller}) that $\sum_{j = 1}^{d_n} Y_{n,j} \rightsquigarrow \mathcal{N}\left ( -\frac{\tau}{2} - \frac{m^\star}{2},  m^\star\right) $, and thus

\begin{equation*}
    \ell_n(\hat{\mu}_n; z^\star, \mu, C_\sigma) 
    \rightsquigarrow \mathcal{N}\left ( - \frac{m^\star}{2},  m^\star\right) 
\end{equation*}

Similarly, \underline{\textit{Under $H_1$:}}

Using Lemma~\ref{lem:exp_comp}, $\expect_1 [ Y_n,j ] = - \frac{1}{2n} + \frac{(z_j^\star - \mu_j)^2}{2 n \sigma_j^2} + o\left (\frac{1}{n} \right)$ and $V_1 [ Y_n,j ] = \frac{(z_j^\star - \mu_j)^2}{n \sigma_j^2} + o\left (\frac{1}{n} \right)$.

We get:
\begin{itemize}
    \item $\sum_{j = 1}^{d_n} \expect_1 [ Y_{n,j} ] \rightarrow -\frac{\tau}{2} + \frac{m^\star}{2}$

    \item $\sum_{j = 1}^{d_n} V_1 [ Y_{n,j} ] \rightarrow m^\star$
\end{itemize}

Using Lemma~\ref{lem:lind_cond}, we have that $Y_{n,j}$ verify the Lindeberg-Feller condition, i.e.

\begin{equation*}
    \sum_{j = 1}^{d_n} \expect_1 \left [ Y_{n, j}^2 \ind{|Y_{n,j} | > \epsilon }  \right] \rightarrow 0
\end{equation*}

for every $\epsilon >0$.

We conclude using the Lindeberg-Feller CLT (Theorem~\ref{thm:lind_feller}) that $\sum_{j = 1}^{d_n} Y_{n,j} \rightsquigarrow \mathcal{N}\left ( -\frac{\tau}{2} + \frac{m^\star}{2},  m^\star\right) $, and thus

\begin{equation*}
    \ell_n(\hat{\mu}_n; z^\star, \mu, C_\sigma) 
    \rightsquigarrow \mathcal{N}\left (\frac{m^\star}{2},  m^\star\right) 
\end{equation*}

\end{proof}

\begin{remark}~\label{rmk:asymp}
    Expanding $\frac{1}{2} \left( \left(\delta^\mathrm{out}_{n,j} \right)^2 -   \left(\delta^\mathrm{in}_{n,j} \right)^2\right)$, taking the sum from $j = 1$ until $d_n$, we get that
    \begin{align*}
    \ell_n(\hat{\mu}_n; z^\star, \mu, C_\sigma) \sim &(z^\star - \mu)^T C_\sigma^{-1}(\hat{\mu}_n - \mu) 
    - \frac{1}{2n} \|z^\star - \mu\|^2_{C_\sigma^{-1}}
\end{align*}

Let $X_n \defn (z^\star - \mu)^T C_\sigma^{-1}(\hat{\mu}_n - \mu) 
    - \frac{1}{2n} \|z^\star - \mu\|^2_{C_\sigma^{-1}}$.

This asymptotic representation of the LR test is useful to get directly the means and variances of the limit distribution of the LR test. Specifically, since $ \expect_0( \hat{\mu}_n) = \mu $, $ \expect_1( \hat{\mu}_n) = \frac{n - 1}{n} \mu + \frac{1}{n} z^\star $ and $\Var_0(\hat{\mu}_n) = \Var_1(\hat{\mu}_n) = C_\sigma$, we get that 
\begin{align*}
    \expect_0 \left [ X_n \right] &=  - \frac{1}{2n} \|z^\star - \mu\|^2_{C_\sigma^{-1}}\\
    \expect_1 \left [X_n \right] &=  \frac{1}{2n} \|z^\star - \mu\|^2_{C_\sigma^{-1}}\\
    \Var_0 \left [ X_n\right] &= \Var_1 \left [ X_n\right]
    =  \frac{1}{n} \|z^\star - \mu\|^2_{C_\sigma^{-1}}
\end{align*}

Taking the limit as $n \rightarrow \infty$ retrieves the results of Theorem~\ref{thm:asymp_lr}.
    
\end{remark}

\subsection{The Three Technical Lemmas Used in the Proof of Lemma~\ref{thm:asymp_lr}}

\begin{lemma}{Asymptotic expansion of the LR score}\label{lem:asymp_exp}

    
    We show that
    \begin{align*}
        \log\left ( \frac{d_{n - 1,j} \left( \delta^\mathrm{in}_{n,j} \right)}{d_{n,j} \left( \delta^\mathrm{out}_{n,j} \right)} \right ) &= \frac{1}{2} \left( \left(\delta^\mathrm{out}_{n,j} \right)^2 -   \left(\delta^\mathrm{in}_{n,j} \right)^2\right) + \frac{\lambda_3 \left (\mu_j - z^\star_j \right)}{n \sigma_j} R_{n,j} + o_{p}\left ( \frac{1}{n} \right) \\
    \end{align*}
    
    where $\delta^\mathrm{out}_{n,j} \defn \frac{\sqrt{n}}{\sigma_j} \left (\hat{\mu}_{n,j} - \mu_j  \right )$, $\delta^\mathrm{in}_{n,j} \defn \frac{n}{\sqrt{n - 1}\sigma_j} \left (\hat{\mu}_{n,j} - \mu_j + \frac{1}{n} \left (\mu_j - z^\star_j \right)  \right )$,
    
    and $R_{n, j} \defn \left(\delta^\mathrm{out}_{n,j} \right)^2 + \delta^\mathrm{out}_{n,j} \delta^\mathrm{in}_{n,j} + \left(\delta^\mathrm{in}_{n,j} \right)^2  - 3$.
\end{lemma}

\begin{proof}
    The proof starts by using the Edgeworth expansion of $d_{n,j}$ up to the order $k = 4$. Then, the final LR expansion can be found using Taylor expansions of the logarithm, exponential and polynomial function to the 2nd order. Here, we present the exact derivations for completeness.

    \underline{\textbf{Step1: Asymptotic expansion of $d_{n,j}$}}

    Using Theorem~\ref{thm:edgeworth}, for $k = 4$, we get that

    \begin{equation}\label{eq:expan_dnj}
        d_{n,j}(x) = \frac{1}{\sqrt{2 \pi}} e^{- \frac{x^2}{2}}\left( 1 + \frac{h_1(x)}{\sqrt{n}} + \frac{h_2(x)}{n}\right) + o\left ( \frac{1}{n} \right)
    \end{equation}

    uniformly in $x$, where

    \begin{align*}
        h_1(x) &\defn \lambda_3 \left(x^3 - 3x \right)\\
        h_2(x) &\defn \frac{\lambda_3^2}{72} \left(x^6 - 15 x^4 + 45 x^2 - 15 \right) + \frac{\lambda_4}{24} \left (x^4 - 6x^2 + 3\right)
    \end{align*}

    and $\lambda_k \defn \frac{\gamma_{j,k}}{\sigma_j^k}$ where $\gamma_{j,k}$ is the $k$-order cumulant of distribution $\mathcal{D}_j$.
    
    We do an expansion up to the $4$-th order, since we need an expansion up to $o_p\left(\frac{1}{n} \right)$.

    \underline{\textbf{Step2: Asymptotic expansion of $d_{n,j} \left( \delta^\mathrm{out}_{n,j} \right)$} and $d_{n -1,j} \left( \delta^\mathrm{in}_{n,j} \right)$}

    Since the convergence in uniform in Equation~\ref{eq:expan_dnj}, we get that
    \begin{align}\label{eq:dnj_out_1}
        d_{n,j} \left( \delta^\mathrm{out}_{n,j} \right) &= \frac{1}{\sqrt{2 \pi}} e^{- \frac{(\delta^\mathrm{out}_{n,j})^2}{2}}\left( 1 + \frac{h_1(\delta^\mathrm{out}_{n,j})}{\sqrt{n}} + \frac{h_2(\delta^\mathrm{out}_{n,j})}{n} + \right) + o_{p}\left ( \frac{1}{n} \right) \notag\\
        &= \frac{1}{\sqrt{2 \pi}} e^{- \frac{(\delta^\mathrm{out}_{n,j})^2}{2}}\left( 1 + \frac{h_1(\delta^\mathrm{out}_{n,j})}{\sqrt{n}} + \frac{h_2(\delta^\mathrm{out}_{n,j})}{n} + e^{\frac{(\delta^\mathrm{out}_{n,j})^2}{2}} o_{p}\left ( \frac{1}{n} \right) \right)
    \end{align}
and
\begin{align}\label{eq:dnj_in_1}
    d_{n - 1,j} \left( \delta^\mathrm{in}_{n,j} \right) &= \frac{1}{\sqrt{2 \pi}} e^{- \frac{(\delta^\mathrm{in}_{n,j})^2}{2}}\left( 1 + \frac{h_1(\delta^\mathrm{in}_{n,j})}{\sqrt{n - 1}} + \frac{h_2(\delta^\mathrm{in}_{n,j})}{n - 1}\right) + o_{p}\left ( \frac{1}{n} \right) \notag\\
    &= \frac{1}{\sqrt{2 \pi}} e^{- \frac{(\delta^\mathrm{in}_{n,j})^2}{2}}\left( 1 + \frac{h_1(\delta^\mathrm{in}_{n,j})}{\sqrt{n - 1}} + \frac{h_2(\delta^\mathrm{in}_{n,j})}{n - 1} + e^{\frac{(\delta^\mathrm{in}_{n,j})^2}{2}} o_{p}\left ( \frac{1}{n} \right) \right)
\end{align}

Under both $H_0$ and $H_1$, $\delta^\mathrm{in}_{n,j} \rightsquigarrow_{n \rightarrow \infty} Z \defn \mathcal{N}(0, 1)$ and $\delta^\mathrm{out}_{n,j} \rightsquigarrow_{n \rightarrow \infty} Z' \defn \mathcal{N}(0, 1)$.

Since $x \rightarrow e^{x^2/2}$ is a continuous function, then $e^{\frac{(\delta^\mathrm{in}_{n,j})^2}{2}} \rightsquigarrow_{n \rightarrow \infty} e^{Z^2/2} $ and $e^{\frac{(\delta^\mathrm{out}_{n,j})^2}{2}} \rightsquigarrow_{n \rightarrow \infty} e^{(Z')^2/2}$.

Using Prohorov's theorem (Thm~\ref{thm:proh}), we get that $e^{\frac{(\delta^\mathrm{in}_{n,j})^2}{2}} = O_p(1)$ and $e^{\frac{(\delta^\mathrm{out}_{n,j})^2}{2}} = O_p(1)$.

This means that 
\begin{equation}\label{eq:exp_in_big_O}
    e^{\frac{(\delta^\mathrm{in}_{n,j})^2}{2}} o_{p}\left ( \frac{1}{n} \right) = O_p(1) o_{p}\left ( \frac{1}{n} \right) = o_{p}\left ( \frac{1}{n} \right)
\end{equation}
 and 
 \begin{equation}\label{eq:exp_out_big_O}
     e^{\frac{(\delta^\mathrm{out}_{n,j})^2}{2}} o_{p}\left ( \frac{1}{n} \right) = O_{p}\left ( \frac{1}{n} \right) o_{p}\left ( \frac{1}{n} \right) = o_{p}\left ( \frac{1}{n} \right).
 \end{equation}

On the other hand, both $h_1$ and $h_2$ are continuous functions (polynomials functions), this similarly gives that $h_1(\delta^\mathrm{in}_{n,j}) = O_p(1)$ and $h_2(\delta^\mathrm{in}_{n,j}) = O_p(1)$.

Combined with the fact that $\frac{1}{\sqrt{n- 1}}= \frac{1}{\sqrt{n}} + o\left (\frac{1}{n} \right ) $ and $\frac{1}{n- 1} = \frac{1}{n} + o\left (\frac{1}{n} \right )$, we get that

\begin{equation}\label{eq:h_1}
    \frac{h_1(\delta^\mathrm{in}_{n,j})}{\sqrt{n- 1}}= \frac{h_1(\delta^\mathrm{in}_{n,j})}{\sqrt{n}} + O_p(1)o\left (\frac{1}{n} \right )  = \frac{h_1(\delta^\mathrm{in}_{n,j})}{\sqrt{n}} + o_p\left (\frac{1}{n} \right )  
\end{equation}
and
\begin{equation}\label{eq:h_2}
    \frac{h_2(\delta^\mathrm{in}_{n,j})}{n - 1}= \frac{h_1(\delta^\mathrm{in}_{n,j})}{n} + O_p(1)o\left (\frac{1}{n} \right )  = \frac{h_1(\delta^\mathrm{in}_{n,j})}{n} + o_p\left (\frac{1}{n} \right )  
\end{equation}

Plugging Eq.~\ref{eq:exp_out_big_O} in Eq.~\ref{eq:dnj_out_1} gives

\begin{equation}\label{eq:dnj_out_2}
    d_{n,j} \left( \delta^\mathrm{out}_{n,j} \right) = \frac{1}{\sqrt{2 \pi}} e^{- \frac{(\delta^\mathrm{out}_{n,j})^2}{2}}\left( 1 + \frac{h_1(\delta^\mathrm{out}_{n,j})}{\sqrt{n}} + \frac{h_2(\delta^\mathrm{out}_{n,j})}{n} + o_{p}\left ( \frac{1}{n} \right) \right)
\end{equation}

and Plugging Eq.~\ref{eq:exp_in_big_O}, Eq.\ref{eq:h_1} and Eq.~\ref{eq:h_2} in Eq.~\ref{eq:dnj_in_1} gives

 \begin{equation}\label{eq:dnj_in_2}
    d_{n-1,j} \left( \delta^\mathrm{in}_{n,j} \right) = \frac{1}{\sqrt{2 \pi}} e^{- \frac{(\delta^\mathrm{in}_{n,j})^2}{2}}\left( 1 + \frac{h_1(\delta^\mathrm{in}_{n,j})}{\sqrt{n}} + \frac{h_2(\delta^\mathrm{in}_{n,j})}{n} + o_{p}\left ( \frac{1}{n} \right) \right)
\end{equation}

\underline{\textbf{Step3: Asymptotic expansion of the logarithm}}

Applying the logarithm to Eq.\ref{eq:dnj_out_2}, we get that
\begin{align}\label{eq:log_dnj_out_1}
    \log\left (d_{n,j} \left( \delta^\mathrm{out}_{n,j} \right) \right) &= - \frac{1}{2} \log(2\pi) - \frac{(\delta^\mathrm{out}_{n,j})^2}{2} +  \log\left( 1 + \frac{h_1(\delta^\mathrm{out}_{n,j})}{\sqrt{n}} + \frac{h_2(\delta^\mathrm{out}_{n,j})}{n} + o_{p}\left ( \frac{1}{n} \right) \right) \notag\\
    &= - \frac{1}{2} \log(2\pi) - \frac{(\delta^\mathrm{out}_{n,j})^2}{2} +  \log\left( 1 + H_n,j \right) 
\end{align}
where $H_{n,j} \defn \frac{h_1(\delta^\mathrm{out}_{n,j})}{\sqrt{n}} + \frac{h_2(\delta^\mathrm{out}_{n,j})}{n} + o_{p}\left ( \frac{1}{n} \right)$.

Let $R(h) = \log(1 + h) - (h - \frac{h^2}{2})$, for $h\in \real^+$.

We have that $R(0) = 0$ and $R(h) = o \left( h^2\right)$ as $h \rightarrow 0$.

On the other hand, $H_{n,j} \rightarrow_{n \rightarrow \infty}^P 0$. Specifically $H_{n, j} = O_p\left (\frac{1}{\sqrt{n}} \right )$.

Using Lemma~\ref{lem:tayl_exp}, we get that
$R(H_n,j) = o_p\left((H_n,j)^2\right) =   o_p\left(\frac{1}{n}\right)$.

All in all, we have shown
\begin{align*}
    \log(1 + H_{n,j}) &= H_{n,j} - \frac{H_{n,j}^2}{2} + o_p\left ( \frac{1}{n} \right)\\
    &= \frac{h_1(\delta^\mathrm{out}_{n,j})}{\sqrt{n}} + \frac{1}{n } \left (\frac{h_2(\delta^\mathrm{out}_{n,j})}{n} - \frac{1}{2} \left (h_1(\delta^\mathrm{out}_{n,j})  \right) ^2\right) + o_{p}\left ( \frac{1}{n} \right)
\end{align*}

Plugging this in Eq.~\ref{eq:log_dnj_out_1}, we get that
\begin{equation*}
    \log\left (d_{n,j} \left( \delta^\mathrm{out}_{n,j} \right) \right) = - \frac{1}{2} \log(2\pi) - \frac{(\delta^\mathrm{out}_{n,j})^2}{2} + \frac{h_1(\delta^\mathrm{out}_{n,j})}{\sqrt{n}} + \frac{1}{n } \left (\frac{h_2(\delta^\mathrm{out}_{n,j})}{n} - \frac{1}{2} \left (h_1(\delta^\mathrm{out}_{n,j})  \right) ^2\right) + o_{p}\left ( \frac{1}{n} \right)
\end{equation*}

Similarly, we can show that
\begin{equation*}
    \log\left (d_{n - 1,j} \left( \delta^\mathrm{in}_{n,j} \right) \right) = - \frac{1}{2} \log(2\pi) - \frac{(\delta^\mathrm{in}_{n,j})^2}{2} + \frac{h_1(\delta^\mathrm{in}_{n,j})}{\sqrt{n}} + \frac{1}{n } \left (\frac{h_2(\delta^\mathrm{in}_{n,j})}{n} - \frac{1}{2} \left (h_1(\delta^\mathrm{in}_{n,j})  \right) ^2\right) + o_{p}\left ( \frac{1}{n} \right)
\end{equation*}

\underline{\textbf{Step4: Asymptotic expansion of polynomials in $\delta^\mathrm{in}_{n,j}$ and $\delta^\mathrm{out}_{n,j}$}}

First, we have

\begin{align*}
    \delta^\mathrm{in}_{n,j} &= \sqrt{\frac{n}{n - 1}} \delta^\mathrm{out}_{n,j} + \frac{\mu_j - z^\star_j}{\sigma_j \sqrt{n - 1}}\\
    &= \delta^\mathrm{out}_{n,j} + \frac{\mu_j - z^\star_j}{\sigma_j \sqrt{n}}  + o_{p}\left ( \frac{1}{\sqrt{n}} \right)
\end{align*}

Which gives 

\begin{align*}
    h_1\left(\delta^\mathrm{out}_{n,j} \right) - h_1\left(\delta^\mathrm{in}_{n,j} \right) &= \lambda_3 \left( \left(\delta^\mathrm{out}_{n,j} \right)^3 - \left(\delta^\mathrm{in}_{n,j} \right)^3 + \left(\delta^\mathrm{out}_{n,j} - \delta^\mathrm{in}_{n,j} \right) \right)\\
    &= \lambda_3 \left(\delta^\mathrm{out}_{n,j} - \delta^\mathrm{in}_{n,j} \right) \left ( \left(\delta^\mathrm{out}_{n,j} \right)^2 + \delta^\mathrm{out}_{n,j} \delta^\mathrm{in}_{n,j} + \left(\delta^\mathrm{in}_{n,j} \right)^2  - 3 \right) \\
    &= \lambda_3 R_{n, j} \left( \frac{ z^\star_j - \mu_j}{\sigma_j \sqrt{n}}  + o_{p}\left ( \frac{1}{\sqrt{n}} \right) \right)
\end{align*}
where $R_{n, j} \defn \left(\delta^\mathrm{out}_{n,j} \right)^2 + \delta^\mathrm{out}_{n,j} \delta^\mathrm{in}_{n,j} + \left(\delta^\mathrm{in}_{n,j} \right)^2  - 3$, which can be written as a polynomial in $\delta^\mathrm{out}_{n,j}$. 

Thus $R_{n, j} = O_p(1)$, which means that
\begin{align*}
     h_1\left(\delta^\mathrm{out}_{n,j} \right) - h_1\left(\delta^\mathrm{in}_{n,j} \right) &= \lambda_3 R_{n, j} \left( \frac{z^\star_j - \mu_j}{\sigma_j \sqrt{n}}\right)  + o_{p}\left ( \frac{1}{\sqrt{n}} \right) \\
     &= O_{p}\left ( \frac{1}{\sqrt{n}} \right)
\end{align*}

Similarly, one can show that
\begin{align*}
     h_2\left(\delta^\mathrm{out}_{n,j} \right) - h_2\left(\delta^\mathrm{in}_{n,j} \right) =  O_{p}\left ( \frac{1}{\sqrt{n}} \right)
\end{align*}

and
\begin{align*}
     h_1^2\left(\delta^\mathrm{out}_{n,j} \right) - h_1^2\left(\delta^\mathrm{in}_{n,j} \right) = O_{p}\left ( \frac{1}{\sqrt{n}} \right).
\end{align*}

\underline{\textbf{Step5: Summarising the asymptotic expansion of the LR score}}

All in all, we get that
\begin{align*}
    \log\left ( \frac{d_{n - 1,j} \left( \delta^\mathrm{in}_{n,j} \right)}{d_{n,j} \left( \delta^\mathrm{out}_{n,j} \right)} \right ) &= \frac{1}{2} \left( \left(\delta^\mathrm{out}_{n,j} \right)^2 -   \left(\delta^\mathrm{in}_{n,j} \right)^2\right) + \frac{\lambda_3 \left (\mu_j - z^\star_j \right)}{n \sigma_j} R_{n,j} + o_{p}\left ( \frac{1}{n} \right) \\
\end{align*}

\end{proof}

\begin{lemma}{Expectation and variance computations}\label{lem:exp_comp}
\begin{align*}
    &\expect_0 [ Y_{n,j} ] = - \frac{1}{2n} - \frac{(z_j^\star - \mu_j)^2}{2 n \sigma_j^2} + o\left (\frac{1}{n} \right)  &&V_0 [ Y_{n,j} ] = \frac{(z_j^\star - \mu_j)^2}{n \sigma_j^2} + o\left (\frac{1}{n} \right)\\
    &\expect_1 [ Y_{n,j} ] = - \frac{1}{2n} + \frac{(z_j^\star - \mu_j)^2}{2 n \sigma_j^2} + o\left (\frac{1}{n} \right) && V_1 [ Y_{n,j} ] = \frac{(z_j^\star - \mu_j)^2}{n \sigma_j^2} + o\left (\frac{1}{n} \right)
\end{align*}
\end{lemma}

\begin{proof}
    The proof is direct from expectation and variance of the mean under $H_0$ and $H_1$. Specifically, under $H_O$ we have that $\expect_0(\hat{\mu}_{n,j}) = \mu_j$ and $\Var_0(\hat{\mu}_{n,j}) = \frac{1}{n} \sigma_j^2$. On the other hand, under $H_1$, we have that $\expect_1(\hat{\mu}_{n,j}) = \mu_j + \frac{1}{n} \left( z^\star_j - \mu_j\right)$ and $\Var_1(\hat{\mu}_{n,j}) = \frac{n - 1}{n^2}\sigma_j^2$. Here, we present the exact derivations for completeness.
    
    Let us recall that
    \begin{equation*}
        Y_{n, j} \defn \frac{1}{2} \left( \left(\delta^\mathrm{out}_{n,j} \right)^2 -   \left(\delta^\mathrm{in}_{n,j} \right)^2\right) + \frac{\lambda_3 \left (\mu_j - z^\star_j \right)}{n \sigma_j} R_{n,j}
    \end{equation*}
    
    where $R_{n,j} \defn \left(\delta^\mathrm{out}_{n,j} \right)^2 + \delta^\mathrm{out}_{n,j} \delta^\mathrm{in}_{n,j} + \left(\delta^\mathrm{in}_{n,j} \right)^2  - 3 $
    
    \underline{Under $H_0$}:
    
    Since $\expect_0(\hat{\mu}_{n,j}) = \mu_j$ and $\Var_0(\hat{\mu}_{n,j}) = \frac{1}{n} \sigma_j^2$, we get that
    \begin{align*}
        &\expect_0\left(\delta^\mathrm{out}_{n,j} \right) = 0 && \expect_0\left(\left(\delta^\mathrm{out}_{n,j} \right)^2 \right) = 1\\
        &\expect_0\left(\delta^\mathrm{in}_{n,j} \right) = \frac{\mu_j - z^\star_j}{\sigma_j \sqrt{n - 1}} && \expect_0\left(\left(\delta^\mathrm{in}_{n,j} \right)^2 \right) = 1 + \frac{1}{n} + \frac{\left(\mu_j - z^\star_j \right)^2}{n \sigma_j^2} + o\left(\frac{1}{n} \right)
    \end{align*}
    
    Also
    \begin{equation*}
        \expect_0\left( \delta^\mathrm{out}_{n,j} \delta^\mathrm{in}_{n,j} \right) = \sqrt{\frac{n}{n - 1}} = 1 + \frac{1}{2n} + o\left (\frac{1}{n}\right)
    \end{equation*}
    
    Which means that
    \begin{equation*}
        \expect_0 \left( R_{n,j} \right) = \frac{3}{2n} + \frac{\left(\mu_j - z^\star_j \right)^2}{n \sigma_j^2} + o\left (\frac{1}{n}\right) = o(1)
    \end{equation*}
    
    Finally
    \begin{equation*}
        \expect_0 [ Y_{n,j} ] = - \frac{1}{2n} - \frac{(z_j^\star - \mu_j)^2}{2 n \sigma_j^2} + o\left (\frac{1}{n} \right)
    \end{equation*}
    
    On the other hand, 
    \begin{align*}
        \Var_0 [ Y_{n,j} ] &= \expect_0 [ Y_{n,j}^2 ] - \left (\expect_0 [ Y_{n,j}] \right)^2\\
        &= \expect_0 [ Y_{n,j}^2 ] + o\left (\frac{1}{n} \right)\\
        &= \frac{1}{4} \expect_0 \left [ \left (  \left(\delta^\mathrm{out}_{n,j} \right)^2 - \left(\delta^\mathrm{in}_{n,j} \right)^2  \right)^2 \right ] + \frac{\lambda_3^2}{n^2 \sigma_j^2} \left ( \mu_j - z^\star_j\right)^2 \expect_0[R_{n,j}^2] \\
        &+ \frac{\lambda_3}{n \sigma_j} \left (\mu_j - z^\star_j  \right) \expect_0\left [ R_{n,j} \left ( \left(\delta^\mathrm{out}_{n,j} \right)^2 - \left(\delta^\mathrm{in}_{n,j} \right)^2 \right)\right] 
    \end{align*}
    
    We compute the following expectations:
    \begin{align*}
        \expect_0[\left(\delta^\mathrm{out}_{n,j} \right)^3] &= \frac{n^{3/2}}{\sigma_j^3} \expect_0\left [\left ( \hat{\mu}_{n, j} - \mu_j \right)^3 \right]\\
        &= \frac{1}{\sqrt{n}} \frac{\mu_{3, j}}{\sigma_j^3}\\
        \expect_0[\left(\delta^\mathrm{out}_{n,j} \right)^4] &= \frac{n^2}{\sigma_j^4} \expect_0\left [\left ( \hat{\mu}_{n, j} - \mu_j \right)^4 \right]\\
        &= 3 + \frac{1}{n} \left ( \frac{\mu_{4,j}}{\sigma_j^4} - 3 \right)\\
        \expect_0[\left(\delta^\mathrm{in}_{n,j} \right)^4] &= \expect_0\left [\left ( \sqrt{\frac{n}{n - 1}} \delta^\mathrm{out}_{n,j} + \frac{\mu_j - z^\star_j}{\sigma_j \sqrt{n - 1}} \right)^4 \right]\\
        &= 3 + \frac{1}{n} \left ( \frac{\mu_{4,j}}{\sigma_j^4} - 3 \right)+ \frac{6}{n} + \frac{4}{n} \frac{\mu_j - z^\star_j}{\sigma_j^4}\mu_{3,j}  + \frac{6}{n}  \frac{(\mu_j - z^\star_j)^2}{\sigma_j^2} + o\left(\frac{1}{n} \right) \\
        \expect_0[\left(\delta^\mathrm{out}_{n,j} \right)^2 \left(\delta^\mathrm{in}_{n,j} \right)^2 ] &= \expect_0\left[\left(\delta^\mathrm{out}_{n,j} \right)^2 \left(\sqrt{\frac{n}{n - 1}} \delta^\mathrm{out}_{n,j} + \frac{\mu_j - z^\star_j}{\sigma_j \sqrt{n - 1}}  \right)^2 \right] \\
        &= 3 + \frac{1}{n} \left ( \frac{\mu_{4,j}}{\sigma_j^4} - 3 \right) + \frac{3}{n} + \frac{2}{n} \frac{\mu_j - z^\star_j}{\sigma_j^4}\mu_{3,j}  + \frac{1}{n}  \frac{(\mu_j - z^\star_j)^2}{\sigma_j^2} + o\left(\frac{1}{n} \right) 
    \end{align*}
    
    All in all, we have
    \begin{align*}
        \expect_0 \left [ \left (  \left(\delta^\mathrm{out}_{n,j} \right)^2 - \left(\delta^\mathrm{in}_{n,j} \right)^2  \right)^2 \right ] &= \expect_0[\left(\delta^\mathrm{out}_{n,j} \right)^4] + \expect_0[\left(\delta^\mathrm{in}_{n,j} \right)^4] - 2 \expect_0[\left(\delta^\mathrm{out}_{n,j} \right)^2 \left(\delta^\mathrm{in}_{n,j} \right)^2 ]\\
        &= \frac{4}{n}  \frac{(\mu_j - z^\star_j)^2}{\sigma_j^2} + o\left(\frac{1}{n} \right) \\
        \expect_0[R_{n,j}^2] &= O(1)\\
        \expect_0\left [ R_{n,j} \left ( \left(\delta^\mathrm{out}_{n,j} \right)^2 - \left(\delta^\mathrm{in}_{n,j} \right)^2 \right)\right]  &= o(1)
    \end{align*}
    
    Thus 
    \begin{equation*}
        \Var_0 [ Y_{n,j} ] = \frac{1}{n}  \frac{(\mu_j - z^\star_j)^2}{\sigma_j^2} + o\left(\frac{1}{n} \right)
    \end{equation*}

    \underline{Under $H_1$}:
    
    Since $\expect_1(\hat{\mu}_{n,j}) = \mu_j + \frac{1}{n} \left( z^\star_j - \mu_j\right)$ and $\Var_1(\hat{\mu}_{n,j}) = \frac{n - 1}{n^2}\sigma_j^2$, we get that
    \begin{align*}
        &\expect_1\left(\delta^\mathrm{out}_{n,j} \right) = \frac{ z^\star_j - \mu_j}{\sigma_j \sqrt{n}}  && \expect_1\left(\left(\delta^\mathrm{out}_{n,j} \right)^2 \right) = 1 - \frac{1}{n} + \frac{\left(\mu_j - z^\star_j \right)^2}{n \sigma_j^2} \\
        &\expect_1\left(\delta^\mathrm{in}_{n,j} \right) = 0 && \expect_1\left(\left(\delta^\mathrm{in}_{n,j} \right)^2 \right) = 1
    \end{align*}
    
    Also
    \begin{equation*}
        \expect_1\left( \delta^\mathrm{out}_{n,j} \delta^\mathrm{in}_{n,j} \right) = \sqrt{\frac{n - 1}{n}} = 1 - \frac{1}{2n} + o\left (\frac{1}{n}\right)
    \end{equation*}
    
    Which means that
    \begin{equation*}
        \expect_1 \left( R_{n,j} \right) = - \frac{3}{2n} + \frac{\left(\mu_j - z^\star_j \right)^2}{n \sigma_j^2} + o\left (\frac{1}{n}\right) = o(1)
    \end{equation*}
    
    Finally
    \begin{equation*}
        \expect_1 [ Y_{n,j} ] = - \frac{1}{2n} + \frac{(z_j^\star - \mu_j)^2}{2 n \sigma_j^2} + o\left (\frac{1}{n} \right)
    \end{equation*}
    
    
    Under $H_1$, one can rewrite $\delta^\mathrm{in}_{n, j} = \hat{Z}_{n - 1, j}$ and $ \delta^\mathrm{out}_{n, j} = \sqrt{\frac{n - 1}{n}} \delta^\mathrm{in}_{n, j} + \frac{z^\star_j - \mu_j}{\sigma_j \sqrt{n}}$.
    
    Thus using the same steps as in $H_0$, we get
    \begin{align*}
        \expect_1[\left(\delta^\mathrm{in}_{n,j} \right)^3] &= \frac{1}{\sqrt{n - 1}} \frac{\mu_{3, j}}{\sigma_j^3}\\
        \expect_1[\left(\delta^\mathrm{in}_{n,j} \right)^4] 
        &= 3 + \frac{1}{n - 1} \left ( \frac{\mu_{4,j}}{\sigma_j^4} - 3 \right)\\
        \expect_1[\left(\delta^\mathrm{out}_{n,j} \right)^4] 
        &= 3 + \frac{1}{n} \left ( \frac{\mu_{4,j}}{\sigma_j^4} - 3 \right)+ \frac{6}{n} + \frac{4}{n} \frac{z^\star_j - \mu_j}{\sigma_j^4}\mu_{3,j}  + \frac{6}{n}  \frac{(z^\star_j - \mu_j)^2}{\sigma_j^2} + o\left(\frac{1}{n} \right) \\
        \expect_1[\left(\delta^\mathrm{out}_{n,j} \right)^2 \left(\delta^\mathrm{in}_{n,j} \right)^2 ] 
        &= 3 + \frac{1}{n} \left ( \frac{\mu_{4,j}}{\sigma_j^4} - 3 \right) + \frac{3}{n} + \frac{2}{n -1} \frac{z^\star_j - \mu_j}{\sigma_j^4}\mu_{3,j}  + \frac{1}{n}  \frac{(z^\star_j - \mu_j)^2}{\sigma_j^2} + o\left(\frac{1}{n} \right) 
    \end{align*}
    
    All in all, we have
    \begin{align*}
        \expect_1 \left [ \left (  \left(\delta^\mathrm{out}_{n,j} \right)^2 - \left(\delta^\mathrm{in}_{n,j} \right)^2  \right)^2 \right ]
        &= \frac{4}{n}  \frac{(\mu_j - z^\star_j)^2}{\sigma_j^2} + o\left(\frac{1}{n} \right) \\
        \expect_1[R_{n,j}^2] &= O(1)\\
        \expect_1\left [ R_{n,j} \left ( \left(\delta^\mathrm{out}_{n,j} \right)^2 - \left(\delta^\mathrm{in}_{n,j} \right)^2 \right)\right]  &= o(1)
    \end{align*}
    
    Thus 
    \begin{equation*}
        \Var_1 [ Y_{n,j} ] = \frac{1}{n}  \frac{(\mu_j - z^\star_j)^2}{\sigma_j^2} + o\left(\frac{1}{n} \right)
    \end{equation*}
    
\end{proof}

\begin{lemma}{The Lindeberg-Feller condition}\label{lem:lind_cond}
    
    The random variables $(Y_{n,j})_{j = 1}^{d_n}$ verify the Lindeberg-Feller condition.
\end{lemma}

\begin{proof}
    Let $\epsilon > 0$.
    Let $h \in \{0, 1\}$. 
    
    Let $\delta > 0$.
    We have that
    \begin{align*}
        \expect_h \left [ Y_{n, j}^2 \ind{|Y_{n,j} | > \epsilon }  \right] &= \expect_h \left [ \frac{Y_{n, j}^{2 + \delta}}{Y_{n, j}^ \delta} \ind{|Y_{n,j} | > \epsilon }  \right] \\
        &\leq \frac{1}{\epsilon^\delta} \expect_h \left [ Y_{n, j}^{2 + \delta} \right]
    \end{align*}
    
    On the other hand, we have that $Y_{n, j} = \frac{1}{2} \left( \left(\delta^\mathrm{out}_{n,j} \right)^2 -   \left(\delta^\mathrm{in}_{n,j} \right)^2\right) + \frac{\lambda_3 \left (\mu_j - z^\star_j \right)}{n \sigma_j} R_{n,j}$, where 
    \begin{equation*}
        \left( \left(\delta^\mathrm{out}_{n,j} \right)^2 -   \left(\delta^\mathrm{in}_{n,j} \right)^2\right) = O_p \left (\frac{1}{\sqrt{n}}\right) \text{ and } R_{n,j} = O_p(1)
    \end{equation*}
    Thus $Y_{n, j} = O_p \left (\frac{1}{\sqrt{n}}\right)$, and $\expect_h \left [ Y_{n, j}^{2 + \delta} \right] = o\left (\frac{1}{n} \right)$.
    
    Which means that $ \expect_h \left [ Y_{n, j}^2 \ind{|Y_{n,j} | > \epsilon }  \right] = o\left (\frac{1}{n} \right) $ and 
    
    \begin{equation*}
        \sum_{j = 1}^{d_n} \expect_h \left [ Y_{n, j}^2 \ind{|Y_{n,j} | > \epsilon }  \right]  =  o\left (\frac{d_n}{n} \right) = o(1) \rightarrow 0
    \end{equation*}
\end{proof}

\subsection{Proof of Theorem~\ref{crl:per_leak}}

\begin{reptheorem}{crl:per_leak}[Target-dependent leakage of the empirical mean]
The asymptotic target-dependent leakage of $z^\star$ in the empirical mean is
\begin{equation*}
    \lim_{n,d} \xi_n(z^\star, \mathcal{M}^\mathrm{emp}_n, \mathcal{D}) = \Phi\left (\frac{\sqrt{m^\star}}{2}\right) - \Phi\left(-\frac{\sqrt{m^\star}}{2}\right)
\end{equation*}
The asymptotic trade-off function, achievable with threshold $\tau_\alpha = -\frac{m^\star}{2} + \sqrt{m^\star} \Phi^{-1}(1 - \alpha)$, is  
\begin{equation*}
    \lim_{n,d} \mathrm{Pow}_n(\ell_n, \alpha, z^\star) = \Phi\left( \Phi^{-1}(\alpha) + \sqrt{m^\star}\right)
\end{equation*}
\end{reptheorem}

\begin{proof}
    From the asymptotic distribution of the LR score, we get directly that
    \begin{align*}
        \lim_{n,d} \xi_n(z^\star, \mathcal{M}^\mathrm{emp}_n, \mathcal{D}) &= \mathrm{Pr} \left(\mathcal{N}\left(-\frac{ m^\star}{2}, m^\star \right) < 0 \right) - \mathrm{Pr} \left(\mathcal{N}\left(\frac{ m^\star}{2}, m^\star \right) < 0 \right) \notag\\
        &= \Phi\left (\frac{m^\star/2}{\sqrt{m^\star}}\right) - \Phi\left (-\frac{m^\star/2}{\sqrt{m^\star}}\right) \notag\\
        &= \Phi\left (\frac{\sqrt{m^\star}}{2}\right) - \Phi\left(-\frac{\sqrt{m^\star}}{2}\right)
    \end{align*}
    
    The threshold $\tau_\alpha$ for which the asymptotic LR attack achieves significance $\alpha$ verifies:
    \begin{align*}
        \mathrm{Pr} \left(\mathcal{N}\left(- \frac{ m^\star}{2}, m^\star \right) \geq \tau_\alpha ) \right) = \alpha
    \end{align*}
    
    Thus $\tau_\alpha = -\frac{m^\star}{2} + \sqrt{m^\star} \Phi^{- 1}(1 - \alpha)$.
    
    Finally, we find the power of the test by
    \begin{align*}
        \lim_{n,d} \mathrm{Pow}_n(\ell_n, \alpha, z^\star) &= \mathrm{Pr}\left(\mathcal{N}\left(\frac{ m^\star}{2}, m^\star \right) \geq \tau_\alpha  \right)\\
        &= \mathrm{Pr}\left(\frac{ m^\star}{2} + \sqrt{m^\star}\mathcal{N}\left(0, 1 \right) \geq  -\frac{m^\star}{2} + \sqrt{m^\star} \Phi^{- 1}(1 - \alpha) \right)\\
        &= \mathrm{Pr}\left(\sqrt{m^\star}\mathcal{N}\left(0, 1 \right) \geq  -m^\star - \sqrt{m^\star} \Phi^{- 1}(\alpha) \right)\\
        &= \mathrm{Pr}\left(\mathcal{N}\left(0, 1 \right) \leq  \sqrt{m^\star} +  \Phi^{- 1}(\alpha) \right)\\
        &= \Phi\left( \Phi^{-1}(\alpha) + \sqrt{m^\star}\right)
    \end{align*}
\end{proof}

\subsection{Effects of Sub-sampling, Proof of Theorem~\ref{thm:lr_sub}}\label{sec:proof_sub}

\begin{reptheorem}{thm:lr_sub}[Target-dependent leakage for $\mathcal{M}^{\mathrm{sub}, \rho}_n$]

First, we show that as $d,n \rightarrow \infty$ while $d/n = \tau$:

- Under $H_0$, $ \ell^\mathrm{sub, \rho}_n(\hat{\mu}^\mathrm{sub}_{n}; z^\star, \mu, C_\sigma) \rightsquigarrow \mathcal{N}\left(- \rho \frac{ m^\star}{2}, \rho m^\star \right)$.
        
- Under $H_1$, $ \ell^\mathrm{sub, \rho}_n(\hat{\mu}^\mathrm{sub}_{n}; z^\star, \mu, C_\sigma) \rightsquigarrow \mathcal{N}\left( \rho \frac{ m^\star}{2}, \rho m^\star \right)$.

The asymptotic target-dependent leakage of $z^\star$ in $\mathcal{M}^{\mathrm{sub}, \rho}_n$ is
\begin{equation*}
    \lim_{n,d} \xi_n(z^\star, \mathcal{M}^{\mathrm{sub}}_{n, \rho}, \mathcal{D}) = \Phi\left (\frac{\sqrt{\rho  m^\star}}{2}\right) - \Phi\left(-\frac{\sqrt{\rho  m^\star}}{2}\right)
\end{equation*}
The optimal trade-off function obtained with $\tau_\alpha = -\frac{\rho  m^\star}{2} + \sqrt{\rho  m^\star} \Phi^{-1}(1 - \alpha)$, is  
\begin{equation*}
    \lim_{n,d} \mathrm{Pow}_n(\ell^\mathrm{sub}_{n, \rho}, \alpha, z^\star) = \Phi\left( z_\alpha + \sqrt{\rho  m^\star}\right).
\end{equation*}
\end{reptheorem}

\begin{proof}
    We have that $\hat{\mu}^\mathrm{sub}_{n} = \frac{1}{k_n} \sum_{i=1}^n Z_i \ind{\varsigma(i) \leq k_n}$, where $k_n \defn \rho n$, $Z_i$ are i.i.d and $\varsigma \sim^{\mathrm{unif}} S_n$ is a permutation sampled uniformly from the set of permutations of $\{1 \dots, n\}$ i.e. $S_n$ and independently from $(Z_i)$.

    We denote $\hat{\mu}^\mathrm{sub}_n = (\hat{\mu}^\mathrm{sub}_{n, j})_{j = 1}^{d_n}$.
    
    \underline{\textbf{Step 1: Rewriting the LR score}}

    Let $j \in [1, d_n]$.

    \underline{\textit{Under $H_0$,}} we can re-write 
    \begin{equation*}
        \hat{\mu}^\mathrm{sub}_{n, j} = \mu_j + \frac{\sigma_j}{\sqrt{k_n}} \hat{Z}_{k_n, j},
    \end{equation*}
     where 
     \begin{align*}
         \hat{Z}_{d_n, j} &\defn \sqrt{k_n} \left ( \frac{\hat{\mu}^\mathrm{sub}_{n, j}- \mu_j}{\sigma} \right)\\
         &= \frac{1}{\sqrt{k_n}} \sum_{i = 1}^{k_n} \frac{Z_{\varsigma^{-1}(i), j} - \mu_j}{\sigma_j}.
     \end{align*}
     Since $(Z_{i,j})_{i = 1}^n$ are i.i.d from $\mathcal{D}_j$, and $ \varsigma \sim^{\mathrm{unif}} S_n$ and ind. from $(Z_i)$, then $(Z_{\varsigma^{-1}(i), j})_{i = 1}^{k_n}$. 
     
     Using the CLT, $\hat{Z}_{d_n, j} \rightsquigarrow_{n \rightarrow \infty} \mathcal{N}(0,1)$.

     Let $d_{n, j}$ be the density function of $\hat{Z}_{n,j}$.

     The density $p^\mathrm{out, sub}_{n,j}$ of $\hat{\mu}^\mathrm{sub}_{n, j}$ under $H_0$ can be written as
    \begin{equation*}
        p^\mathrm{out, sub}_{n,j}(x; z^\star_j, \mu_j, \sigma_j) = \frac{\sqrt{k_n}}{\sigma_j} d_{k_n, j}\left [ \frac{\sqrt{k_n}}{\sigma_j} (x - \mu_j) \right]
    \end{equation*}

    \underline{\textit{Under $H_1$,}} we can re-write 
    \begin{align*}
        \hat{\mu}^\mathrm{sub}_{n, j} &= \frac{1}{k_n} z^\star_j \ind{\varsigma(n) \leq k_n} + \frac{1}{k_n} \sum_{i=1}^{n- 1} Z_i \ind{\varsigma(i) \leq k_n}\\
    \end{align*}

    
    Let $A = \{ \ind{\varsigma(n) \leq k_n} \}$ the event that $z^\star$ was sub-sampled. We have that $\mathrm{Pr}(A) = \rho$.
    
     The density $p^\mathrm{in, sub}_{n,j}$ of $\hat{\mu}^\mathrm{sub}_{n, j}$ under $H_1$ and given $A$ is
    \begin{equation*}
        \frac{k_n}{\sigma_j \sqrt{k_n - 1}} d_{k_n - 1, j}\left [ \frac{k_n}{\sigma_j \sqrt{k_n - 1}} \left (x - \mu_j - \frac{1}{k_n} \left (z^\star_{j} - \mu_j \right)  \right) \right]
    \end{equation*}
    
    The density $p^\mathrm{in, sub}_{n,j}$ of $\hat{\mu}^\mathrm{sub}_{n, j}$ under $H_1$ and given $A^c$ is
    \begin{equation*}
        \frac{\sqrt{k_n}}{\sigma_j} d_{k_n, j}\left [ \frac{\sqrt{k_n}}{\sigma_j} (x - \mu_j) \right]
    \end{equation*}
    
    Thus, the density $p^\mathrm{in, sub}_{n,j}$ of $\hat{\mu}^\mathrm{sub}_{n, j}$ under $H_1$ can be written as
    \begin{align*}
        p^\mathrm{in, sub}_{n,j}(x; z^\star_j, \mu_j, \sigma_j) &= (1 - \rho)\frac{\sqrt{k_n}}{\sigma_j} d_{k_n, j}\left [ \frac{\sqrt{k_n}}{\sigma_j} (x - \mu_j) \right] \\ &+ \rho\frac{k_n}{\sigma_j \sqrt{k_n - 1}} d_{k_n - 1, j}\left [ \frac{k_n}{\sigma_j \sqrt{k_n - 1}} \left (x - \mu_j - \frac{1}{k_n} \left (z^\star_{j} - \mu_j \right)  \right) \right]
    \end{align*}

    \textit{The additional technical hardness of this proof comes from the `mixture' nature of the `in' distribution.}

    The LR score is 
    \begin{align*}
        \ell^\mathrm{sub, \rho}_n(\hat{\mu}^\mathrm{sub}_{n}; z^\star, \mu, C_\sigma) &= \sum_{j = 1}^{d_n} \log\left ( \frac{p^\mathrm{in, sub}_{n,j}(\hat{\mu}^\mathrm{sub}_{n, j}; z^\star_j, \mu_j, \sigma_j) }{p^\mathrm{out, sub}_{n,j}(\hat{\mu}^\mathrm{sub}_{n, j}; z^\star_j, \mu_j, \sigma_j) } \right ) \notag\\
        &= \sum_{j = 1}^{d_n} \log \left ( \frac{ (1 - \rho)\frac{\sqrt{k_n}}{\sigma_j} d_{k_n, j}\left ( \delta^\mathrm{out, sub}_{k_n, j} \right) + \rho \frac{k_n}{\sigma_j \sqrt{k_n - 1}} d_{k_n - 1, j}\left (\delta^\mathrm{in, sub}_{k_n, j}\right) }{ \frac{\sqrt{k_n}}{\sigma_j} d_{k_n, j}\left ( \delta^\mathrm{out, sub}_{k_n, j} \right)} \right)\\
        &= \sum_{j = 1}^{d_n} \log \left ( (1 - \rho) + \rho \sqrt{\frac{k_n}{k_n -1}} \frac{d_{k_n - 1,j} \left( \delta^\mathrm{in, sub}_{k_n,j} \right)}{d_{k_n,j} \left( \delta^\mathrm{out, sub}_{k_n,j} \right)} \right)
    \end{align*}
    where 
    \begin{align*}
        \delta^\mathrm{out, sub}_{k_n, j} &\defn \frac{\sqrt{k_n}}{\sigma_j} \left (\hat{\mu}^\mathrm{sub}_{n,j} - \mu_j  \right ) \\
        \delta^\mathrm{in, sub}_{k_n, j} &\defn \frac{k_n}{\sqrt{k_n - 1}\sigma_j} \left (\hat{\mu}^\mathrm{sub}_{n,j} - \mu_j + \frac{1}{k_n} \left (\mu_j - z^\star_j \right)  \right )
    \end{align*}

\underline{\textbf{Step 2: Asymptotic expansion of the LR score}}

Using Lemma~\ref{lem:asymp_sub}, we have
\begin{align*}
    \log \left( (1- \rho) + \rho \sqrt{\frac{k_n}{k_n -1}} \frac{d_{k_n - 1,j} \left( \delta^\mathrm{in, sub}_{k_n,j} \right)}{d_{k_n,j} \left( \delta^\mathrm{out, sub}_{k_n,j} \right)} \right) =  W_{n, j} + o_p\left (\frac{1}{n} \right)
\end{align*}

where

\begin{equation*}
        W_{n, j} \defn \frac{\rho}{2}(\delta^\mathrm{out, sub}_{k_n,j})^2  -\frac{\rho}{2} (\delta^\mathrm{in, sub}_{k_n,j})^2   + \rho \frac{\lambda_3 \left (\mu_j - z^\star_j \right)}{k_n \sigma_j} R_{k_n,j} + \frac{\rho}{2 k_n}  + \frac{\rho( 1 - \rho)}{8} \left( (\delta^\mathrm{out, sub}_{k_n,j})^2  -(\delta^\mathrm{in, sub}_{k_n,j})^2 \right)^2
\end{equation*}

The extra hardness of this proof comes from the fact that the density under $H_1$ is now a mixture of two Gaussians, rather than just one Gaussian in the case of the exact empirical mean.

Thus 

\begin{align*}
    \ell^\mathrm{sub, \rho}_n(\hat{\mu}^\mathrm{sub}_{n}; z^\star, \mu, C_\sigma) &= o_p(1) + \sum_{j = 1}^{d_n} W_{n, j} 
\end{align*}

because $\frac{d_n}{n} = \tau + o(1)$.

\underline{\textbf{Step3: Concluding using the Lindeberg-Feller CLT
}}


\underline{\textit{Under $H_0$:}}

Using Lemma~\ref{lem:exp_sub}, $\expect_0 [ W_n,j ] = - \frac{\rho}{2} \frac{\left(\mu_j - z^\star_j \right)^2}{n \sigma_j^2} + o\left (\frac{1}{n} \right)$ and $\Var_0 [ W_{n,j} ] = \rho  \frac{(\mu_j - z^\star_j)^2}{n \sigma_j^2} + o\left(\frac{1}{n} \right)$.

Since $\sum_{j = 1}^{d_n} \frac{(z_j^\star - \mu_j)^2}{n \sigma_j^2}  = \frac{\| z^\star - \mu \|^2_{C_\sigma^{-1}}}{n}$, we get:
\begin{itemize}
    \item $\sum_{j = 1}^{d_n} \expect_0 [ W_n,j ] \rightarrow - \frac{\rho m^\star}{2}$

    \item $\sum_{j = 1}^{d_n} V_0 [ W_n,j ] \rightarrow \rho m^\star$
\end{itemize}

Similarly to Lemma~\ref{lem:lind_cond}, we can show that $W_{n,j}$ verify the Lindeberg-Feller condition, i.e.

\begin{equation*}
    \sum_{j = 1}^{d_n} \expect_0 \left [ W_{n, j}^2 \ind{|W_{n,j} | > \epsilon }  \right] \rightarrow 0
\end{equation*}

for every $\epsilon >0$.

We conclude using the Lindeberg-Feller CLT (Theorem~\ref{thm:lind_feller}) that

\begin{equation*}
    \ell^\mathrm{sub, \rho}_n(\hat{\mu}^\mathrm{sub}_{n}; z^\star, \mu, C_\sigma)
    \rightsquigarrow \mathcal{N}\left ( -  \rho \frac{m^\star}{2},  \rho m^\star\right) 
\end{equation*}

Similarly, \underline{\textit{Under $H_1$:}}

Using Lemma~\ref{lem:exp_sub}, $\expect_1 [ W_n,j ] = \frac{\rho}{2} \frac{\left(\mu_j - z^\star_j \right)^2}{n \sigma_j^2} + o\left (\frac{1}{n} \right)$ and $\Var_1 [ W_{n,j} ] = \rho \frac{(z_j^\star - \mu_j)^2}{n \sigma_j^2} + o\left (\frac{1}{n} \right)$.

We get:
\begin{itemize}
    \item $\sum_{j = 1}^{d_n} \expect_1 [ W_n,j ] \rightarrow \frac{ \rho m^\star}{2}$

    \item $\sum_{j = 1}^{d_n} V_1 [ W_n,j ] \rightarrow \rho m^\star$
\end{itemize}

Similarly to Lemma~\ref{lem:lind_cond}, we can show that $W_{n,j}$ verify the Lindeberg-Feller condition, i.e.

\begin{equation*}
    \sum_{j = 1}^{d_n} \expect_1 \left [ W_{n, j}^2 \ind{|W_{n,j} | > \epsilon }  \right] \rightarrow 0
\end{equation*}

for every $\epsilon >0$.

We conclude using the Lindeberg-Feller CLT (Theorem~\ref{thm:lind_feller}) that

\begin{equation*}
    \ell^\mathrm{sub, \rho}_n(\hat{\mu}^\mathrm{sub}_{n}; z^\star, \mu, C_\sigma)
    \rightsquigarrow \mathcal{N}\left (\rho \frac{m^\star}{2},  \rho m^\star\right) 
\end{equation*}

\underline{\textbf{Step4: Characterising the advantage and the power function
}}

Using the same step as in the proof of Corollary~\ref{crl:per_leak}, we conclude.

\end{proof}

Now we present the proof of the helpful technical lemmas.
\begin{lemma}[Asymptotic expansion of the LR score for sub-sampling]\label{lem:asymp_sub}
We show that
\begin{align*}
    \log \left( (1- \rho) + \rho \sqrt{\frac{k_n}{k_n -1}} \frac{d_{k_n - 1,j} \left( \delta^\mathrm{in, sub}_{k_n,j} \right)}{d_{k_n,j} \left( \delta^\mathrm{out, sub}_{k_n,j} \right)} \right) =  W_{n, j} + o_p\left (\frac{1}{n} \right)
\end{align*}

where

\begin{align*}
        W_{n, j} &\defn \frac{\rho}{2}(\delta^\mathrm{out, sub}_{k_n,j})^2  -\frac{\rho}{2} (\delta^\mathrm{in, sub}_{k_n,j})^2   + \rho \frac{\lambda_3 \left (\mu_j - z^\star_j \right)}{k_n \sigma_j} R_{k_n,j} + \frac{\rho}{2 k_n}  \\
        &+ \frac{\rho( 1 - \rho)}{8} \left( (\delta^\mathrm{out, sub}_{k_n,j})^2  -(\delta^\mathrm{in, sub}_{k_n,j})^2 \right)^2
\end{align*}
\end{lemma}

\begin{proof}
    First, we recall that
\begin{equation*}
    d_{k_n,j} \left( \delta^\mathrm{out, sub}_{k_n,j} \right) = \frac{1}{\sqrt{2 \pi}} e^{- \frac{(\delta^\mathrm{out, sub}_{k_n,j})^2}{2}}\left( 1 + \frac{h_1(\delta^\mathrm{out, sub}_{k_n,j})}{\sqrt{k_n}} + \frac{h_2(\delta^\mathrm{out, sub}_{k_n,j})}{k_n} + o_{p}\left ( \frac{1}{n} \right) \right)
\end{equation*}
 Taking the inverse, and developing it we get
\begin{align*}
    &\left(d_{k_n,j} \left( \delta^\mathrm{out, sub}_{k_n,j} \right) \right)^{- 1} \\
    = &\sqrt{2 \pi} e^{ \frac{(\delta^\mathrm{out, sub}_{k_n,j})^2}{2}}\left( 1 + \frac{h_1(\delta^\mathrm{out, sub}_{k_n,j})}{\sqrt{k_n}} + \frac{h_2(\delta^\mathrm{out, sub}_{k_n,j})}{k_n} + o_{p}\left ( \frac{1}{n} \right) \right)^{-1}\\
    = &\sqrt{2 \pi} e^{ \frac{(\delta^\mathrm{out, sub}_{k_n,j})^2}{2}}\left( 1 - \frac{h_1(\delta^\mathrm{out, sub}_{k_n,j})}{\sqrt{k_n}} - \frac{h_2(\delta^\mathrm{out, sub}_{k_n,j})}{k_n} + \frac{1}{2}\frac{\left(h_1(\delta^\mathrm{out, sub}_{k_n,j})\right)^2}{k_n}  + o_{p}\left ( \frac{1}{n} \right) \right)
\end{align*}

Combining with the other result

 \begin{equation*}
    d_{k_n-1,j} \left( \delta^\mathrm{in, sub}_{k_n,j} \right) = \frac{1}{\sqrt{2 \pi}} e^{- \frac{(\delta^\mathrm{in, sub}_{k_n,j})^2}{2}}\left( 1 + \frac{h_1(\delta^\mathrm{in, sub}_{k_n,j})}{\sqrt{k_n}} + \frac{h_2(\delta^\mathrm{in, sub}_{k_n,j})}{k_n} + o_{p}\left ( \frac{1}{n} \right) \right) 
\end{equation*}


Putting them together gives
\begin{align*}
    &\frac{d_{k_n-1,j} \left( \delta^\mathrm{in, sub}_{k_n,j} \right)}{d_{k_n,j} \left( \delta^\mathrm{out, sub}_{k_n,j} \right)}\\
    = &e^{\frac{(\delta^\mathrm{out, sub}_{k_n,j})^2  -(\delta^\mathrm{in, sub}_{k_n,j})^2}{2}} \left( 1 - \frac{h_1(\delta^\mathrm{out, sub}_{k_n,j})}{\sqrt{k_n}} - \frac{h_2(\delta^\mathrm{out, sub}_{k_n,j})}{k_n} + \frac{1}{2}\frac{\left(h_1(\delta^\mathrm{out, sub}_{k_n,j})\right)^2}{k_n}  + o_{p}\left ( \frac{1}{n} \right) \right)\\
    &\left( 1 + \frac{h_1(\delta^\mathrm{in, sub}_{k_n,j})}{\sqrt{k_n}} + \frac{h_2(\delta^\mathrm{in, sub}_{k_n,j})}{k_n} + o_{p}\left ( \frac{1}{n} \right) \right)\\
    = &e^{\frac{(\delta^\mathrm{out, sub}_{k_n,j})^2  -(\delta^\mathrm{in, sub}_{k_n,j})^2}{2}}  \Bigg( 1 + \frac{ h_1(\delta^\mathrm{in, sub}_{k_n,j}) - h_1(\delta^\mathrm{out, sub}_{k_n,j})}{\sqrt{k_n}}\\
    &\qquad\qquad+ \frac{1}{k_n} \left( \frac{1}{2}\left(h_1(\delta^\mathrm{out, sub}_{k_n,j})\right)^2    -  h_1(\delta^\mathrm{in, sub}_{k_n,j}) h_1(\delta^\mathrm{out, sub}_{k_n,j})  \right) + o_{p}\left ( \frac{1}{n} \right) \Bigg)
\end{align*}

Now, we proceed to do Taylor expansion of the exponential function, to get



\begin{align*}
    e^{\frac{(\delta^\mathrm{out, sub}_{k_n,j})^2  -(\delta^\mathrm{in, sub}_{k_n,j})^2}{2}} = 1 + \frac{1}{2}(\delta^\mathrm{out, sub}_{k_n,j})^2  -\frac{1}{2} (\delta^\mathrm{in, sub}_{k_n,j})^2 + \frac{1}{8} \left( (\delta^\mathrm{out, sub}_{k_n,j})^2  -(\delta^\mathrm{in, sub}_{k_n,j})^2 \right)^2 + o_{p}\left ( \frac{1}{n} \right)
\end{align*}


Putting everything together, we get

\begin{align*}
    &(1- \rho) + \rho \sqrt{\frac{k_n}{k_n -1}} \frac{d_{k_n - 1,j} \left( \delta^\mathrm{in, sub}_{k_n,j} \right)}{d_{k_n,j} \left( \delta^\mathrm{out, sub}_{k_n,j} \right)}\\
    = &1 +  \frac{\rho}{2 k_n} + \frac{\rho}{2}(\delta^\mathrm{out, sub}_{k_n,j})^2  -\frac{\rho}{2} (\delta^\mathrm{in, sub}_{k_n,j})^2 + \frac{\rho}{8} \left( (\delta^\mathrm{out, sub}_{k_n,j})^2  -(\delta^\mathrm{in, sub}_{k_n,j})^2 \right)^2 \\
    &+ \rho \frac{ h_1(\delta^\mathrm{in, sub}_{k_n,j}) - h_1(\delta^\mathrm{out, sub}_{k_n,j})}{\sqrt{k_n}}  + \frac{\rho}{k_n} \left( \frac{1}{2}\left(h_1(\delta^\mathrm{out, sub}_{k_n,j})\right)^2    -  h_1(\delta^\mathrm{in, sub}_{k_n,j}) h_1(\delta^\mathrm{out, sub}_{k_n,j})  \right) + o_{p}\left ( \frac{1}{n} \right) 
\end{align*}

Finally, another Taylor expansion of the logarithm gives
    
\begin{align*}
    &\log \left( (1- \rho) + \rho \sqrt{\frac{k_n}{k_n -1}} \frac{d_{k_n - 1,j} \left( \delta^\mathrm{in, sub}_{k_n,j} \right)}{d_{k_n,j} \left( \delta^\mathrm{out, sub}_{k_n,j} \right)} \right)\\
    = &1 +  \frac{\rho}{2 k_n} + \frac{\rho}{2}(\delta^\mathrm{out, sub}_{k_n,j})^2  -\frac{\rho}{2} (\delta^\mathrm{in, sub}_{k_n,j})^2 + \frac{\rho}{8} \left( (\delta^\mathrm{out, sub}_{k_n,j})^2  -(\delta^\mathrm{in, sub}_{k_n,j})^2 \right)^2 \\
    &+ \rho \frac{ h_1(\delta^\mathrm{in, sub}_{k_n,j}) - h_1(\delta^\mathrm{out, sub}_{k_n,j})}{\sqrt{k_n}}  + \frac{\rho}{k_n} \left( \frac{1}{2}\left(h_1(\delta^\mathrm{out, sub}_{k_n,j})\right)^2    -  h_1(\delta^\mathrm{in, sub}_{k_n,j}) h_1(\delta^\mathrm{out, sub}_{k_n,j})  \right) + o_{p}\left ( \frac{1}{n} \right) \\
    &- \frac{1}{2} \frac{\rho^2}{4} \left( (\delta^\mathrm{out, sub}_{k_n,j})^2  -(\delta^\mathrm{in, sub}_{k_n,j})^2 \right)^2 - \frac{\rho^2}{2 k_n} \left (h_1(\delta^\mathrm{in, sub}_{k_n,j}) - h_1(\delta^\mathrm{out, sub}_{k_n,j}) \right)^2 
\end{align*}

which after simplifications
\begin{align*}
    &\log \left( (1- \rho) + \rho \sqrt{\frac{k_n}{k_n -1}} \frac{d_{k_n - 1,j} \left( \delta^\mathrm{in, sub}_{k_n,j} \right)}{d_{k_n,j} \left( \delta^\mathrm{out, sub}_{k_n,j} \right)} \right)\\
    =  &\frac{\rho}{2}(\delta^\mathrm{out, sub}_{k_n,j})^2  -\frac{\rho}{2} (\delta^\mathrm{in, sub}_{k_n,j})^2   + \rho \frac{ h_1(\delta^\mathrm{in, sub}_{k_n,j}) - h_1(\delta^\mathrm{out, sub}_{k_n,j})}{\sqrt{k_n}}\\
    &+  \frac{\rho}{2 k_n}  + \frac{\rho( 1 - \rho)}{8} \left( (\delta^\mathrm{out, sub}_{k_n,j})^2  -(\delta^\mathrm{in, sub}_{k_n,j})^2 \right)^2  + o_{p}\left ( \frac{1}{n} \right)  
\end{align*}
concludes the proof.


    
\end{proof}

\begin{lemma}{Expectation and variance computations for sub-sampling}\label{lem:exp_sub}
\begin{align*}
    &\expect_0 [ W_{n,j} ] = - \frac{\rho}{2} \frac{\left(\mu_j - z^\star_j \right)^2}{n \sigma_j^2} + o\left (\frac{1}{n} \right)  &&\Var_0 [ W_{n,j} ] = \rho  \frac{(\mu_j - z^\star_j)^2}{n \sigma_j^2} + o\left(\frac{1}{n} \right)\\
    &\expect_1 [ W_{n,j} ] = \frac{\rho}{2} \frac{\left(\mu_j - z^\star_j \right)^2}{n \sigma_j^2} + o\left (\frac{1}{n} \right) && \Var_1 [ W_{n,j} ] = \rho \frac{(z_j^\star - \mu_j)^2}{n \sigma_j^2} + o\left (\frac{1}{n} \right)
\end{align*}
\end{lemma}

\begin{proof}
    Let us recall that
    \begin{equation*}
        W_{n, j} \defn \frac{\rho}{2}(\delta^\mathrm{out, sub}_{k_n,j})^2  -\frac{\rho}{2} (\delta^\mathrm{in, sub}_{k_n,j})^2   + \rho \frac{\lambda_3 \left (\mu_j - z^\star_j \right)}{k_n \sigma_j} R_{k_n,j} + \frac{\rho}{2 k_n}  + \frac{\rho( 1 - \rho)}{8} \left( (\delta^\mathrm{out, sub}_{k_n,j})^2  -(\delta^\mathrm{in, sub}_{k_n,j})^2 \right)^2
    \end{equation*}
    
    where $R_{k_n,j} \defn \left(\delta^\mathrm{out, sub}_{k_n,j} \right)^2 + \delta^\mathrm{out, sub}_{k_n,j} \delta^\mathrm{in, sub}_{k_n,j} + \left(\delta^\mathrm{in, sub}_{k_n,j} \right)^2  - 3 $
    
    \underline{Under $H_0$}:
    
    Since $\expect_0(\hat{\mu}^\mathrm{sub}_{n,j}) = \mu_j$ and $\Var_0(\hat{\mu}^\mathrm{sub}_{n,j}) = \frac{1}{k_n} \sigma_j^2$, we get that
    \begin{align*}
        &\expect_0\left(\delta^\mathrm{out, sub}_{k_n,j} \right) = 0 && \expect_0\left(\left(\delta^\mathrm{out, sub}_{k_n,j} \right)^2 \right) = 1\\
        &\expect_0\left(\delta^\mathrm{in, sub}_{k_n,j} \right) = \frac{\mu_j - z^\star_j}{\sigma_j \sqrt{k_n - 1}} && \expect_0\left(\left(\delta^\mathrm{in, sub}_{k_n,j} \right)^2 \right) = 1 + \frac{1}{k_n} + \frac{\left(\mu_j - z^\star_j \right)^2}{k_n \sigma_j^2} + o\left(\frac{1}{n} \right)
    \end{align*}
    
    Also
    \begin{equation*}
        \expect_0\left( \delta^\mathrm{out, sub}_{k_n,j} \delta^\mathrm{in, sub}_{k_n,j} \right) = \sqrt{\frac{k_n}{k_n - 1}} = 1 + \frac{1}{2 k_n} + o\left (\frac{1}{n}\right)
    \end{equation*}
    
    Which means that
    \begin{equation*}
        \expect_0 \left( R_{k_n,j} \right) = \frac{3}{2k_n} + \frac{\left(\mu_j - z^\star_j \right)^2}{k_n \sigma_j^2} + o\left (\frac{1}{n}\right) = o(1)
    \end{equation*}
    
    We compute the following expectations:
    \begin{align*}
        \expect_0[\left(\delta^\mathrm{out, sub}_{k_n,j} \right)^3] &= \frac{1}{\sqrt{k_n}} \frac{\mu_{3, j}}{\sigma_j^3}\\
        \expect_0[\left(\delta^\mathrm{out, sub}_{k_n,j} \right)^4] &= 3 + \frac{1}{k_n} \left ( \frac{\mu_{4,j}}{\sigma_j^4} - 3 \right)\\
        \expect_0[\left(\delta^\mathrm{in, sub}_{k_n,j} \right)^4] &= 3 + \frac{1}{k_n} \left ( \frac{\mu_{4,j}}{\sigma_j^4} - 3 \right)+ \frac{6}{k_n} + \frac{4}{k_n} \frac{\mu_j - z^\star_j}{\sigma_j^4}\mu_{3,j}  + \frac{6}{k_n}  \frac{(\mu_j - z^\star_j)^2}{\sigma_j^2} + o\left(\frac{1}{n} \right) \\
        \expect_0[\left(\delta^\mathrm{out, sub}_{k_n,j} \right)^2 \left(\delta^\mathrm{in, sub}_{k_n,j} \right)^2 ] &= 3 + \frac{1}{k_n} \left ( \frac{\mu_{4,j}}{\sigma_j^4} - 3 \right) + \frac{3}{k_n} + \frac{2}{k_n} \frac{\mu_j - z^\star_j}{\sigma_j^4}\mu_{3,j}  + \frac{1}{k_n}  \frac{(\mu_j - z^\star_j)^2}{\sigma_j^2} + o\left(\frac{1}{n} \right) 
    \end{align*}
    
    All in all, we have
    \begin{align*}
        \expect_0 \left [ \left (  \left(\delta^\mathrm{out, sub}_{k_n,j} \right)^2 - \left(\delta^\mathrm{in, sub}_{k_n,j} \right)^2  \right)^2 \right ] &= \frac{4}{k_n}  \frac{(\mu_j - z^\star_j)^2}{\sigma_j^2} + o\left(\frac{1}{n} \right)
    \end{align*}
    
    Finally
    \begin{align*}
        \expect_0 [ W_{n,j} ] &=  - \rho \frac{\left(\mu_j - z^\star_j \right)^2}{2 k_n \sigma_j^2} + \frac{\rho (1 - \rho)}{2} \frac{\left(\mu_j - z^\star_j \right)^2}{ k_n \sigma_j^2}  + o\left (\frac{1}{n} \right)\\
        &= - \frac{\rho}{2} \frac{\left(\mu_j - z^\star_j \right)^2}{n \sigma_j^2} + o\left (\frac{1}{n} \right)
    \end{align*}
    
    On the other hand, 
    \begin{align*}
        \Var_0 [ W_{n,j} ] &= \expect_0 [ W_{n,j}^2 ] - \left (\expect_0 [ W_{n,j}] \right)^2 \notag\\
        &= \expect_0 [ W_{n,j}^2 ] + o\left (\frac{1}{n} \right) \notag\\
        &= \frac{\rho^2}{4} \expect_0 \left [ \left (  \left(\delta^\mathrm{out, sub}_{k_n,j} \right)^2 - \left(\delta^\mathrm{in, sub}_{k_n,j} \right)^2  \right)^2 \right ] + o\left (\frac{1}{n} \right) \notag\\
        &= \frac{\rho^2}{k_n}  \frac{(\mu_j - z^\star_j)^2}{\sigma_j^2} + o\left(\frac{1}{n} \right) + o\left (\frac{1}{n} \right) \notag\\
        &= \frac{\rho}{n}  \frac{(\mu_j - z^\star_j)^2}{\sigma_j^2} + o\left(\frac{1}{n} \right)
    \end{align*}

    \underline{Under $H_1$}:
    
     Let $A = \{ \ind{\varsigma(n) \leq k_n} \}$ the event that $z^\star$ was sub-sampled. We have that $\mathrm{Pr}(A) = \rho$.
    
    We have the following
    \begin{align*}
        \expect_1(\hat{\mu}^\mathrm{sub}_{n,j}) &= \rho \expect_1(\hat{\mu}^\mathrm{sub}_{n,j} \mid A) + (1 - \rho) \expect_1(\hat{\mu}^\mathrm{sub}_{n,j} \mid A^c) \\
        &= \rho ( \mu_j + \frac{1}{k_n} (z_j^\star - \mu_j)) + (1 - \rho)\mu_j\\
        &=\mu_j + \frac{1}{n} \left( z^\star_j - \mu_j\right)\\
        \Var_1\left(\hat{\mu}^\mathrm{sub}_{n,j}\right) &= \expect_1\left[\left(\hat{\mu}^\mathrm{sub}_{n,j} - \mu_j - \frac{1}{n} \left( z^\star_j - \mu_j\right) \right)^2\right]\\
        &= \rho \expect_1\left[\left(\hat{\mu}^\mathrm{sub}_{n,j} - \mu_j - \frac{1}{n} \left( z^\star_j - \mu_j\right) \right)^2 \mid A\right] + (1 - \rho) \expect_1\left[\left(\hat{\mu}^\mathrm{sub}_{n,j} - \mu_j + \frac{1}{n} \left( z^\star_j - \mu_j\right) \right)^2 \mid A^c\right] \\
        &= \rho \left( \frac{k_n - 1}{k_n^2} \sigma_j^2 + \left( \frac{1 - \rho}{k_n} (z_j^\star - \mu_j)\right)^2 \right) + (1 - \rho)\left( \frac{1}{k_n} \sigma_j^2 + \frac{\rho^2}{k_n^2} \left( z^\star_j - \mu_j\right)^2 \right)\\
        &= \frac{\sigma_j^2}{k_n} \left( 1 - \frac{\rho}{k_n}\right) + \frac{\rho(1 - \rho)}{k_n^2} \left( z^\star_j - \mu_j\right)^2 \\
    \end{align*}
    
    Again, we compute the following expectations:
    \begin{align*}
        \expect_1\left(\left(\delta^\mathrm{out, sub}_{k_n,j} \right)^2\right) &= \frac{k_n}{\sigma_j^2} \left( \Var_1\left(\hat{\mu}^\mathrm{sub}_{n,j}\right) + \frac{\left(\mu_j - z^\star_j \right)^2}{n^2}   \right)\\
        &=  \frac{k_n}{\sigma_j^2} \left( \frac{\sigma_j^2}{k_n} \left( 1 - \frac{\rho}{k_n}\right) + \frac{\rho(1 - \rho)}{k_n^2} \left( z^\star_j - \mu_j\right)^2  + \rho^2\frac{\left(\mu_j - z^\star_j \right)^2}{k_n^2}   \right)\\
        &= 1 - \frac{\rho}{k_n} + \rho \frac{\left(\mu_j - z^\star_j \right)^2}{k_n \sigma_j^2}
    \end{align*}
    
    and
    \begin{align*}
        \expect_1\left(\left(\delta^\mathrm{in, sub}_{k_n,j} \right)^2\right) &= \frac{k_n^2}{(k_n - 1)\sigma_j^2} \left( \Var_1\left(\hat{\mu}^\mathrm{sub}_{n,j}\right) + \left( \frac{1 - \rho}{k_n} (z_j^\star - \mu_j)\right)^2   \right)\\
        &=  \frac{k_n^2}{(k_n - 1)\sigma_j^2} \left( \frac{\sigma_j^2}{k_n} \left( 1 - \frac{\rho}{k_n}\right) + \frac{\rho(1 - \rho)}{k_n^2} \left( z^\star_j - \mu_j\right)^2  + \left( \frac{1 - \rho}{k_n} (z_j^\star - \mu_j)\right)^2  \right)\\
        &= 1 + \frac{1 - \rho}{k_n} + \left (1 - \rho \right)\frac{\left(\mu_j - z^\star_j \right)^2}{k_n \sigma_j^2} + o\left (\frac{1}{n} \right)
    \end{align*}
    
    
    Thus again, we get
    \begin{align*}
        &\expect_1\left(\delta^\mathrm{out, sub}_{k_n,j} \right) = \sqrt{\rho} \frac{ z^\star_j - \mu_j}{\sigma_j \sqrt{n}}  && \expect_1\left(\left(\delta^\mathrm{out, sub}_{k_n,j} \right)^2 \right) = 1 - \frac{1}{n} + \frac{\left(\mu_j - z^\star_j \right)^2}{n \sigma_j^2} \\
        &\expect_1\left(\delta^\mathrm{in, sub}_{k_n,j} \right) = (1 - \rho) \frac{ \mu_j - z^\star_j}{\sigma_j \sqrt{k_n - 1}}  && \expect_1\left(\left(\delta^\mathrm{in, sub}_{k_n,j} \right)^2 \right) = 1 + \frac{1 - \rho}{k_n} + \left (1 - \rho \right)\frac{\left(\mu_j - z^\star_j \right)^2}{k_n \sigma_j^2} + o\left (\frac{1}{n} \right)
    \end{align*}
    
    Now, we find that
    \begin{align*}
        \expect_1\left(\left(\delta^\mathrm{out, sub}_{k_n,j} \right)^2\right) - \expect_1\left(\left(\delta^\mathrm{in, sub}_{k_n,j} \right)^2\right) = - \frac{1}{k_n} + \left (2 \rho - 1 \right) \frac{\left(\mu_j - z^\star_j \right)^2}{k_n \sigma_j^2} + o\left (\frac{1}{n} \right)
    \end{align*}
    
    Also,
    \begin{align*}
        \expect_1\left( \delta^\mathrm{out, sub}_{k_n,j} \delta^\mathrm{in, sub}_{k_n,j} \right) &= \expect_1 \left ( \sqrt{\frac{k_n}{k_n - 1}} \left (\delta^\mathrm{out, sub}_{k_n,j} \right)^2 + \frac{\mu_j - z^\star_j}{\sigma_j \sqrt{k_n - 1}} \delta^\mathrm{out, sub}_{k_n,j}\right)\\
        &= \sqrt{\frac{k_n}{k_n - 1}} \left ( 1 - \frac{1}{n} + \frac{\left(\mu_j - z^\star_j \right)^2}{n \sigma_j^2} \right ) + \frac{\mu_j - z^\star_j}{\sigma_j \sqrt{k_n - 1}} \left ( \sqrt{\rho} \frac{ z^\star_j - \mu_j}{\sigma_j \sqrt{n}} \right)\\
        &= 1 + o(1)
    \end{align*}
    
    Which means that
    \begin{equation*}
        \expect_1 \left( R_{k_n,j} \right) = o(1)
    \end{equation*}
    
    Again, following the same steps as in $H_0$, we arrive at 
    \begin{align*}
        \expect_1 \left [ \left (  \left(\delta^\mathrm{out, sub}_{k_n,j} \right)^2 - \left(\delta^\mathrm{in, sub}_{k_n,j} \right)^2  \right)^2 \right ] &= \frac{4}{k_n}  \frac{(\mu_j - z^\star_j)^2}{\sigma_j^2} + o\left(\frac{1}{n} \right)
    \end{align*}
    
    Finally
    \begin{align*}
        \expect_1 [ W_{n,j} ] &=  \rho (2 \rho - 1) \frac{\left(\mu_j - z^\star_j \right)^2}{2 k_n \sigma_j^2} + \frac{\rho (1 - \rho)}{2} \frac{\left(\mu_j - z^\star_j \right)^2}{ k_n \sigma_j^2}  + o\left (\frac{1}{n} \right) \notag\\
        &= \rho \frac{\left(\mu_j - z^\star_j \right)^2}{2 n \sigma_j^2} + o\left (\frac{1}{n} \right)
    \end{align*}
    
    Again, 
    \begin{align*}
        \Var_1 [ W_{n,j} ] &= \frac{\rho^2}{4} \expect_1 \left [ \left (  \left(\delta^\mathrm{out, sub}_{k_n,j} \right)^2 - \left(\delta^\mathrm{in, sub}_{k_n,j} \right)^2  \right)^2 \right ] + o\left (\frac{1}{n} \right) \notag\\
        &= \frac{\rho^2}{k_n}  \frac{(\mu_j - z^\star_j)^2}{\sigma_j^2} + o\left(\frac{1}{n} \right) + o\left (\frac{1}{n} \right) \notag\\
        &= \frac{\rho}{n}  \frac{(\mu_j - z^\star_j)^2}{\sigma_j^2} + o\left(\frac{1}{n} \right)
    \end{align*}
    
\end{proof}

\subsection{Effect of Misspecifiaction, Proof of Theorem~\ref{thm:misspec} }

\begin{reptheorem}{thm:misspec}[Leakage of a misspecified adversary]
We show that as $d, n \gets \infty$ while $d/n = \tau$,

- Under $H_0$, $\ell_n(\hat{\mu}_n; z^{\mathrm{targ}}, \mu, C_\sigma) \rightsquigarrow \mathcal{N}\left(-\frac{ m^\mathrm{targ}}{2}, m^\star \right)$.

- Under $H_1$, $\ell_n(\hat{\mu}_n; z^{\mathrm{targ}}, \mu, C_\sigma) \rightsquigarrow \mathcal{N}\left(\frac{m^\star - m^\mathrm{diff}}{2}, m^\star \right)$
    
Let the adversary $\mathcal{A}_\mathrm{miss}$ use $\ell_n(\hat{\mu}_n; z^{\mathrm{targ}}, \mu, C_\sigma)$ as the LR test function. Then,
\begin{center}
    $\lim_{n,d} \mathrm{Adv}_n(\mathcal{A}_\mathrm{miss}) = \Phi\left (\frac{|m^\mathrm{scal} |}{2 \sqrt{m^\star}}\right) - \Phi\left(-\frac{|m^\mathrm{scal} |}{2 \sqrt{m^\star}}\right)$.
\end{center}
\end{reptheorem}

\begin{proof}
        
\underline{\textbf{Step 1: Asymptotic expansion of the LR score}}

Directly using Equation~\eqref{eq:asymp_lr} from the proof in Section~\ref{sec:prf_thm_mean}, by only replacing $z^\star$ by $z^\mathrm{targ}$, we get

\begin{align*}
    \ell_n(\hat{\mu}_n; z^\star, \mu, C_\sigma) &= \frac{\tau}{2} + o_p(1) + \sum_{j = 1}^{d_n} Y^\mathrm{targ}_{n, j} 
\end{align*}
where

\begin{align*}
   Y^\mathrm{targ}_{n, j}  \defn \frac{1}{2} \left( \left(\delta^\mathrm{out, targ}_{n,j} \right)^2 -   \left(\delta^\mathrm{in, targ}_{n,j} \right)^2\right) + \frac{\lambda_3 \left (\mu_j - z^\mathrm{targ}_j \right)}{n \sigma_j} R^\mathrm{targ}_{n,j},
\end{align*}
and 

\begin{align*}
        \delta^\mathrm{out, targ}_{n,j} &\defn \frac{\sqrt{n}}{\sigma_j} \left (\hat{\mu}_{n,j} - \mu_j  \right ) \\
        \delta^\mathrm{in, targ}_{n,j} &\defn \frac{n}{\sqrt{n - 1}\sigma_j} \left (\hat{\mu}_{n,j} - \mu_j + \frac{1}{n} \left (\mu_j - z^\mathrm{targ}_j \right)  \right ) \\
        R^\mathrm{targ}_{n,j} &\defn \left(\delta^\mathrm{out, targ}_{n,j} \right)^2 + \delta^\mathrm{out, targ}_{n,j} \delta^\mathrm{in, targ}_{n,j} + \left(\delta^\mathrm{in, targ}_{n,j} \right)^2  - 3  
    \end{align*}    

\underline{\textbf{Step 2: Computing expectations and variances}} 

This is the step where the effect of misspecification appears.


    
    
    \underline{Under $H_0$}:
    
    Since $\expect_0(\hat{\mu}_{n,j}) = \mu_j$ and $\Var_0(\hat{\mu}_{n,j}) = \frac{1}{n} \sigma_j^2$, we get that
    \begin{align*}
        &\expect_0\left(\delta^\mathrm{out, targ}_{n,j} \right) = 0 && \expect_0\left(\left(\delta^\mathrm{out, targ}_{n,j} \right)^2 \right) = 1\\
        &\expect_0\left(\delta^\mathrm{in, targ}_{n,j} \right) = \frac{\mu_j - z^\mathrm{targ}_j}{\sigma_j \sqrt{n - 1}} && \expect_0\left(\left(\delta^\mathrm{in, targ}_{n,j} \right)^2 \right) = 1 + \frac{1}{n} + \frac{\left(\mu_j - z^\mathrm{targ}_j \right)^2}{n \sigma_j^2} + o\left(\frac{1}{n} \right)
    \end{align*}

    Using similar steps as in Lemma~\ref{lem:exp_comp}, we get
    \begin{equation*}
        \expect_0 [ Y^\mathrm{targ}_{n,j} ] = - \frac{1}{2n} - \frac{(z_j^\mathrm{targ} - \mu_j)^2}{2 n \sigma_j^2} + o\left (\frac{1}{n} \right)
    \end{equation*}
    
    And
    \begin{equation*}
        \Var_0 [ Y^\mathrm{targ}_{n,j} ] = \frac{1}{n}  \frac{(\mu_j - z^\mathrm{targ}_j)^2}{\sigma_j^2} + o\left(\frac{1}{n} \right)
    \end{equation*}

    \underline{Under $H_1$}:
    
    Since $\expect_1(\hat{\mu}_{n,j}) = \mu_j + \frac{1}{n} \left( z^\star_j - \mu_j\right)$ and $\Var_1(\hat{\mu}_{n,j}) = \frac{n - 1}{n^2}\sigma_j^2$,
    let $$\delta^\mathrm{in, true}_{n,j} \defn \frac{n}{\sqrt{n - 1}\sigma_j} \left (\hat{\mu}_{n,j} - \mu_j + \frac{1}{n} \left (\mu_j - z^\star_j \right)  \right ),$$ 
    
    and thus 
    
     \begin{align*}
         &\expect_1\left(\delta^\mathrm{in, true}_{n,j} \right) = 0, && \expect_1\left(\left(\delta^\mathrm{in, true}_{n,j} \right)^2 \right) = 1
     \end{align*}
     
     On the other hand, we can rewrite
     
     \begin{align*}
         \delta^\mathrm{out, targ}_{n,j} &= \sqrt{\frac{n - 1}{n}} \delta^\mathrm{in, true}_{n,j} + \frac{ z^\star_j - \mu_j}{\sigma_j \sqrt{n}} \\
        \delta^\mathrm{in, targ}_{n,j} &=  \delta^\mathrm{in, true}_{n,j} + \frac{z^\star_j - z^\mathrm{targ}_j}{\sigma_j \sqrt{n - 1}}
     \end{align*}
     
    Thus, we get that
    \begin{align*}
        &\expect_1\left(\delta^\mathrm{out, targ}_{n,j} \right) = \frac{ z^\star_j - \mu_j}{\sigma_j \sqrt{n}}  && \expect_1\left(\left(\delta^\mathrm{out, targ}_{n,j} \right)^2 \right) = 1 - \frac{1}{n} + \frac{\left(\mu_j - z^\star_j \right)^2}{n \sigma_j^2} \\
        &\expect_1\left(\delta^\mathrm{in, targ}_{n,j} \right) = \frac{z^\star_j - z^\mathrm{targ}_j}{\sigma_j \sqrt{n - 1}} && \expect_1\left(\left(\delta^\mathrm{in, targ}_{n,j} \right)^2 \right) = 1 + \frac{1}{(n - 1) \sigma_j^2} \left( z^\star_j - z^\mathrm{targ}_j \right)^2
    \end{align*}
    
    Also
    \begin{align*}
        \expect_1\left( \delta^\mathrm{out, targ}_{n,j} \delta^\mathrm{in, targ}_{n,j} \right) &= \sqrt{\frac{n - 1}{n}} + \frac{\left(z^\star_j -\mu_j \right)\left(z^\star_j - z^\mathrm{targ} \right)}{ \sigma_j^2 \sqrt{n} \sqrt{n - 1}}\\
        &= 1 - \frac{1}{2n} + \frac{\left(z^\star_j -\mu_j \right)\left(z^\star_j - z^\mathrm{targ} \right)}{ n \sigma_j^2} + o\left (\frac{1}{n}\right)
    \end{align*}
    

    Which means that
    \begin{equation*}
        \expect_1 \left( R^\mathrm{targ}_{n,j} \right) = o(1)
    \end{equation*}
    
    Finally
    \begin{equation*}
        \expect_1 [ Y^\mathrm{targ}_{n,j} ] = - \frac{1}{2n} + \frac{(z_j^\star - \mu_j)^2 - \left( z^\star_j - z^\mathrm{targ}_j \right)^2}{2 n \sigma_j^2} + o\left (\frac{1}{n} \right)
    \end{equation*}

    \begin{equation*}
        \Var_1 [ Y^\mathrm{targ}_{n,j} ] = \frac{1}{n}  \frac{(\mu_j - z^\mathrm{targ}_j)^2}{\sigma_j^2} + o\left(\frac{1}{n} \right)
    \end{equation*}
    
All in all, we proved that

\begin{align*}
    &\expect_0 [ Y^\mathrm{targ}_{n,j} ] =  - \frac{1}{2n} - \frac{(z_j^\mathrm{targ} - \mu_j)^2}{2 n \sigma_j^2} + o\left (\frac{1}{n} \right)  \\
    &\Var_0 [ Y^\mathrm{targ}_{n,j} ] = \frac{1}{n}  \frac{(\mu_j - z^\mathrm{targ}_j)^2}{\sigma_j^2} + o\left(\frac{1}{n} \right)\\
    &\expect_1 [ Y^\mathrm{targ}_{n,j} ] = - \frac{1}{2n} + \frac{(z_j^\star - \mu_j)^2 - \left( z^\star_j - z^\mathrm{targ}_j \right)^2}{2 n \sigma_j^2} + o\left (\frac{1}{n} \right) \\
    &\Var_1 [ Y^\mathrm{targ}_{n,j} ] = \frac{1}{n}  \frac{(\mu_j - z^\mathrm{targ}_j)^2}{\sigma_j^2} + o\left(\frac{1}{n} \right)
\end{align*}

\underline{\textbf{Step3: Concluding using the Lindeberg-Feller CLT
}}

\underline{\textit{Under $H_0$:}}

Using the results of Step2, we have

\begin{itemize}
    \item $\sum_{j = 1}^{d_n} \expect_0 [ Y^\mathrm{targ}_n,j ] \rightarrow -\frac{\tau}{2} - \frac{m^\mathrm{targ}}{2}$

    \item $\sum_{j = 1}^{d_n} V_0 [ Y_n,j ] \rightarrow m^\mathrm{targ}$
\end{itemize}

Similarly to Lemma~\ref{lem:lind_cond}, we can show that $Y^\mathrm{targ}_{n,j}$ verify the Lindeberg-Feller condition, i.e.

\begin{equation*}
    \sum_{j = 1}^{d_n} \expect_0 \left [ (Y^\mathrm{targ}_{n, j})^2 \ind{|Y^\mathrm{targ}_{n,j} | > \epsilon }  \right] \rightarrow 0
\end{equation*}

for every $\epsilon >0$.

We conclude using the Lindeberg-Feller CLT (Theorem~\ref{thm:lind_feller}) that $\sum_{j = 1}^{d_n} Y^\mathrm{targ}_{n,j} \rightsquigarrow \mathcal{N}\left ( -\frac{\tau}{2} - \frac{m^\mathrm{targ}}{2},  m^\mathrm{targ}\right) $, and thus

\begin{equation*}
    \ell_n(\hat{\mu}_n; z^\mathrm{targ}, \mu, C_\sigma) 
    \rightsquigarrow \mathcal{N}\left ( - \frac{m^\mathrm{targ}}{2},  m^\mathrm{targ}\right) 
\end{equation*}

Similarly, \underline{\textit{Under $H_1$:}}

We use the results of Step3
\begin{itemize}
    \item $\sum_{j = 1}^{d_n} \expect_1 [ Y^\mathrm{targ}_{n,j} ] \rightarrow -\frac{\tau}{2} + \frac{m^\star - m^\mathrm{diff}}{2}$

    \item $\sum_{j = 1}^{d_n} V_1 [ Y^\mathrm{targ}_{n,j} ] \rightarrow m^\mathrm{targ}$
\end{itemize}

Similarly to Lemma~\ref{lem:lind_cond}, we can show that $Y^\mathrm{targ}_{n,j}$ verify the Lindeberg-Feller condition, i.e.

\begin{equation*}
    \sum_{j = 1}^{d_n} \expect_1 \left [(Y^\mathrm{targ}_{n,j})^2 \ind{|Y^\mathrm{targ}_{n,j} | > \epsilon }  \right] \rightarrow 0
\end{equation*}

for every $\epsilon >0$.

We conclude using the Lindeberg-Feller CLT (Theorem~\ref{thm:lind_feller}) that $\sum_{j = 1}^{d_n} Y^\mathrm{targ}_{n,j} \rightsquigarrow \mathcal{N}\left (  -\frac{\tau}{2} + \frac{m^\star - m^\mathrm{diff}}{2},  m^\mathrm{targ}\right) $, and thus

\begin{equation*}
    \ell_n(\hat{\mu}_n; z^\mathrm{targ}, \mu, C_\sigma) 
    \rightsquigarrow \mathcal{N}\left ( \frac{m^\star - m^\mathrm{diff}}{2},  m^\mathrm{targ}\right) 
\end{equation*}

\underline{\textbf{Step4: Getting the advantage of the misspecified attack.
}}

We use that $\TV{\mathcal{N}\left(\mu_0, \sigma_0^2 \right)}{\mathcal{N}\left(\mu_1, \sigma_0^2 \right)} = \Phi \left ( \frac{|\mu_0 - \mu_1 |}{2 \sigma_0} \right) - \Phi \left (- \frac{|\mu_0 - \mu_1 |}{2 \sigma_0} \right)$, so that

\begin{align*}
    \lim_{n,d} \mathrm{Adv}_n(\mathcal{A}_\mathrm{miss}) &= \TV{\mathcal{N}\left(- \frac{m^\mathrm{targ}}{2},  m^\mathrm{targ}\right)}{\mathcal{N} \left(\frac{m^\star - m^\mathrm{diff}}{2},  m^\mathrm{targ} \right)}  \\
    &= \Phi\left (\frac{| m^\star + m^\mathrm{targ} - m^\mathrm{diff} |}{4 \sqrt{m^\star}}\right) - \Phi\left (-\frac{| m^\star + m^\mathrm{targ} - m^\mathrm{diff} |}{4 \sqrt{m^\star}}\right) \\
    &= \Phi\left (\frac{|m^\mathrm{scal} |}{2 \sqrt{ m^\mathrm{targ}}}\right) - \Phi\left(-\frac{|m^\mathrm{scal} |}{2 \sqrt{ m^\mathrm{targ}}}\right)
\end{align*}
because $m^\mathrm{diff} = m^\star + m^\mathrm{targ} - 2 m^\mathrm{scal}$.
\end{proof}

\begin{remark}[Simple way to get the expectations computation]

Thanks to Remark~\ref{rmk:asymp}, we recall that

\begin{align*}
    \ell_n(\hat{\mu}_n; z^\mathrm{targ}, \mu, C_\sigma) \approx (z^\mathrm{targ} - \mu)^T C_\sigma^{-1}(\hat{\mu}_n - \mu) 
    - \frac{1}{2n} \|^\mathrm{targ} - \mu\|^2_{C_\sigma^{-1}}
\end{align*}

And thus taking the expectation under $H_0$ and $H_1$, using that $ \expect_0( \hat{\mu}_n) = \mu $, $ \expect_1( \hat{\mu}_n) = \frac{n - 1}{n} \mu + \frac{1}{n} z^\star $ and $\Var_0(\hat{\mu}_n) = \Var_1(\hat{\mu}_n) = C_\sigma$, we get back the same expectations and variances values as the result of Theorem~\ref{thm:misspec}.

\end{remark}

%% file: appendix/audit_pseudo_code.tex
\section{THE WHITE-BOX FEDERATED LEARNING SETTING}\label{app:audt}

First, we discuss the white-box federated learning threat model. Then, we connect our canary choosing strategy (Algorithm~\ref{alg:cana_maha}) and covariance attack  (Algorithm~\ref{alg:cov_att}) to the literature.

\subsection{Threat Models of Attacks in Supervised Learning}

In supervised learning, the mechanism to be attacked is a learning algorithm that takes as input a dataset $D$ and outputs a machine learning model $o \defn f$. The dataset $D$ is composed of $n$ tuples $(x_i, y_i)$ where $x_i$ is a feature and $y_i$ is a label, i.e. $D \defn \{(x_1, y_1), \dots, (x_n, y_n)\}$. The machine learning model $f$ produced can then be queried for an input feature $x$ to get a label $y = f(x)$. The model $f$ is generally found by minimising over a class of models $\mathcal{F}$ some type of error $\ell$ in the input dataset $D$, i.e. $f \defn \arg\min_{g \in \mathcal{F}} \ell(g, D)$. 

The class of models $\mathcal{F}$ can be parameterised by $\theta \in \real^d$, i.e. $f = f_\theta$. In this case, the threat model depends on whether the adversary has access to the parameter $\theta$ or only query access to the model $f_\theta$. The setting where the adversary can observe the parameter $\theta \in \real^d$ is called the \textbf{white-box setting}. On the other hand, when the adversary can only query the final model $f_\theta$, i.e. send input features $x$ to $f$ and observe the outputs $y = f(x)$ is called the \textbf{black-box setting}.

In the parameterised setting, the quintessential training algorithms are based on Gradient Descent. The Gradient Descent algorithm start with an initial parameter $\theta_0 \in \real^d$, and then updates sequentially the parameter at each step $t$ by $\theta_t \defn \theta_{t - 1} - \eta \nabla_{\theta_{t-1}} \ell(\theta_{t - 1}, d)$. In the white-box setting, the adversary may have access to only the final parameter $\theta_T$, and we call this setting \textbf{white-box final parameter}~\citep{nasr2023tight}. The adversary can have access to all (or a subset of) the intermediate parameters sequence $(\theta_0, \dots, \theta_T)$, and we call this setting \textbf{white-box federated learning} setting~\citep{maddock2022canife}.

Fixed-target MI game for the empirical mean mechanism can directly provide an adversary and a canary design strategy to attack/audit gradient descent algorithms in the white box federated learning setting.

\noindent \textbf{Extended details on the white-box federated learning threat model.} Gradient Descent algorithms start with an initial parameter $\theta_0 \in \real^d$, and then update sequentially the parameter at each step $t$ by $$\theta_t \defn \theta_{t - 1} - \eta \nabla_{\theta_{t - 1}} Q(\theta_{t- 1}),$$ where $\eta$ is the learning rate, and $Q(\theta_{t- 1})$ is a quantity that depends on the loss on "some input samples". For example,

\begin{itemize}
    \item[(a)] in batch gradient descent $$\nabla_{\theta_{t - 1}} Q(\theta_{t- 1}) \defn \frac{1}{n} \sum_{i = 1}^n  \nabla_{\theta_{t - 1}} \ell(f_{\theta_{t - 1}}(x_i), y_i)$$ is the gradient with respect to the whole dataset.
    \item[(b)] in mini-batch gradient descent, the dataset is divided into a set of mini-batches $D = \cup B_k$. At each step $t$, a mini-batch $B$ is sampled uniformly and $$\nabla_{\theta_{t - 1}} Q(\theta_{t- 1}) \defn \frac{1}{|B|} \sum_{i \in B}  \nabla_{\theta_{t - 1}} \ell(f_{\theta_{t - 1}}(x_i), y_i).$$ We call $|B|$ the batch size. 
    \item[(c)] in stochastic gradient descent,  $$\nabla_{\theta_{t - 1}} Q(\theta_{t- 1}) \defn \nabla_{\theta_{t - 1}} \ell(f_{\theta_{t - 1}}(x_i), y_i),$$ where $i \sim \mathcal{U}([1,n])$ is sampled randomly from $\{1, n\}$. 
    \item[(d)] in DP-SGD~\cite{dpsgd}, $$\nabla_{\theta_{t - 1}} Q(\theta_{t- 1}) \defn \left(\frac{1}{|B|}  \sum_{i \in B} \operatorname{Clip}_{C} \left[ \nabla_{\theta_{t - 1}} \ell(f_{\theta_{t - 1}}(x_i), y_i) \right] \right)   + \mathcal{N}\left( 0, \gamma^2 C^2 I_d \right),$$ where $B$ is again a batch uniformly sampled, $ \operatorname{Clip}_{C}(x) \defn \min\{1, C/ \|x\| \}  x$ is the clipping norm function, i.e. $\operatorname{ClipNorm}_{C}(x) = x $ if $\|x\| \leq C$ otherwise $\operatorname{ClipNorm}_{C}(x) = C \frac{x}{\|x\|}$. Here, $C > 0$ is a gradient norm bound and $\gamma>0$ is the noise magnitude.
 DP-SGD can be shown to verify a DP constraint, where the privacy budget depends on the clipping gradient norm $C$, the noise $\gamma$, the batch size $|B|$ and the number of gradient iterations $T$.
\end{itemize}

The goal is to attack gradient descent algorithms. Specifically, the mechanism to be attacked is the "training" algorithm that takes as input the private dataset $D \defn \{(x_i, y_i)\}_{i = 1}^n$, and produces sequentially the parameter estimates $\{\theta_t \}_{t = 1}^T$. 
In the white-box federated learning setting, the adversary can observe the full sequence of iterates $\{\theta_t \}_{t = 1}^T$.

\begin{figure}[t!]
    \centering
    \includegraphics[width= \textwidth]{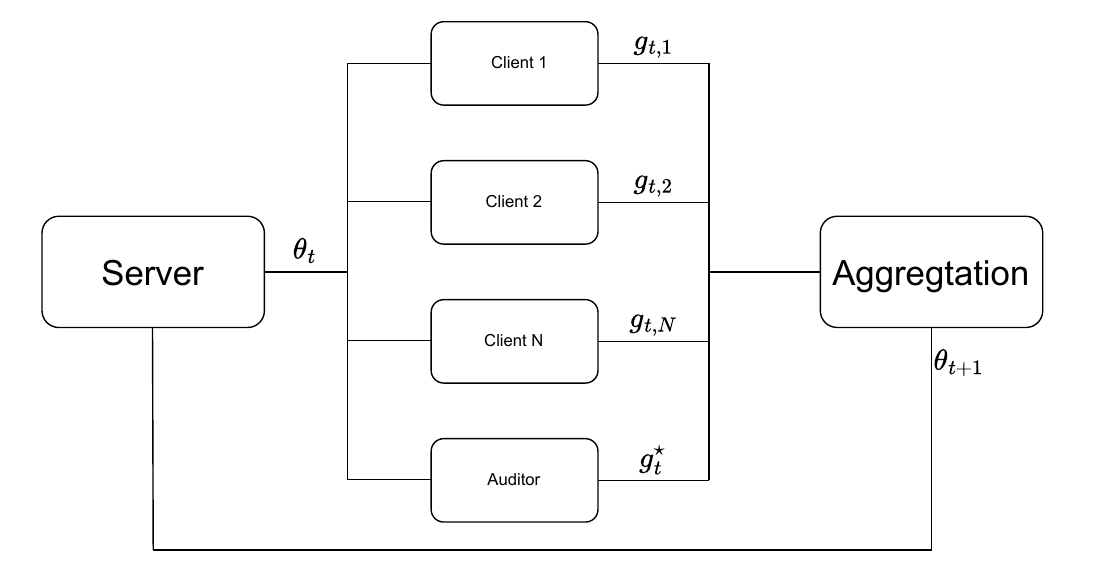}
    \caption{The White-box Federated learning threat model. At each step, the server sends a global model $\theta_t$ to each client. The adversary is a client, too. Each client $i$ computes and sends the local update $g_{i,t}$ to the server. The adversary also computes the local update $g^\star$ on a canary $z^\star$, and sends it to the global server with probability $1/2$. The server aggregates the updates received to compute $\theta_{t+1}$~\citep{maddock2022canife}.}\label{fig:fl}
\end{figure}

The white-box federated learning setting is a fundamental setting for understanding privacy auditing. In addition, this setting has practical uses too. For example, a potential use of this setting is for ``debugging" implementations of private gradient descent algorithms. In this case, the adversary is the programmer trying to release an implementation of their favourite private descent algorithm. To verify the guarantees of their algorithm, the programmer runs a white-box federated learning attack on their implementation and compares the empirical privacy guarantees retrieved with the theoretical analysis. The white-box federated learning setting captures well this use case since the adversary \textit{is} the programmer themselves, and thus have access to all the intermediate iterates and gradients.   

On the other hand, as the name of the setting suggests, another natural application of this threat model is auditing the private Federated Learning (FL) protocol. In a standard FL protocol (Figure~\ref{fig:fl}), a server computes a global model $\theta_t$ at each step $t$ by taking the (noisy) average of client updates, where each client computes a gradient estimate using their local datasets. We suppose that the auditor is also a client in an "honest-but-curious" FL protocol, i.e. the auditor plays the role of the "adversary" but follows the rules of the protocol: At each step $t$ of the protocol, the server sends a global model $\theta_{t}$ to each client. Then, each client $i$ computes their local gradient update $g_{t,i}$ and sends it to the server. The auditor, being also a client in the protocol, computes a gradient update $g^\star_t \defn \nabla_{\theta_{t}} \ell(f_{\theta_{t}(x^\star)}, y^\star) $ on their local "canary" datapoint $z^\star \defn (x^\star, y^\star)$, i.e. the most-leaking sample. Then, the auditor sends the gradient update to the server with probability $1/2$, or otherwise sends nothing to the server. In the next step of the interaction, the auditor observes the new updated global model $\theta_{t + 1}$. The goal of the auditor is to decide based solely on $\theta_{t + 1}$, $\theta_{t}$ and $g^\star_t$ whether the canary update was indeed sent to the server or not. This is exactly a tracing problem of the aggregation mechanism, which is generally a variant of the empirical mean. To design an audit in this setting, the auditor has to design two algorithms. First, an attack score to determine based on $\theta_{t + 1}$, $\theta_{t}$ and $g^\star_t$ whether $g^\star_t$ was included in $\theta_{t + 1}$ or not, i.e. a score function that takes as input $\theta_{t + 1}$, $\theta_{t}$ and $g^\star_t$ and outputs a score that would be high if the canary was included, otherwise the score is low. The auditor also needs to choose a good canary $z^\star_t$ at each step $t$. A good canary should be an easy point to trace, i.e., one for which the target-dependent leakage is high.

To summarise, attacking a gradient descent algorithm (batch, minibatch and DP-SGD) in the white-box federated learning setting reduces to attacking variants of the empirical mean mechanism, applied to loss gradient data $\{ \nabla_{\theta_{t - 1}} \ell(f_{\theta_{t - 1}}(x_i), y_i)\}_{i = 1}^n$. Thus, our results from Section~\ref{sec:mia_mean} and Section~\ref{sec:defences} can readily be applied.


\subsection{Connection to the Existing Canary Selection Strategies}

\paragraph{Related Works.}
The main intuition for canary selection strategies in the literature is to propose heuristics to generate out-of-distribution data~\cite{jagielski2020auditing,  maddock2022canife, nasr2023tight, steinke2023privacy, andrew2023one}.
For example,~\cite{nasr2023tight} proposes the Dirac canary strategy, which suggests the use of canaries gradient updates with all the values zero except at a single index. The intuition behind this choice is that the sparse nature of the Dirac gradient makes it an out-of-distribution sample in natural datasets. For CIFAR10,~\cite{nasr2023tight} shows the effectiveness of such a canary choice for auditing neural nets in the white-box federated learning setting. The Dirac canary is a type of \textit{gradient canaries}, because it directly suggests what gradient $g^\star$ the auditor should include or not in the training. The other type of canaries is \textit{input canaries}. As the name suggests, input canaries are pairs of features and label $z^\star = (x^\star, y^\star)$ that the auditor chooses to include or not in the training. There are different heuristics in the literature to generate input canaries, i.e. mislabeled examples or blank examples~\citep{nasr2023tight}, or adversarial examples~\citep{jagielski2020auditing}.

On the other hand,~\cite{maddock2022canife} proposes CANIFE, an algorithm that learns to craft canaries by back-propagating to the input level the following loss function
\begin{equation}\label{eq:canife}
    \ell^\mathrm{CANIFE}(z^\star) \defn \sum_i^{n_r} \left \langle u_i, g^\star_t \right \rangle^2 + \max(C -  \|g^\star_t \|, 0)^2.
\end{equation}
Here, $g^\star_t \defn \nabla_{\theta_{t}} \ell(f_{\theta_t}(x^\star), y^\star) $ is the gradient of the loss of the canary $z^\star$ at step $t$, $u_i$ is the gradient of the loss with respect to a reference sample and $n_r$ is the number of reference samples. CANIFE can be used to craft an input canary $z^\star$ by directly minimising the loss $\ell^\mathrm{CANIFE}$.

\paragraph{Connection to our work.} In addition to proposing a new gradient and a new input canary strategy, our Mahalanobis score also explains the success of the heuristics presented above. Specifically, Dirac canaries, black examples, or mislabeled examples are all points with high Mahalanobis distance, thus making them great canary candidates. Our Mahalanobis score can also be run over "in-distribution" canary candidates to find the most leaking one over them. This could come in handy when the auditor, while trying to participate in the white-box audit protocol (e.g. Figure~\ref{fig:fl}), does not want to hurt the accuracy of the final model. Thus, the auditor wants to send gradient updates that are helpful for accuracy, i.e.  "in-distribution", but with a high enough Mahalanobis score to be distinguishable. The Mahalanobis score solves the tradeoff between the "accuracy of the model" and the "success of the MI attack" by choosing points with a moderate Mahalanobis score.

Finally, our Mahalanobis score also explains the CANIFE loss $\ell^\mathrm{CANIFE}$ of Equation~\eqref{eq:canife}. In~\cite[Appendix A]{maddock2022canife}, the intuition to explain the CANIFE loss starts by expressing the LR score between two Gaussian distributions. Then,~\cite{maddock2022canife} concludes that to make the two Guassians distinguishable (separable) enough, one should maximise $(g^\star_t)^T C^{-1} g^\star_t$, which is precisely the Mahalanobis score. Finally, they claim that maximising the score is "equivalent" to minimising $(g^\star_t)^T C g^\star_t$, which yields exactly the CANIFE loss $\ell^\mathrm{CANIFE}$ when substituting $C = \sum_i u_i u_i^{\top}$. Thus, the Mahalanobis score also explains the CANIFE loss, and our results rigorously justify the success of this approach beyond Gaussian distributions.

\subsection{Connection to White-box Scores in the Literature}

\paragraph{Related Works.} The scalar product~\cite{dwork2015robust} is the most popular score used in white-box attacks in the literature~\cite{maddock2022canife, nasr2023tight, steinke2023privacy, andrew2023one}. The scalar product score takes as input $\theta_{t + 1}$, $\theta_{t}$ and $g^\star_t$ and outputs the scalar product $\langle \theta_{t + 1} - \theta_{t}, g^\star_t \rangle$. This score is a direct application of the tracing attack against of~\cite{dwork2015robust} to the white-box federated learning setting, since $\theta_{t + 1} - \theta_{t}$ is an empirical-mean like quantity. 

On the other hand, \cite{leemann2023gaussian} proposes a different score attack, called Gradient Likelihood Ratio (GLiR) Attack, i.e.  Algorithm 1 in~\cite{leemann2023gaussian}. This attack is based on an analysis of the LR test of the empirical mean. However, the analysis of~\cite{leemann2023gaussian} arrives at a different score function compared to our results. Specifically, let $g_\text{batch}^t = \frac{\theta_{t+1} - \theta_{t}}{\eta}$ be the batch gradient. Then, the GLiR attack needs to estimate an empirical mean $\hat \mu_0$ and covariance $\hat C_0$ of reference data's gradient. Then, the attack computes the statistics $\hat{S} = (|B| - 1)(g_\text{batch}^t - g^\star_t)^T \hat C_0^{-1} (g_\text{batch}^t - g^\star_t)$ and $ \hat{K} = \| \hat C_0^{-1/2} (\hat{\mu}_0 - g^\star_t)\|$. The GLiR score is $$ s^\mathrm{GLiR}(g_\text{batch}^t, g^\star_t) = \log\left( F^{-1}_{\chi^2_d(|B| \hat{K})} \left( \hat{S} \right) \right)$$ where $F^{-1}_{\chi^2_d(\gamma)}$ is the inverse of the CDF of the non-central chi-squared distribution with $d$ degrees of freedom and non-centrality parameter $\gamma$ and $|B|$ is the batch size. For some threshold $\tau$, if $s^\mathrm{GLiR}(g_\text{batch}^t, g^\star_t) < \tau$, the GLiR attack suggests that $g^\star$ was included, otherwise it was not. Next, we comment on the relation between our analysis and that of~\citep{leemann2023gaussian}.

\paragraph{Connection to our work.} The covariance attack of Algorithm~\ref{alg:cov_att} is provably better than the scalar product attack, at the expense of estimating the inverse of the covariance matrix well. The shape of the covariance matrix is $d \times d$, where $d$ is the number of parameters of the model $\theta$. This means that storing and inverting this covariance matrix is computationally expensive for models with many parameters. A simple trick to deal with this problem is only running the attack on a subset of the parameters. For example, we can run the covariance attack over the last layer of a neural net. If the last layer has $d_\ell$ parameters, the covariance matrix becomes $d_\ell \times d_\ell$ with $d_\ell \ll d$.

We provide the following remarks to connect our analysis to that of~\cite{leemann2023gaussian}.

(a) In Step 1 of the proof, in \cite[Section E.1.]{leemann2023gaussian} declares that "We suppose that the number of averaged samples is sufficiently large such that we can apply the Central Limit Theorem", and thus, considers\footnote{Following equations are restatements of Equations 45 and 46 of~\cite{leemann2023gaussian} using our notations.} under $H_0$,
\begin{equation}
 \hat{\mu}_n \sim \mathcal{N}\left( \mu, \frac{1}{n} C_\sigma \right)
\end{equation}
and under $H_1$,
\begin{equation}
 \hat{\mu}_n \sim \mathcal{N}\left(\frac{1}{n} z^\star + \frac{n - 1}{n} \mu, \frac{n - 1}{n^2} C_\sigma \right)\,.
\end{equation}
However, the Central Limit Theorem is a "limit in distribution" of the empirical means. Thus, \textit{the limit distribution of $\hat{\mu}_n$ is just the constant $\mu$ under both $H_0$ and $H_1$.} The effect of $z^\star$ disappears in this statement, as $n \rightarrow \infty$. For their claim to be ``rigorously'' correct, one should assume that the data-generating distribution is exactly a Gaussian distribution. This gives the exact distribution of $\hat{\mu}_n$ under $H_0$ and $H_1$ as expressed by the two equations above. Supposing that the data-generating distributions are Gaussian distributions simplifies the analysis, since now there is no need to go for asymptotics in $n$ and $d$, and thus, there is no need for Edgeworth expansions and Lindeberg CLT. In contrast, our results provide a way to rigorously justify under which conditions this holistic view of ``equivalence to testing between Gaussians'' is correct, i.e.  finite $4$-th moment of the data distribution. 

(b) As a score function, \cite{leemann2023gaussian} chooses to analyse the distribution of the "norm squared" of a re-centred and normalised version of the mean i.e.  $S_n \approx \|C^{- 1/2} ( \hat{\mu}_n - \mu) \|^2 $ while hiding some constants specific to their analysis. They characterise the distribution of $S_n$ and show that it is a (scaled) non-central chi-squared distribution with $d$ degrees of freedom, with different parameters under $H_0$ and $H_1$. In our case, we provide the asymptotic distribution of the LR score directly under $H_0$ and $H_1$, which provides a simpler covariance score.

%% file: appendix/experiments_appendix.tex
\section{EXTENDED EXPERIMENTS}\label{app:exp_details}
We present additional experiments and results, for both synthetic and real data settings.

\subsection{Experiments on Synthetic Data}~\vspace{-.5em}
We aim to validate additional theoretical results empirically. Thus, we ask

1. \textit{Is the distribution of the LR test tightly characterised by the Gaussian distributions of Lemma~\ref{thm:asymp_lr}?} 

2. \textit{Is the asymptotic approximations for the log-likelihood ratio test suggested by our analysis tight enough? }

3. \textit{Are the power of the LR tests tightly determined by Theorem~\ref{crl:per_leak},~\ref{thm:noisy_mean}, and~\ref{thm:lr_sub} for the empirical mean, noisy empirical mean, and sub-sampled empirical mean mechanisms, respectively?}

4. \textit{What is the effect of $n_0$, the number of reference points, on the power of the empirical attack?}

\textbf{Experimental setup} 
We take $n = 1000$, $\tau =5$ and thus $d = 5000$. The data generating distribution $\mathcal{D}$ is a $d$ dimensional Bernoulli distribution, with parameter $p \in [0,1]^d$. To choose $p$, we sample it uniformly randomly from $[a, 1-a]^d$, where $a=0.25$ to avoid limit cases. Once $p$ is chosen, we use the same $\mathcal{D} \defn \mathrm{Bern}(p)$ for all the experiments. The three mechanisms considered are $\mathcal{M}^\mathrm{emp}_n$, $\mathcal{M}^\gamma_n$ and $\mathcal{M}^\mathrm{sub}_{n, \rho}$. The adversaries chosen to each mechanism are the thresholding adversaries based on the asymptotic approximations of LR tests, suggested by the analysis. For $\mathcal{M}^\mathrm{emp}_n$, we use the approximation of Eq~\eqref{eq:asymp_lr_real}. The approximate LR test for $\mathcal{M}^\gamma_n$ is taken by replacing $C_\sigma$ by $C_{\sigma} + C_\gamma$ in Eq~\eqref{eq:asymp_lr_real}. For $\mathcal{M}^\mathrm{sub, \rho}_n$, we take the approximate log-likelihood ration from Section~\ref{sec:proof_sub} of the analysis. Finally, we choose target points in $\{0, 1\}^d$ in three ways: 

(a) The \textbf{easiest point to attack} $z^\star_{\mathrm{easy}}$ is the point with the highest Mahalanobis distance with respect to $p$. It is the point with binary coordinates being the furthest away from the coordinates of $p$: $z^\star_{\mathrm{easy}} = \left(\ind{p_i \leq 1/2} \right)_{i = 1}^d$. 

(b) Similarly, the \textbf{hardest point to attack} is $z^\star_{\mathrm{hard}} = \left(\ind{p_i > 1/2} \right)_{i = 1}^d$, which takes the binaries closest to $p$. 

(c) A \textbf{medium point to attack} would be just a point randomly sampled from the data-generating distribution, i.e. $z^\star_{\mathrm{med}} \sim \mathrm{Bern}(p)$ , for which the Mahalanobis distance is of order $d$ and the leakage score of order $\tau = d/n$.

All the algorithms are implemented in Python (version $3.8$) and are tested with an 8-core 64-bit Intel i5@1.6 GHz CPU. We run each fixed-target MI game $1000$ times, and plot the results in Figures~\ref{fig:ext_experiments}. \textbf{All the assumptions from the analysis are verified in this setting.}

\textbf{Analysis of results}
1. Figure~\ref{fig:ext_experiments} (a) and (b) shows the distribution of the covariance LR scores as in Eq.~\eqref{eq:asymp_lr_real}. To generate these two figures, we first simulate the crafter, i.e. Algorithm~\ref{alg:crafter}, for $T= 1000$ to get a list of $500$ empirical means where the target was not included and $500$ empirical means where the target was included. Then we compute the covariance LR score of Eq.~\ref{eq:asymp_lr_real} for each of the $1000$ empirical means, with $\mu = p$ and $C_\sigma = \mathrm{diag}(p(1 - p))$ and $z^\star = z^\star_{\mathrm{easy}} $ in Figure~\ref{fig:ext_experiments} (a) and $z^\star = z^\star_{\mathrm{hard}} $ for Figure~\ref{fig:ext_experiments} (b). These two figures show that the LR score is indeed distributed as Gaussians with means $- \frac{m^\star}{2}$ and $\frac{m^\star}{2}$, and variance $m^\star$ as predicted by Theorem~\ref{thm:asymp_lr}.

2. Figure~\ref{fig:ext_experiments} (c) compares the theoretical power function with the empirical ROC curve of the for the approximate log-likelihood ratio of Eq.~\ref{eq:asymp_lr_real} with $z^\star_{\mathrm{hard}}$. The ``Theoretical" dashed line in Figure~\ref{fig:ext_experiments} (c) corresponds to plotting the theoretical $\mathrm{Pow}_n$ of (Eq~\ref{eq:pow_emp}). The ``empirical" solid line is generated by simulating the MI game of Algorithm~\ref{alg:fix_game} for $T= 1000$ rounds. Then by varying the thresholds of the approximate log-likelihood ratio adversary, we empirically compute the false positive and true positive rates. Eq.~\ref{eq:asymp_lr_real} perfectly characterises the ROC curve as predicted by Corollary. 

3. Figures~\ref{fig:ext_experiments} (d), (e) and (f) are generated similarly to Figure (c), for the three mechanisms $\mathcal{M}^\mathrm{emp}_n$, $\mathcal{M}^\gamma_n$ and $\mathcal{M}^\mathrm{sub, \rho}_n$. Figure~\ref{fig:ext_experiments} (d) investigates the effect of $m^\star$ by varying the target datum, where the blue line corresponds to $z^\star_{\mathrm{easy}}$, the green line to $z^\star_{\mathrm{med}}$ and the red line to $z^\star_{\mathrm{hard}}$. Figure~\ref{fig:ext_experiments} (e) and (f) investigate the effect of the noise scale $\gamma$ and the sub-sampling ratio $\rho$, by varying these parameters for the same $z^\star_{\mathrm{easy}}$. The three figures validate the prediction of the theoretical power functions, and show that $m^\star$, $\tilde{m}^\star_\gamma$ and $\rho m^\star$ perfectly characterise the target-dependent leakage of the three mechanisms $\mathcal{M}^\mathrm{emp}_n$, $\mathcal{M}^\gamma_n$ and $\mathcal{M}^\mathrm{sub}_{n, \rho}$. 

4. In Figure~\ref{fig:effect_ref}, we plot the power of the empirical covariance attack on the easy target, for different number of reference points $n_0$. Figure~\ref{fig:effect_ref} shows that the power of the test decreases as $n_0$ gets smaller. However, this decrease is negligible even in the small $n_0$ regime. 

\begin{figure*}[t!]
\centering
\begin{tabular}{cc}
\subfloat[LR distribution on easy target $z^\star_{\mathrm{easy}}$]{\includegraphics[width=0.45\textwidth]{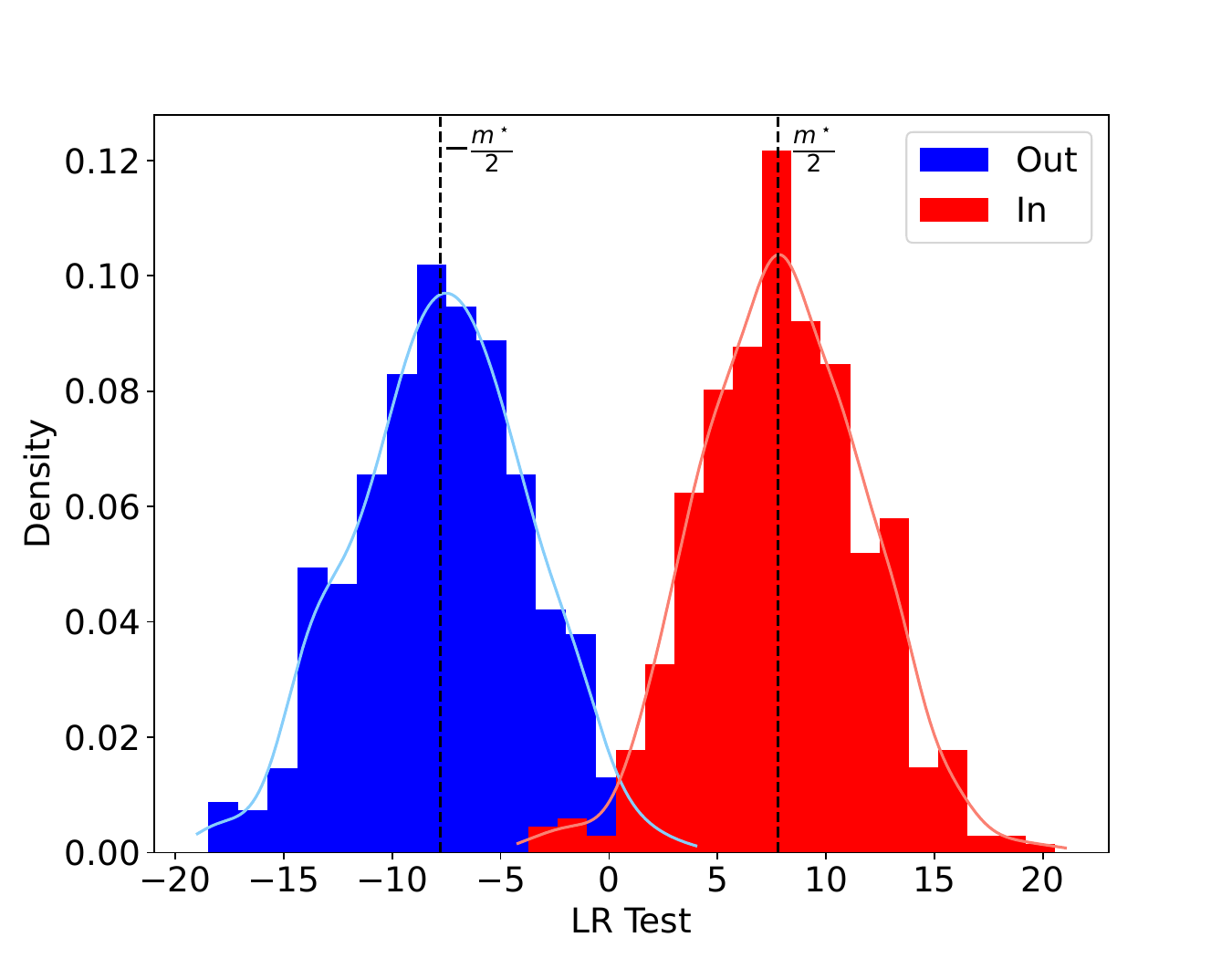}} & 
\subfloat[LR distribution on hard target $z^\star_{\mathrm{hard}}$]{\includegraphics[width=0.45\textwidth]{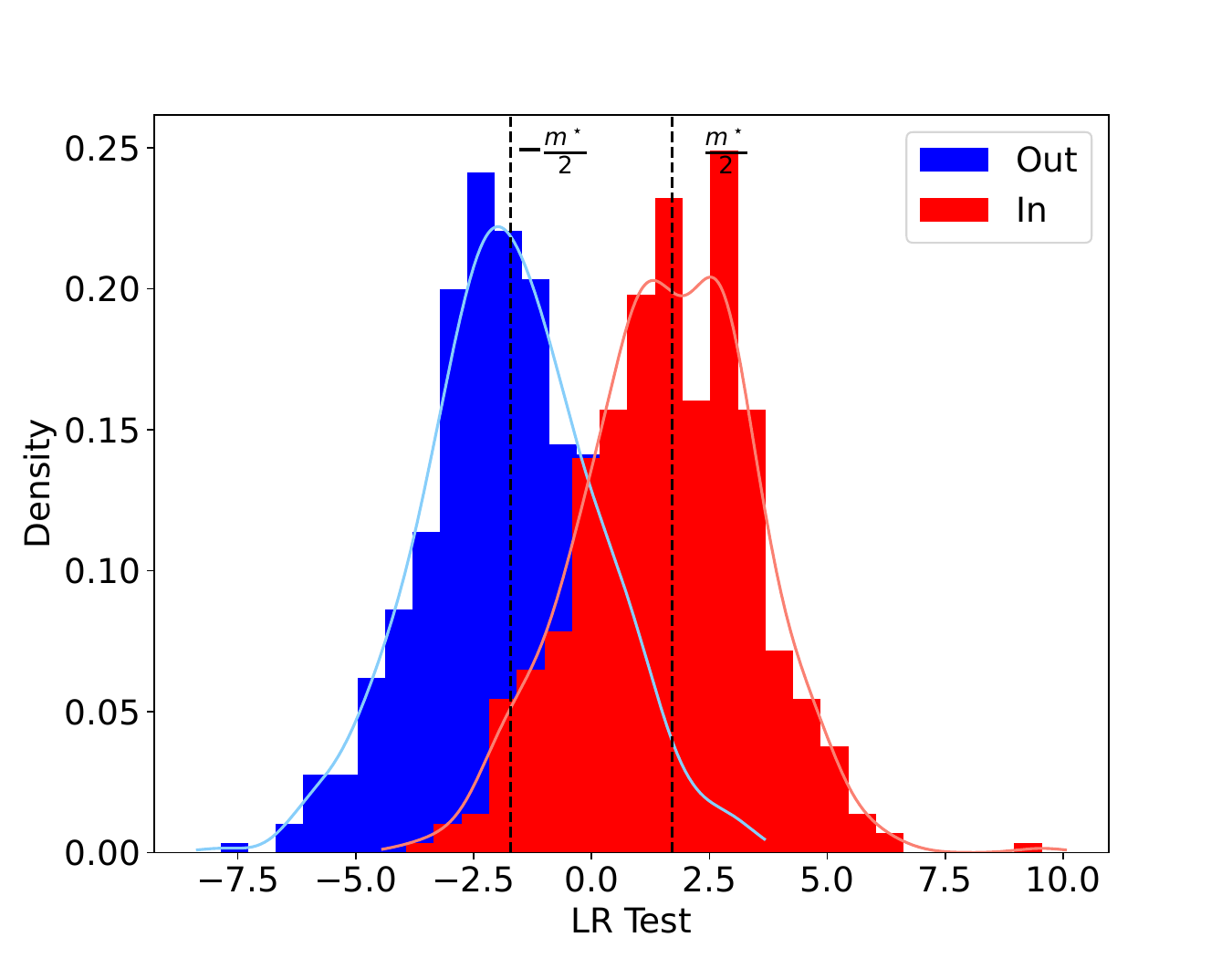}} \\ 
\subfloat[Theoretical and empirical trade-offs]{\includegraphics[width=0.45\textwidth]{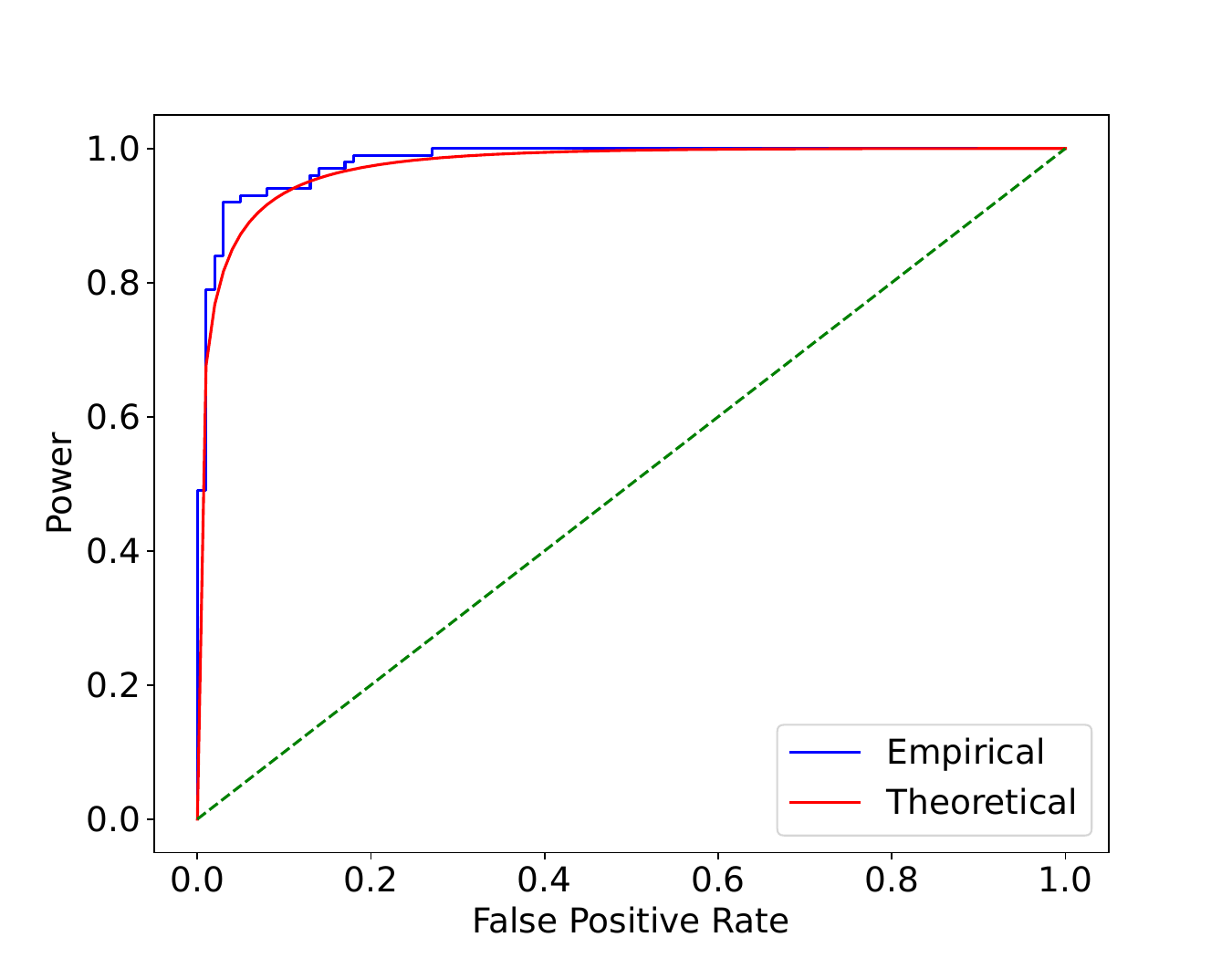}} &
\subfloat[Effect of $m^\star$]{\includegraphics[width=0.45\textwidth]{figures/effect_m_star4.pdf}} \\
\subfloat[Effect of $\gamma$]{\includegraphics[width=0.45\textwidth]{figures/effect_gamma3.pdf}}&
\subfloat[Effect of $\rho$]{\includegraphics[width=0.45\textwidth]{figures/effect_rho3.pdf}} \\
\end{tabular}
\caption{Experimental demonstration of the theoretical results and impacts of $m^\star$, noise, and sub-sampling ratio on leakage. Dotted lines represent theoretical bounds and solid lines represent the empirical results.}\label{fig:ext_experiments}
\end{figure*}

\begin{figure*}[t!]
\centering
\begin{tabular}{cc}
\subfloat[Large $n_0$ regime]{\includegraphics[width=0.45\textwidth]{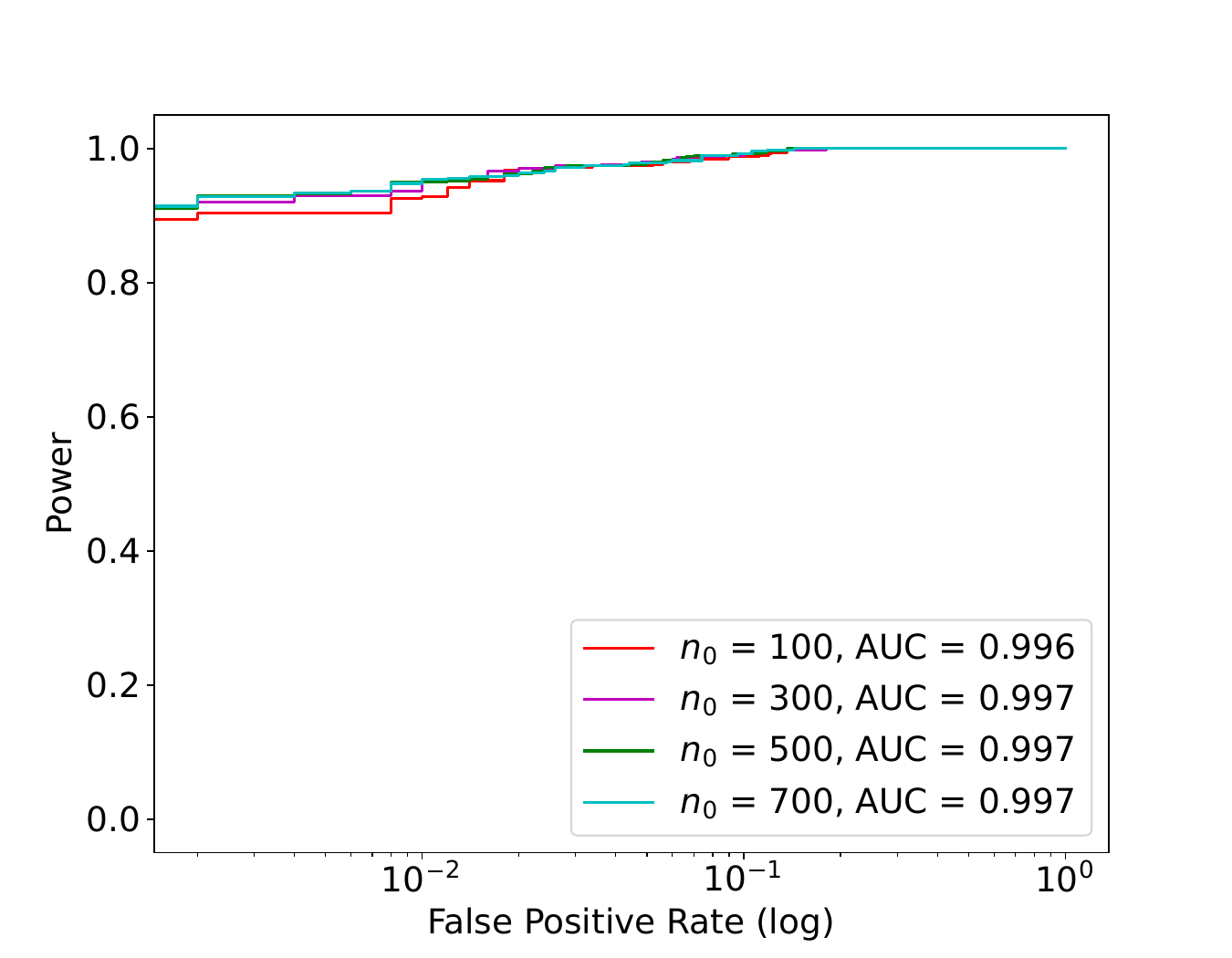}} & 
\subfloat[Small $n_0$ regime]{\includegraphics[width=0.45\textwidth]{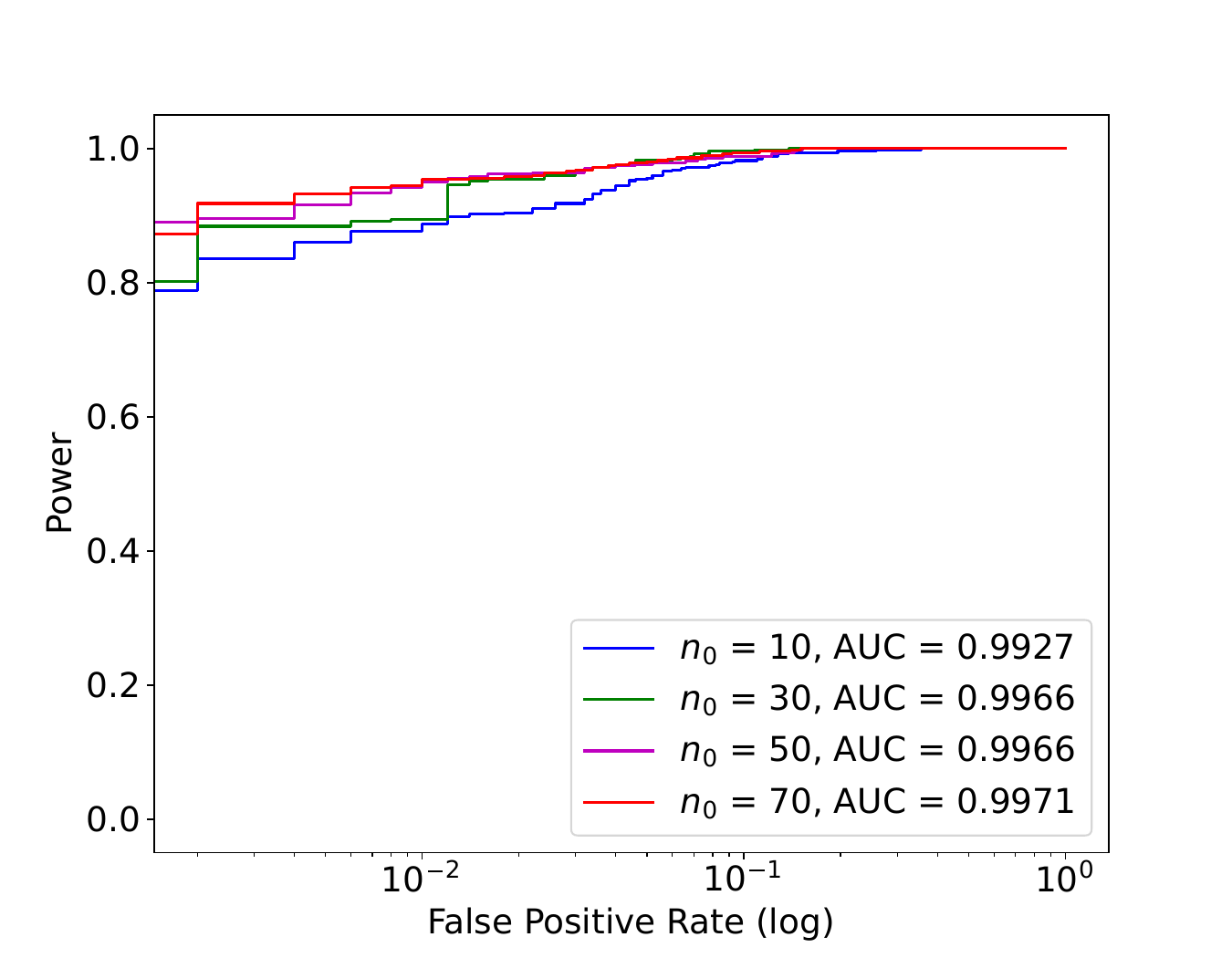}} \\ 
\end{tabular}
\caption{Effect of the number of reference points $n_0$ on the power of the covariance attack.}\label{fig:effect_ref}
\end{figure*}

\newpage
\subsection{Attacking in the White-box Federated Learning Setting}
All the attacks are implemented in Python (version $3.8$) and are tested with an NVIDIA GeForce RTX 2080 Ti GPU. We run each attack $100$ times, and plot the results in Figures~\ref{fig:experiments_white_box}. We train our neural network models using PyTorch (version $2.3.0$)~\citep{paszke2019pytorch}.

In addition to the results stated in Section~\ref{sec:exp}, here we add the a few additional results for this setting.

We plot the distribution of the Mahalanobis score, for FMNIST and CIFAR10 in Figure~\ref{fig:maha_dist}. We also plot the easiest, medium and hardest points to attack $z_\text{easy}$, $z_\text{med}$ and $z_\text{hard}$ for FMNIST and CIFAR10 in Figure~\ref{fig:canaries}. Interestingly, the easy instances from both datasets are the ones that are hard to detect as real objects, i.e. they incur a high loss for an ML model. In contrast, the hard ones are like any other well-classifiable data point in the dataset. This echoes the practical results in MI attack literature~\citep{carlini2022membership} from the black-box settings.

\noindent \textbf{Relation to the theoretical assumptions in the analysis.} For this setting, the assumptions from our analysis are not satisfied. Specifically, the gradient loss distribution, which is the data-generating distribution in this setting, does not satisfy two assumptions from Section~\ref{sec:mia_mean}: finite 4-th moment and independence of the components. It is hard to say anything about the distribution of the gradient loss on natural data. However, our results are a step in the right direction: we don't assume that the gradients are exactly distributed as Gaussians~\cite{maddock2022canife} or Bernoulli~\cite{sankararaman2009genomic}. We leave it for future work to weaken the assumptions on the data-generating distribution. We provide the following details regarding the importance of the ``independence" assumption.

Our analysis assumes that the data-generating distribution $\mathcal{D}$ is a product distribution, i.e. the columns of the input are independent. This assumption is standard and has been used in different related works in the tracing literature~\citep{homer2008resolving, sankararaman2009genomic, dwork2015robust}. Our proof could be adapted to the dependent case using a \textit{multivariate} Edgeworth expansion in the likelihood ratio test. The same conclusions of our analysis will follow, with the only difference being that the covariance matrix will no longer be diagonal but a ``full'' matrix. Additional technical assumptions must be added to rigorously use a high-dimensional \textit{multivariate} Edgeworth expansion, making the analysis very technical without yielding additional insights. We leave it as a future direction to adapt the proof to the dependent case.

\begin{figure*}[t!]
\centering
\begin{tabular}{cc}
\subfloat[FMNIST]{\includegraphics[width=0.4\textwidth]{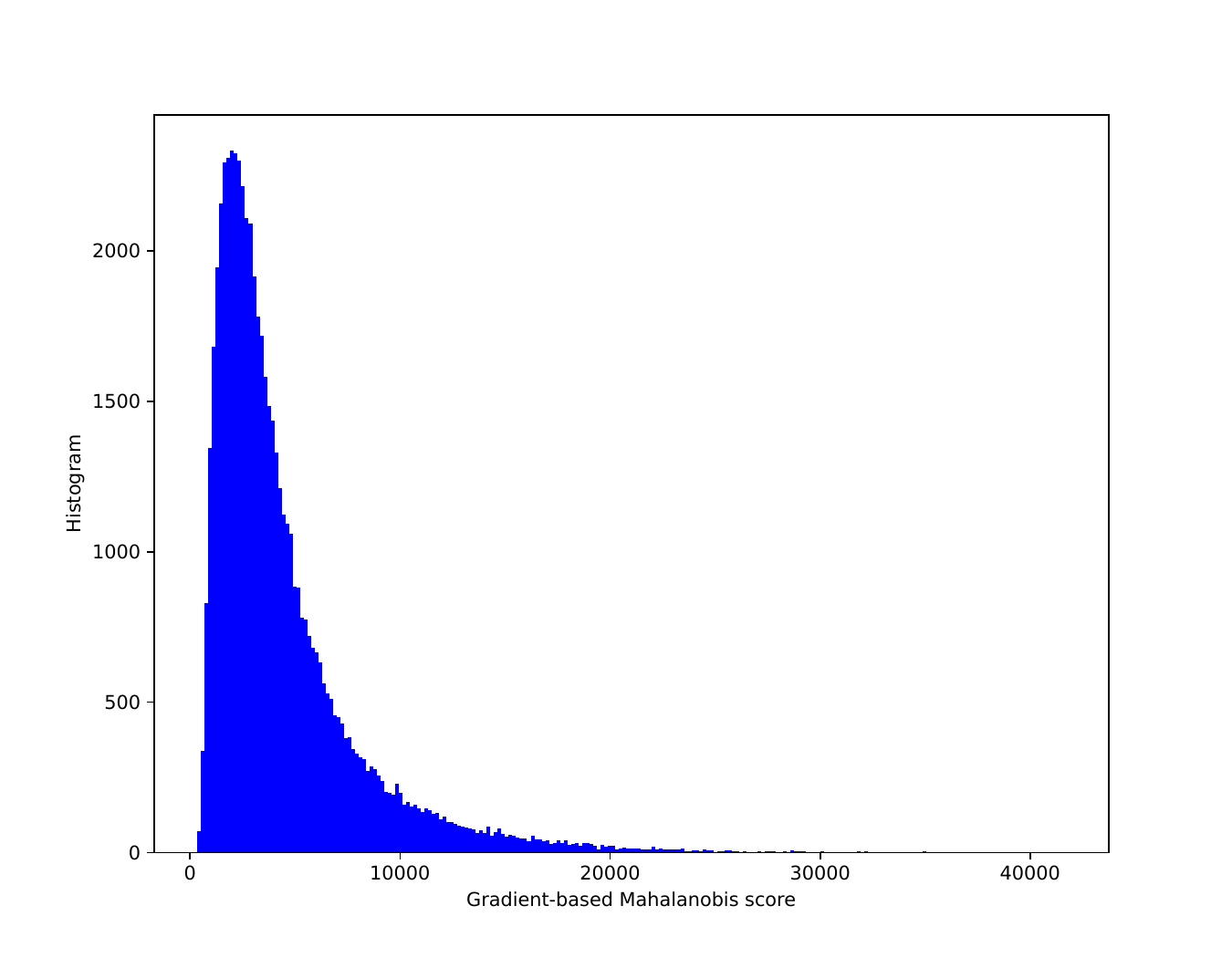}} & 
\subfloat[CIFAR10]{\includegraphics[width=0.4\textwidth]{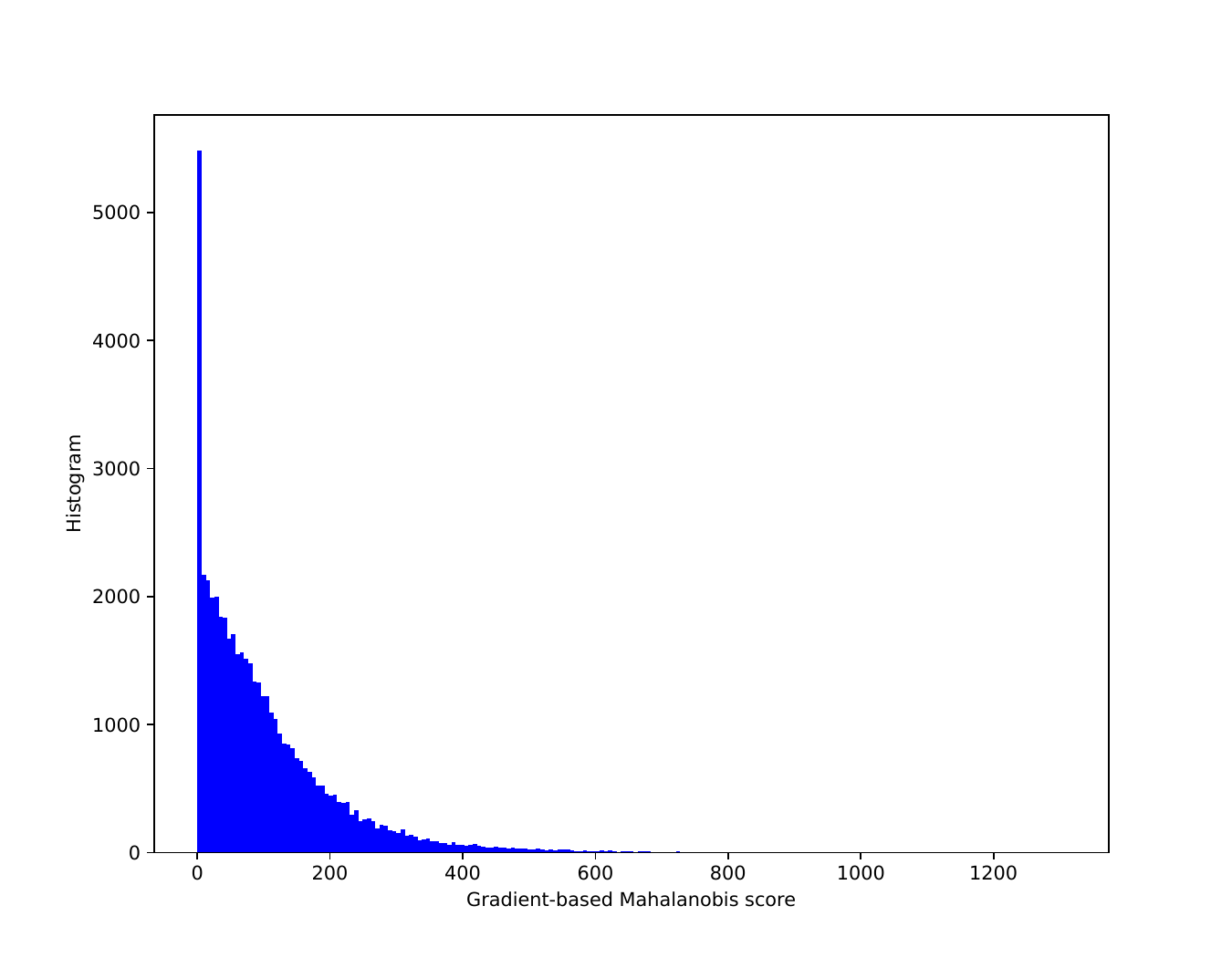}} \\ 
\end{tabular}
\caption{Distribution of the Mahalanobis score.}\label{fig:maha_dist}
\end{figure*}

\begin{figure*}[t!]
\centering
\begin{tabular}{cc}
\subfloat[$z_\text{easy}$ for FMNIST]{\includegraphics[width=0.38\textwidth]{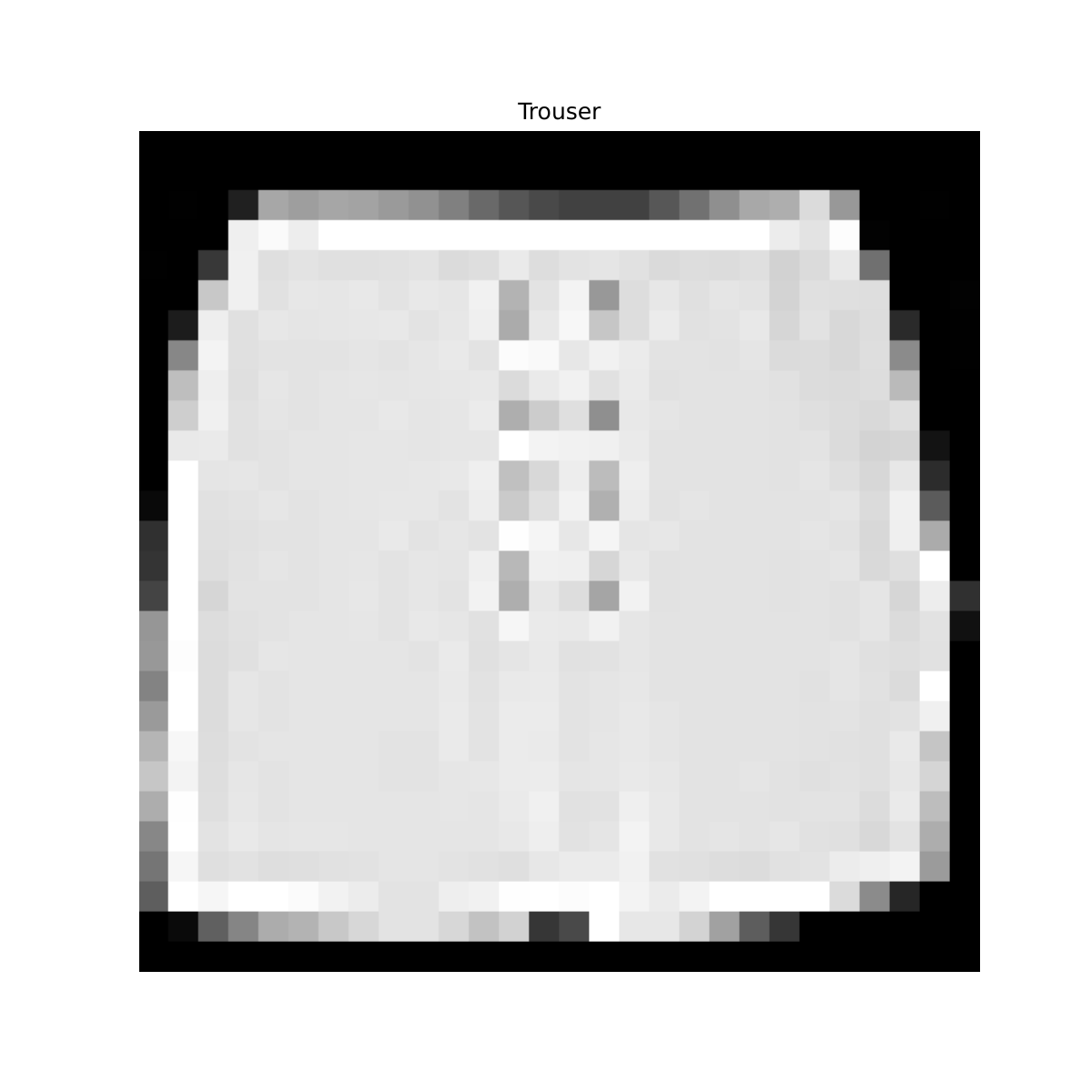}} & 
\subfloat[$z_\text{easy}$ for CIFAR10]{\includegraphics[width=0.38\textwidth]{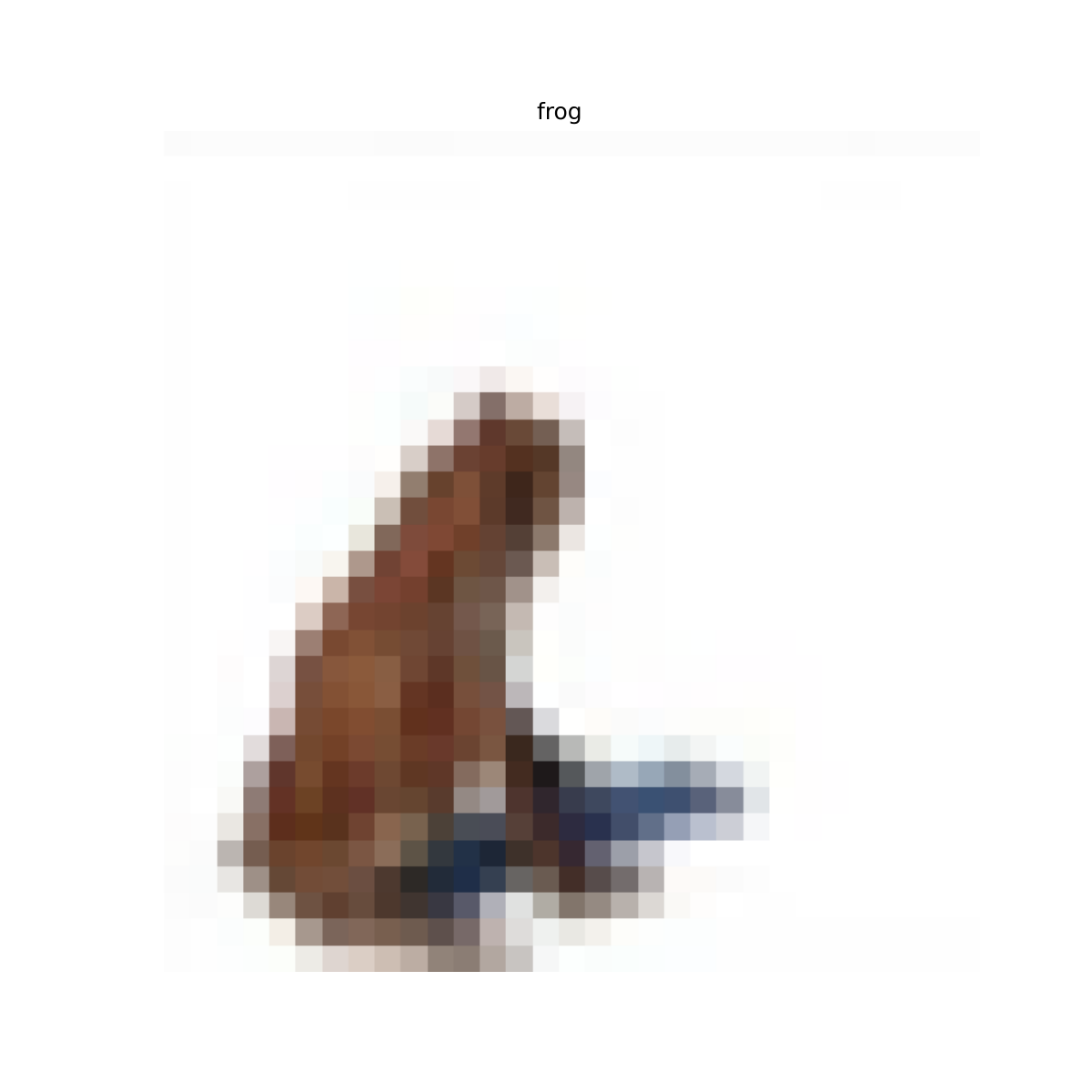}}  \\
\subfloat[$z_\text{med}$ for FMNIST]{\includegraphics[width=0.38\textwidth]{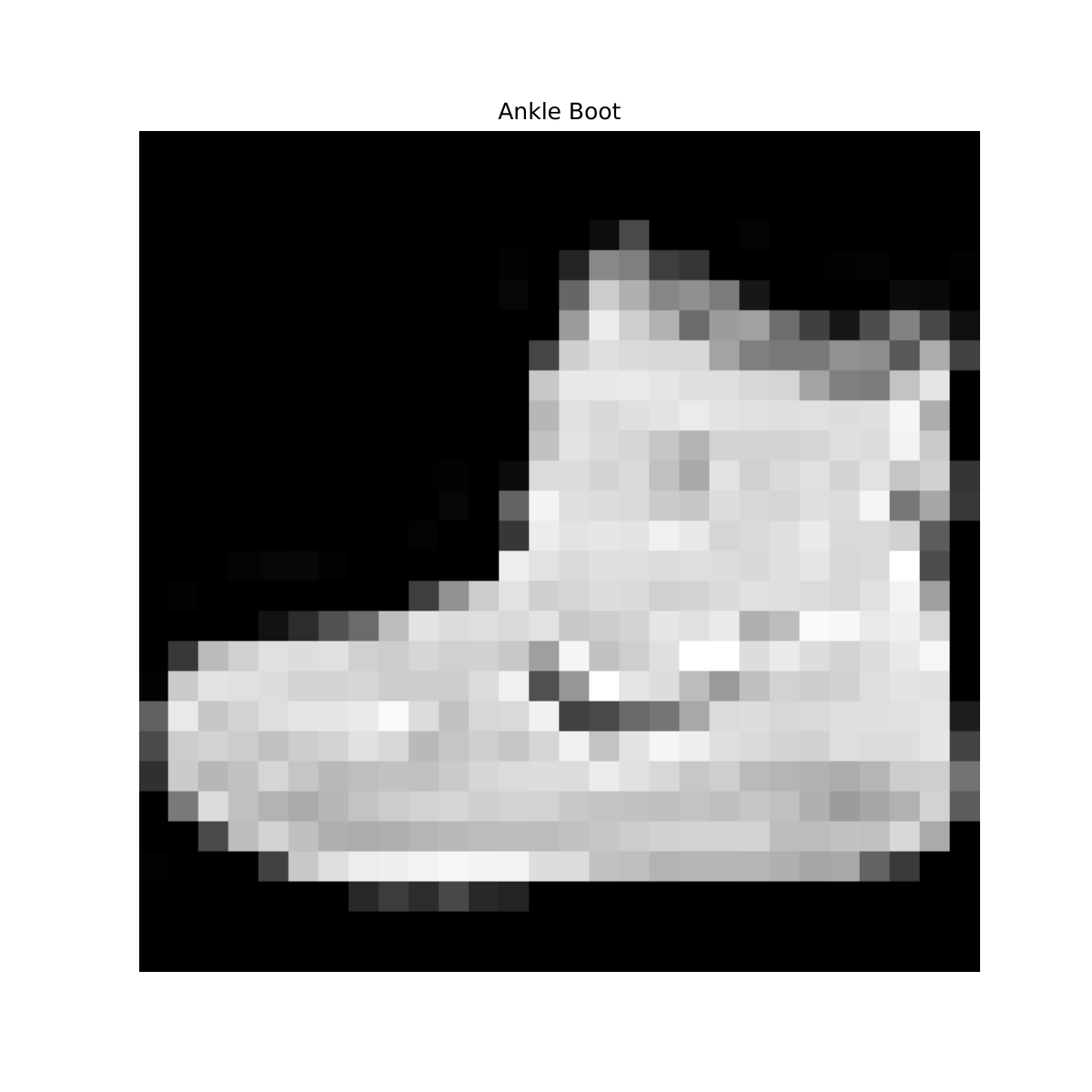}} & 
\subfloat[$z_\text{med}$ for CIFAR]{\includegraphics[width=0.38\textwidth]{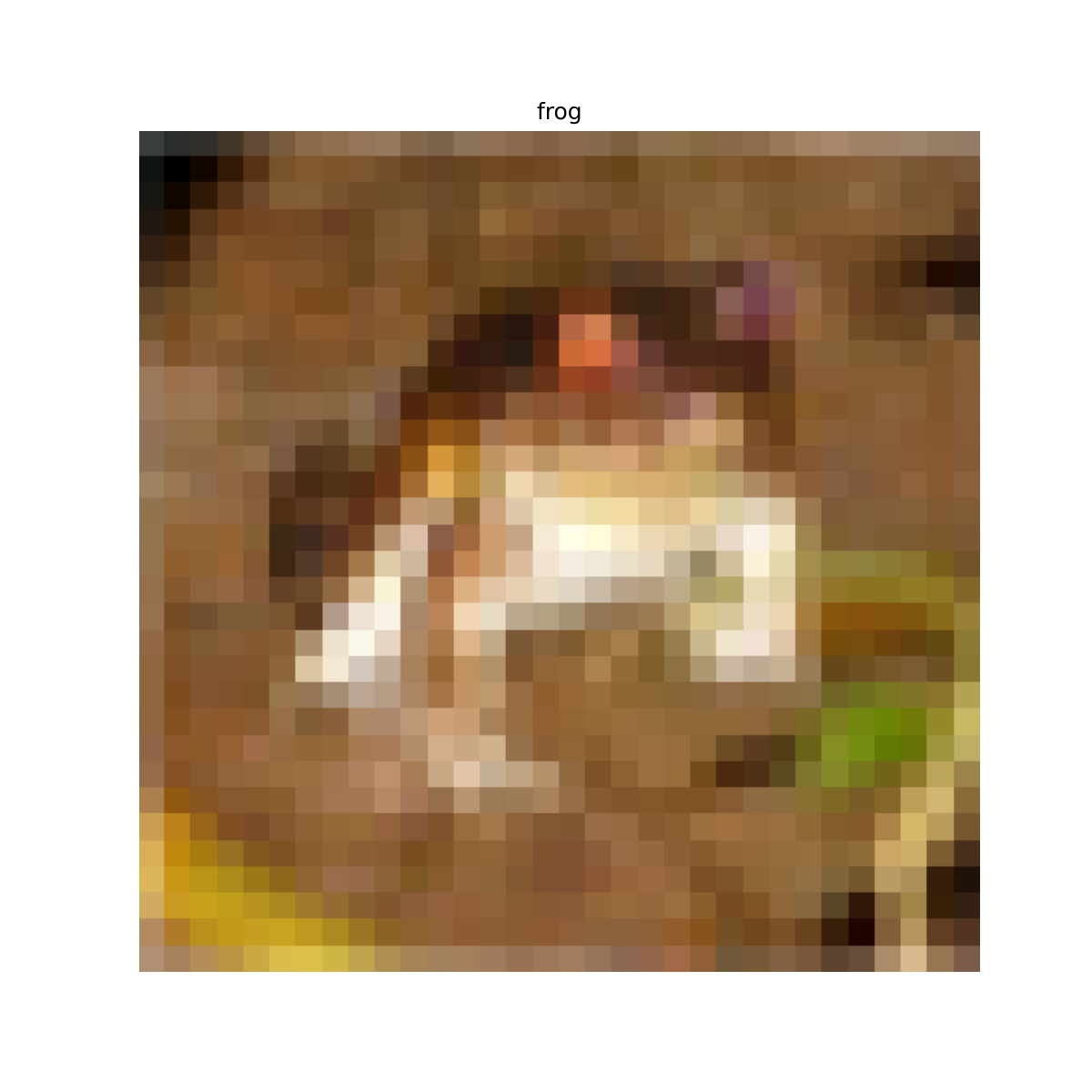}} \\ 
\subfloat[$z_\text{hard}$ for FMNIST]{\includegraphics[width=0.38\textwidth]{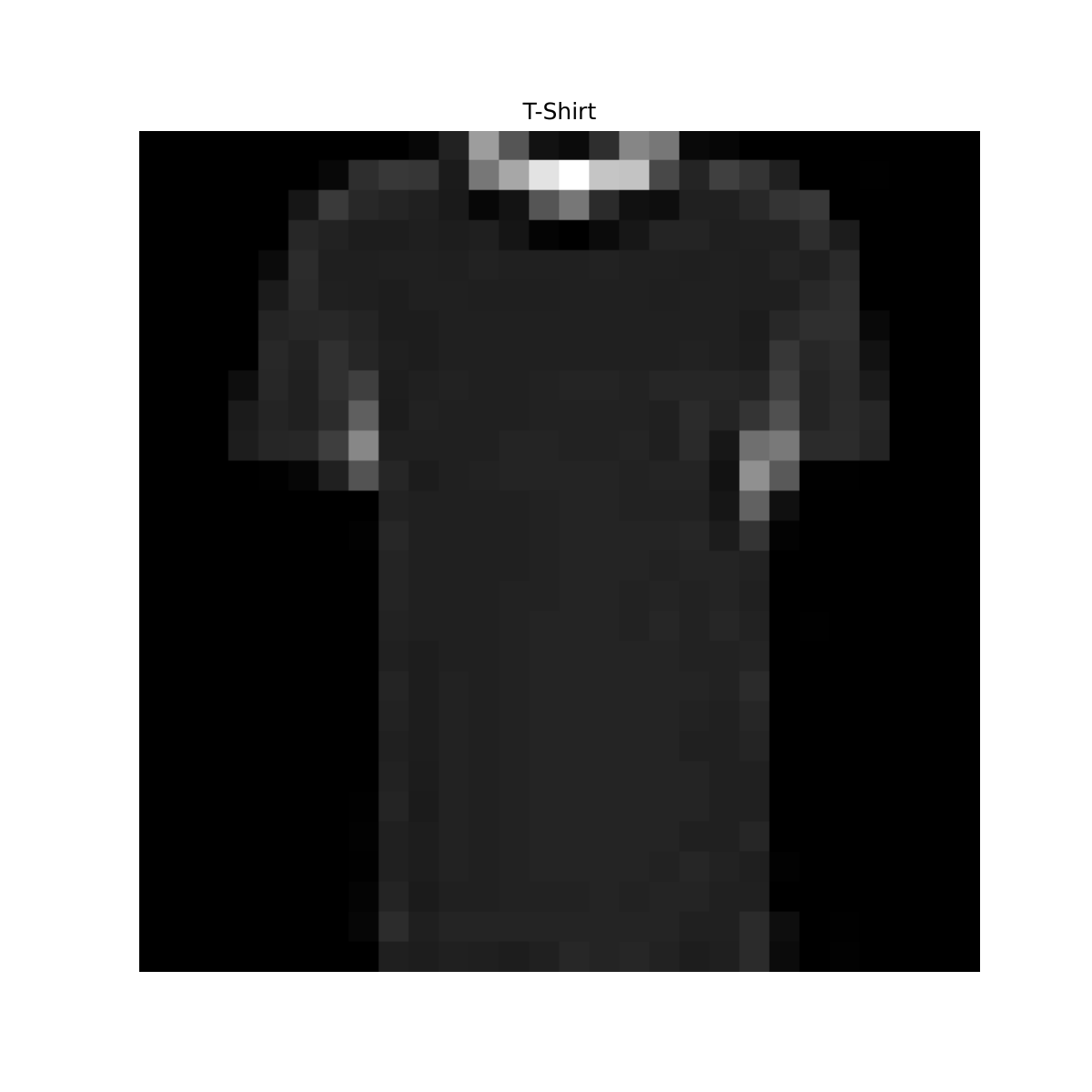}} & 
\subfloat[$z_\text{hard}$ for CIFAR10]{\includegraphics[width=0.38\textwidth]{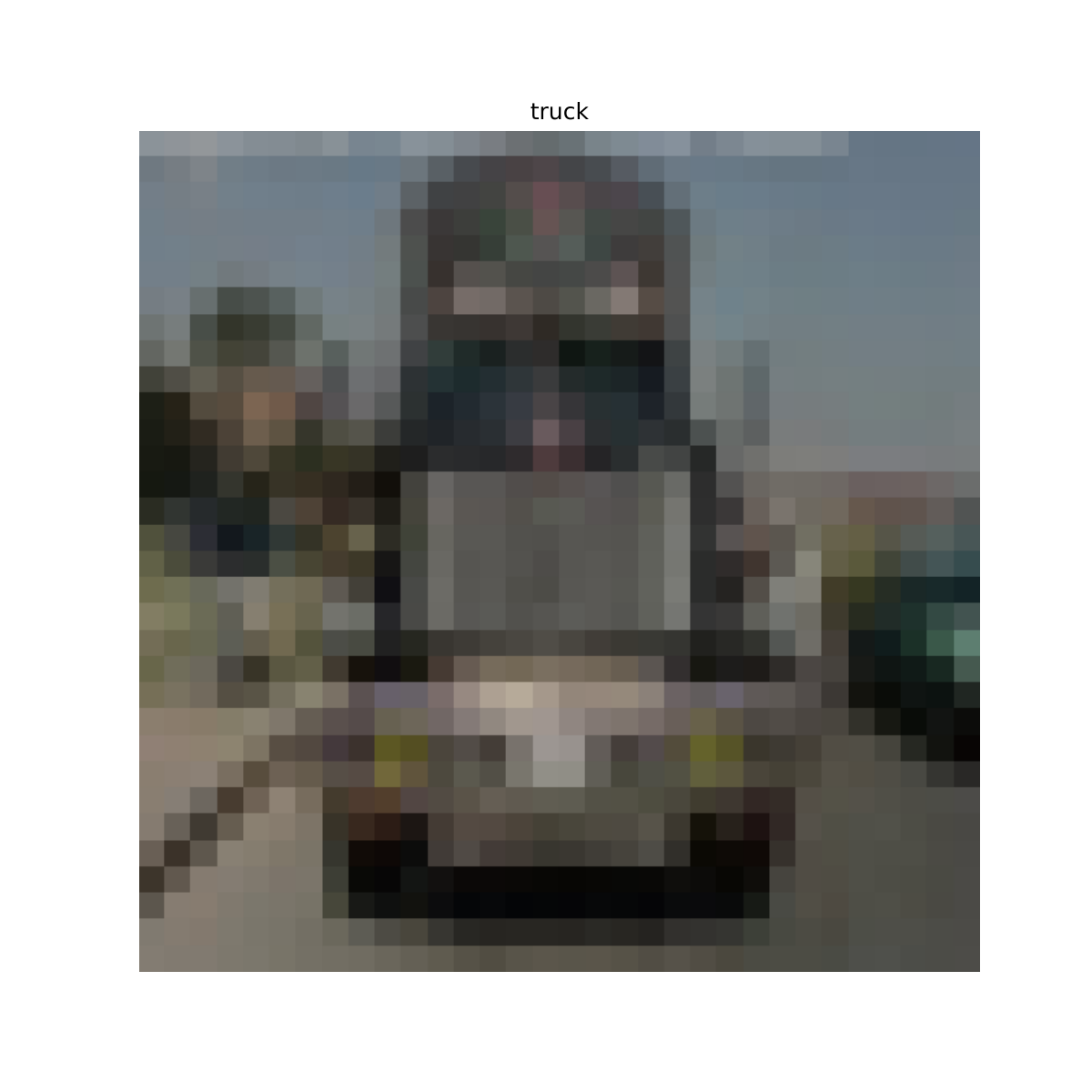}} \\ 
\end{tabular}
\caption{The easiest, medium and hardest points to attack $z_\text{easy}$, $z_\text{med}$ and $z_\text{hard}$ for FMNIST and CIFAR10}\label{fig:canaries}
\end{figure*}

%% file: appendix/limit_impact.tex
\section{BROADER IMPACT}\label{sec:limit}
The goal of this paper is to advance the field of Trustworthy Machine Learning. Our work theoretically quantifies the privacy leakage due to publicly releasing simple statistics, such as the empirical mean of a dataset. Our results are also generalised to ML models. We believe that this work is a step forward giving the analytical tools to quantify the privacy leakage due to participation or not in a dataset. Thus, this provides more informed decisions to the individual on how much their participation may impact the published statistics/models. We point out that our work only studies theoretically the likelihood ratio attacks, and our theoretical tracing attacks to dust for the presence of target individuals were only used in simulated data and public datasets (FMNIST, CIFAR10) in the experiments. 





%% file: main.bbl
\begin{thebibliography}{}

\bibitem[Abadi et~al., 2016]{dpsgd}
Abadi, M., Chu, A., Goodfellow, I., McMahan, H.~B., Mironov, I., Talwar, K., and Zhang, L. (2016).
\newblock Deep learning with differential privacy.
\newblock In {\em Proceedings of the 2016 ACM SIGSAC conference on computer and communications security}, pages 308--318.

\bibitem[Andrew et~al., 2023]{andrew2023one}
Andrew, G., Kairouz, P., Oh, S., Oprea, A., McMahan, H.~B., and Suriyakumar, V. (2023).
\newblock One-shot empirical privacy estimation for federated learning.
\newblock {\em arXiv preprint arXiv:2302.03098}.

\bibitem[Balle et~al., 2018]{balle2018privacy}
Balle, B., Barthe, G., and Gaboardi, M. (2018).
\newblock Privacy amplification by subsampling: Tight analyses via couplings and divergences.
\newblock {\em Advances in neural information processing systems}, 31.

\bibitem[Carlini et~al., 2022a]{carlini2022membership}
Carlini, N., Chien, S., Nasr, M., Song, S., Terzis, A., and Tramer, F. (2022a).
\newblock Membership inference attacks from first principles.
\newblock In {\em 2022 IEEE Symposium on Security and Privacy (SP)}, pages 1897--1914. IEEE.

\bibitem[Carlini et~al., 2022b]{carlini2022privacy}
Carlini, N., Jagielski, M., Zhang, C., Papernot, N., Terzis, A., and Tramer, F. (2022b).
\newblock The privacy onion effect: Memorization is relative.
\newblock {\em Advances in Neural Information Processing Systems}, 35:13263--13276.

\bibitem[Dong et~al., 2019]{dong2019gaussian}
Dong, J., Roth, A., and Su, W.~J. (2019).
\newblock Gaussian differential privacy.
\newblock {\em arXiv preprint arXiv:1905.02383}.

\bibitem[Dwork and Roth, 2014]{dpbook}
Dwork, C. and Roth, A. (2014).
\newblock The algorithmic foundations of differential privacy.
\newblock {\em Foundations and Trends{\textregistered} in Theoretical Computer Science}, 9(3--4):211--407.

\bibitem[Dwork et~al., 2017]{dwork2017exposed}
Dwork, C., Smith, A., Steinke, T., and Ullman, J. (2017).
\newblock Exposed! a survey of attacks on private data.
\newblock {\em Annual Review of Statistics and Its Application}, 4:61--84.

\bibitem[Dwork et~al., 2015]{dwork2015robust}
Dwork, C., Smith, A., Steinke, T., Ullman, J., and Vadhan, S. (2015).
\newblock Robust traceability from trace amounts.
\newblock In {\em 2015 IEEE 56th Annual Symposium on Foundations of Computer Science}, pages 650--669. IEEE.

\bibitem[Homer et~al., 2008]{homer2008resolving}
Homer, N., Szelinger, S., Redman, M., Duggan, D., Tembe, W., Muehling, J., Pearson, J.~V., Stephan, D.~A., Nelson, S.~F., and Craig, D.~W. (2008).
\newblock Resolving individuals contributing trace amounts of dna to highly complex mixtures using high-density snp genotyping microarrays.
\newblock {\em PLoS genetics}, 4(8):e1000167.

\bibitem[Humphries et~al., 2023]{humphries2023investigating}
Humphries, T., Oya, S., Tulloch, L., Rafuse, M., Goldberg, I., Hengartner, U., and Kerschbaum, F. (2023).
\newblock Investigating membership inference attacks under data dependencies.
\newblock In {\em 2023 IEEE 36th Computer Security Foundations Symposium (CSF)}, pages 473--488. IEEE.

\bibitem[Jagielski et~al., 2020]{jagielski2020auditing}
Jagielski, M., Ullman, J., and Oprea, A. (2020).
\newblock Auditing differentially private machine learning: How private is private sgd?
\newblock {\em Advances in Neural Information Processing Systems}, 33:22205--22216.

\bibitem[Krizhevsky et~al., 2009]{cifar10}
Krizhevsky, A., Hinton, G., et~al. (2009).
\newblock Learning multiple layers of features from tiny images.

\bibitem[LeCun et~al., 2015]{cnn}
LeCun, Y., Bengio, Y., and Hinton, G. (2015).
\newblock Deep learning.
\newblock {\em nature}, 521(7553):436--444.

\bibitem[Leemann et~al., 2023]{leemann2023gaussian}
Leemann, T., Pawelczyk, M., and Kasneci, G. (2023).
\newblock Gaussian membership inference privacy.
\newblock {\em arXiv preprint arXiv:2306.07273}.

\bibitem[Maddock et~al., 2022]{maddock2022canife}
Maddock, S., Sablayrolles, A., and Stock, P. (2022).
\newblock Canife: Crafting canaries for empirical privacy measurement in federated learning.
\newblock {\em arXiv preprint arXiv:2210.02912}.

\bibitem[Mahalanobis, 1936]{mahalanobisdistance}
Mahalanobis, P.~C. (1936).
\newblock On the generalised distance in statistics.
\newblock In {\em Proceedings of the National Institute of Science of India}, volume~12, pages 49--55.

\bibitem[Murakonda and Shokri, 2020]{murakonda2020ml}
Murakonda, S.~K. and Shokri, R. (2020).
\newblock Ml privacy meter: Aiding regulatory compliance by quantifying the privacy risks of machine learning.
\newblock {\em arXiv preprint arXiv:2007.09339}.

\bibitem[Nasr et~al., 2023]{nasr2023tight}
Nasr, M., Hayes, J., Steinke, T., Balle, B., Tram{\`e}r, F., Jagielski, M., Carlini, N., and Terzis, A. (2023).
\newblock Tight auditing of differentially private machine learning.
\newblock {\em arXiv preprint arXiv:2302.07956}.

\bibitem[Neyman and Pearson, 1933]{neyman1933ix}
Neyman, J. and Pearson, E.~S. (1933).
\newblock Ix. on the problem of the most efficient tests of statistical hypotheses.
\newblock {\em Philosophical Transactions of the Royal Society of London. Series A, Containing Papers of a Mathematical or Physical Character}, 231(694-706):289--337.

\bibitem[Paszke et~al., 2019]{paszke2019pytorch}
Paszke, A., Gross, S., Massa, F., Lerer, A., Bradbury, J., Chanan, G., Killeen, T., Lin, Z., Gimelshein, N., Antiga, L., et~al. (2019).
\newblock Pytorch: An imperative style, high-performance deep learning library.
\newblock {\em Advances in neural information processing systems}, 32.

\bibitem[Petrov, 2012]{petrov2012sums}
Petrov, V.~V. (2012).
\newblock {\em Sums of independent random variables}, volume~82.
\newblock Springer Science \& Business Media.

\bibitem[Sankararaman et~al., 2009]{sankararaman2009genomic}
Sankararaman, S., Obozinski, G., Jordan, M.~I., and Halperin, E. (2009).
\newblock Genomic privacy and limits of individual detection in a pool.
\newblock {\em Nature genetics}, 41(9):965--967.

\bibitem[Shokri et~al., 2017]{shokri2017membership}
Shokri, R., Stronati, M., Song, C., and Shmatikov, V. (2017).
\newblock Membership inference attacks against machine learning models.
\newblock In {\em 2017 IEEE symposium on security and privacy (SP)}, pages 3--18. IEEE.

\bibitem[Song and Marn, 2020]{privacytesttensorflow}
Song, S. and Marn, D. (2020).
\newblock Introducing a new privacy testing library in tensorflow.

\bibitem[Steinke et~al., 2023]{steinke2023privacy}
Steinke, T., Nasr, M., and Jagielski, M. (2023).
\newblock Privacy auditing with one (1) training run.
\newblock {\em arXiv preprint arXiv:2305.08846}.

\bibitem[Steinke and Ullman, 2020]{diffprivorg_blog}
Steinke, T. and Ullman, J. (2020).
\newblock The pitfalls of averagecase differential privacy.

\bibitem[Vaart, 1998]{Vaart_1998}
Vaart, A. W. v.~d. (1998).
\newblock {\em Asymptotic Statistics}.
\newblock Cambridge Series in Statistical and Probabilistic Mathematics. Cambridge University Press.

\bibitem[Van~der Vaart, 2000]{van2000asymptotic}
Van~der Vaart, A.~W. (2000).
\newblock {\em Asymptotic statistics}, volume~3.
\newblock Cambridge university press.

\bibitem[Xiao et~al., 2017]{fmnist}
Xiao, H., Rasul, K., and Vollgraf, R. (2017).
\newblock Fashion-mnist: a novel image dataset for benchmarking machine learning algorithms.

\bibitem[Ye et~al., 2022]{ye2022enhanced}
Ye, J., Maddi, A., Murakonda, S.~K., Bindschaedler, V., and Shokri, R. (2022).
\newblock Enhanced membership inference attacks against machine learning models.
\newblock In {\em Proceedings of the 2022 ACM SIGSAC Conference on Computer and Communications Security}, pages 3093--3106.

\bibitem[Yeom et~al., 2018]{yeom2018privacy}
Yeom, S., Giacomelli, I., Fredrikson, M., and Jha, S. (2018).
\newblock Privacy risk in machine learning: Analyzing the connection to overfitting.
\newblock In {\em 2018 IEEE 31st computer security foundations symposium (CSF)}, pages 268--282. IEEE.

\bibitem[Zhang et~al., 2020]{zhang2020privacy}
Zhang, B., Yu, R., Sun, H., Li, Y., Xu, J., and Wang, H. (2020).
\newblock Privacy for all: Demystify vulnerability disparity of differential privacy against membership inference attack.
\newblock {\em arXiv preprint arXiv:2001.08855}.

\end{thebibliography}
